\documentclass{article}

\usepackage{microtype}
\usepackage{graphicx}
\usepackage{subfigure}
\usepackage{booktabs} 
\usepackage{mathtools}
\usepackage{nicefrac}       
\usepackage{hyperref}
\usepackage{mathrsfs}
\usepackage{enumitem}

\usepackage[accepted]{icml2024}

\usepackage{amsmath}
\usepackage{amssymb}
\usepackage{mathtools}
\usepackage{amsthm}
\usepackage{multirow}
\usepackage{graphicx}
\usepackage{bbm}

\usepackage[capitalize,noabbrev]{cleveref}

\theoremstyle{plain}
\newtheorem{theorem}{Theorem}[section]

\newtheorem{lemma}[theorem]{Lemma}
\newtheorem{corollary}[theorem]{Corollary}
\theoremstyle{definition}
\newtheorem{definition}[theorem]{Definition}
\newtheorem{assumption}[theorem]{Assumption}
\theoremstyle{remark}
\newtheorem{remark}[theorem]{Remark}

\usepackage[textsize=tiny]{todonotes}

\icmltitlerunning{Towards Optimal Adversarial Robust Q-learning with Bellman Infinity-error}

\begin{document}

\twocolumn[
\icmltitle{Towards Optimal Adversarial Robust Q-learning with Bellman Infinity-error}



\icmlsetsymbol{equal}{*}

\begin{icmlauthorlist}
\icmlauthor{Haoran Li}{yyy}
\icmlauthor{Zicheng Zhang}{yyy}
\icmlauthor{Wang Luo}{yyy}
\icmlauthor{Congying Han}{yyy}
\icmlauthor{Yudong Hu}{yyy}
\icmlauthor{Tiande Guo}{yyy}
\icmlauthor{Shichen Liao}{yyy}
\end{icmlauthorlist}

\icmlaffiliation{yyy}{School of Mathematical Sciences, University of Chinese Academy of Sciences, Beijing, China}

\icmlcorrespondingauthor{Congying Han}{hancy@ucas.ac.cn}

\icmlkeywords{Machine Learning, ICML}

\vskip 0.3in
]



\printAffiliationsAndNotice{}  

\begin{abstract}
Establishing robust policies is essential to counter attacks or disturbances affecting deep reinforcement learning~(DRL) agents. 
 Recent studies explore state-adversarial robustness and suggest the potential lack of an optimal robust policy~(ORP), posing challenges in setting strict robustness constraints. 
 This work further investigates ORP:  
 At first, we introduce a consistency assumption of policy (CAP) stating that optimal actions in the Markov decision process remain consistent with minor perturbations, supported by empirical and theoretical evidence.
 Building upon CAP, we crucially prove the existence of a deterministic and stationary ORP that aligns with the Bellman optimal policy. Furthermore, we illustrate the necessity of $L^{\infty}$-norm when minimizing Bellman error to attain ORP.  This finding clarifies the vulnerability of prior DRL algorithms that target the Bellman optimal policy with $L^{1}$-norm and motivates us to train a Consistent Adversarial Robust Deep Q-Network (CAR-DQN) by minimizing a surrogate of Bellman Infinity-error.  
 The top-tier performance of CAR-DQN across various benchmarks validates its practical effectiveness and
reinforces the soundness of our theoretical analysis.
Our code is available at \href{https://github.com/leoranlmia/CAR-DQN}{https://github.com/leoranlmia/CAR-DQN}.
\end{abstract}

\section{Introduction} \label{sec: intro}
Deep reinforcement learning (DRL) has shown remarkable success in solving intricate problems~\cite{mnih2015human, lillicrap2015continuous, silver2016mastering} and holds promise across diverse practical domains~\cite{ibarz2021train, kiran2021deep, yu2021reinforcement, zheng2018drn}. 
Nevertheless, subtle perturbations in state observations can severely degrade well-trained DRL agents~\cite{huang2017adversarial, behzadan2017vulnerability, lin2017tactics, ilahi2021challenges}, which limits their trustworthy real-world deployment and emphasizes the crucial need for developing robust DRL algorithms against adversarial attacks.


Pioneering work by \citet{zhang2020robust} introduced the state-adversarial paradigm in DRL by formulating a modified Markov decision process (MDP), termed SA-MDP. Here, the underlying true state remains invariant while the observed state is subjected to disturbances.
 They also highlighted the uncertain existence of an optimal robust policy (ORP) within SA-MDP, indicating a potential conflict between robustness and optimal policy. Consequently, existing methods built upon SA-MDP often seek a trade-off between robust and optimal policies through various regularizations~\cite{oikarinen2021robust, liang2022efficient} or adversarial training~\cite{zhang2021robust, sun2021strongest}.  While enhancing robustness, these methods lack theoretical guarantees and completely neglect the study of ORP.

In this paper, we primarily focus on investigating ORP within SA-MDP.  We suppose that only a few exceptional states lack an ORP, thus it is key for theoretical clarity to eliminate these states. 
Therefore, our work starts with a consistency assumption of policy (CAP), where {optimal actions of \textit{all} states within the MDP exhibit consistency despite adversarial disturbance}. This implies that, from a decision-making perspective, adversaries cannot alter the essence of state observations, which we call the intrinsic state. Despite seeming implausible, we support the rationality of CAP through theoretical analysis and empirical experiments against strong adversarial attacks like FGSM \cite{goodfellow2014explaining} and PGD \cite{madry2017towards}, showcasing that the state set violating CAP is nearly empty.

Building upon CAP, we demonstrate that there always exists a stationary and deterministic adversarial ORP coinciding with the Bellman optimal policy, derived from the Bellman optimality equations. Remarkably, this objective has been widely employed in previous DRL algorithms \cite{silver2014deterministic, schulman2015trust, wang2016dueling, mnih2016asynchronous, schulman2017proximal} to maximize the natural returns, lacking robust capabilities. Hence, our findings incredibly unveil that \textit{the Bellman optimal policy doubles as the ORP}, and improving the robustness of DRL agents need not compromise their performance in natural environments. This insight holds significant value for establishing DRL agents with Bellman optimality equations in real-world scenarios where strong adversarial attacks are relatively rare.

In pursuit of ORP, we delve into understanding \textit{why conventional DRL algorithms, which target the Bellman optimal policy, fail to guarantee adversarial robustness.}
By analyzing the theoretical properties of $\| Q_\theta - Q^* \|_p$ and Bellman error $\| \mathcal{T}_{B} Q_{\theta} - Q_{\theta} \|_p$ across diverse Banach spaces,
we identify the substantial impact of the parameter $p$ on adversarial robustness. Specifically, achieving ORP corresponds to minimizing the Bellman Infinity-error, \textit{i.e.}, $p = \infty$, whereas conventional algorithms are typically linked to $p=1$. 
 To address the computational challenges arising from $L^{\infty}$-norm, we propose the Consistent Adversarial Robust Deep Q-Network (CAR-DQN), utilizing a surrogate objective of Bellman Infinity-error for robust policy learning. 

To summarize, our paper makes the following contributions:
\begin{itemize}[leftmargin=0em, itemindent=2em,topsep=0em]
\setlength{\itemsep}{1pt}
\setlength{\parsep}{1pt}
\setlength{\parskip}{1pt}
    \item[(\textbf{1}).] We propose the rational CAP for SA-MDP, confirm the existence of deterministic and stationary ORP, and demonstrate its strict alignment with the Bellman optimal policy, which is a significant advancement over prior research.
    \item[(\textbf{2}).] We underscore the necessity of employing the $L^{\infty}$-norm to minimize Bellman error for theoretical ORP attainment. This stands in contrast to conventional DRL algorithms, which lack robustness due to the use of an $L^{1}$-norm.
    \item[(\textbf{3}).] We devise CAR-DQN solely utilizing a surrogate objective based on the $L^{\infty}$-norm to learn both natural return and robustness. We conduct comparative evaluations of CAR-DQN against state-of-the-art approaches across various benchmarks, validating its practical effectiveness and reinforcing our theoretical soundness.
\end{itemize}

\section{Related Work} \label{sec: preli}

\textbf{Adversarial attacks on DRL agents.}
The seminal work of \citet{huang2017adversarial} first exposed the vulnerability of DRL agents to FGSM attacks~\cite{goodfellow2014explaining} on state observations in Atari games. Subsequently, \citet{lin2017tactics, kos2017delving} introduced limited-steps attacks to deceive DRL policies. \citet{pattanaik2017robust} employed a critic action-value function and gradient descent to degrade DRL performance. 
Additionally, \citet{behzadan2017vulnerability} proposed black-box attacks on DQN and verified the transferability of adversarial examples, while \citet{inkawhich2019snooping} showed that even a constrained adversary with access only to action and reward signals could launch highly effective and damaging attacks. \citet{kiourti2020trojdrl, wang2021backdoorl, bharti2022provable, guo2023policycleanse} investigated backdoor attacks in reinforcement learning.
Besides, \citet{gleave2019adversarial} have studied the adversarial policy within multi-agent scenarios, and \citet{zhang2021robust, sun2021strongest} developed learned adversaries by training attackers as RL agents. \citet{lu2023adversarial} suggested an adversarial cheap talk setting and trained an adversary through meta-learning. 
\citet{korkmaz2023adversarial} analyzed adversarial directions in the Arcade Learning Environment, and found that even state-of-the-art robust agents~\cite{zhang2020robust, oikarinen2021robust} are susceptible to policy-independent sensitivity directions.

\textbf{Robust discrete action for DRL agents.}
Earlier works like \citet{kos2017delving, behzadan2017whatever} incorporated adversarial states into the replay buffer during training in Atari environments, leading to limited robustness. \citet{fischer2019online} proposed to separate DQN architecture into a $Q$ network and a policy network, and robustly trained the policy network with generated adversarial states and provably robust bounds. 
Recently, \citet{zhang2020robust} characterized state-adversarial RL as SA-MDP, and revealed the potential non-existence of ORP. They addressed the challenge by considering a balance between robustness and natural returns through a KL-based regularization. 
\citet{oikarinen2021robust} controlled robustness certification bounds to minimize the overlap between perturbed $Q$ values of the current action and others.
\citet{liang2022efficient} estimated the worst-case value estimation and combined it with the classic Temporal Difference (TD)-target, resulting in higher training efficiency compared to prior methods. 
The latest work by \citet{nie2023improve} built the DRL architecture upon SortNet~\cite{zhang2022rethinking}, enabling global Lipschitz continuity, thus reducing the need for training extra attackers or finding adversaries.
\citet{he2023robust} proposed robust multi-agent Q-learning to solve the robust equilibrium in discrete state and action two-player games. \citet{bukharin2023robust} considered a sub-optimal Lipschitz policy in smooth environments and extended the robustness regularization~\cite{shen2020deep, zhang2020robust} to multi-agent settings. 
Prior methods heuristically constrain local smoothness or invariance to achieve commendable robustness, while compromising natural performance. 
In contrast, our approach seeks optimal robust policies with strict theoretical guarantees, simultaneously improving natural and robust performance.

\section{Preliminaries}
\textbf{Markov decision process (MDP)} is defined by a tuple $\left( \mathcal{S}, \mathcal{A}, r, \mathbb{P}, \gamma, \mu_0 \right)$, where $\mathcal{S}$ represents the state space, $\mathcal{A}$ denotes the action space, $r: \mathcal{S} \times \mathcal{A} \rightarrow \mathbb{R}$ is the reward function, $\mathbb{P}: \mathcal{S} \times \mathcal{A} \rightarrow \Delta\left(\mathcal{S}\right)$ stands for the transition dynamics where $\Delta\left(\mathcal{S}\right)$ is the probability space over $\mathcal{S}$, $\gamma \in [0,1)$ represents the discount factor, and $\mu_0 \in \Delta\left( \mathcal{S} \right)$ is the initial state distribution. 
In this paper, we consider the setting where the continuous state space $\mathcal{S} \subset \mathbb{R}^d$ is a compact set and the action space $\mathcal{A}$ is a finite set.
Given an MDP, we define the state value function $V^\pi(s) = \mathbb{E}_{\pi,\mathbb{P}}\left[ \sum_{t=0}^{\infty} \gamma^t r(s_t,a_t) | s_0=s\right]$ and the $Q$ (or action-value) function $Q^\pi(s,a) = \mathbb{E}_{\pi,\mathbb{P}}\left[ \sum_{t=0}^{\infty} \gamma^t r(s_t,a_t) | s_0=s, a_0=a\right]$ for every policy $\pi$.
MDPs exhibit a notable property: there exists a stationary, deterministic policy that maximizes both $V^\pi(s)$ and $Q^\pi(s, a)$ for all $s\in\mathcal{S}$ and $a\in\mathcal{A}$. Additionally, the optimal $Q$ function $Q^*(s,a)=\sup_{\pi\in\Pi}Q^\pi(s,a)$, satisfies the Bellman optimality
equation
\begin{equation} \label{eq: bellman optimality equation}\notag
    Q^*(s, a)=r(s, a)+\gamma \mathbb{E}_{s^{\prime} \sim \mathbb{P}(\cdot \mid s, a)}\left[\max _{a^{\prime} \in \mathcal{A}} Q^*\left(s^{\prime}, a^{\prime}\right)\right] .
\end{equation}

\textbf{State-Adversarial MDP (SA-MDP)} allows an adversary $\nu$ 
to perturb the observed state $s$ into $s_\nu \in B(s)$, where $B(s)$ is the adversary perturbation set. Let $\pi\circ \nu$ denote the policy under perturbations, the adversarial value and $Q$ functions are
$V^{\pi\circ \nu}(s) = \mathbb{E}_{\pi\circ \nu,\mathbb{P}}\left[ \sum_{t=0}^{\infty} \gamma^t r(s_t,a_t) | s_0=s\right]$ and $Q^{\pi\circ \nu}(s,a) = \mathbb{E}_{\pi\circ \nu,\mathbb{P}}\left[ \sum_{t=0}^{\infty} \gamma^t r(s_t,a_t) | s_0=s, a_0=a\right]$, respectively. There always exists a strongest adversary $\nu^*(\pi) = \mathop{\arg\min}_\nu V^{\pi\circ \nu}$ for any policy $\pi$. An optimal robust policy (ORP) $\pi^*$ should satisfy $V^{\pi^*\circ \nu^*(\pi^*)}(s) = \max_\pi V^{\pi\circ \nu^*(\pi)}(s)$ for all $s\in \mathcal{S}$.

\section{Optimal Adversarial Robustness} 

In this section, we delve into the exploration of ORP within SA-MDP. While \citet{zhang2020robust} noted that ORP does not necessarily exist for a general adversary, we discover that only a few states lack ORP, and the set of these abnormal states has a measure close to zero in complicated tasks. For theoretical clarity, we first propose a consistency assumption of policy (CAP) to eliminate these states. 
Then, we devise a novel consistent adversarial
robust operator $\mathcal{T}_{car}$ for computing the adversarial $Q$ function, and under CAP, we identify its fixed point as exactly $Q^*$, thereby proving the existence of a deterministic and stationary ORP.


\subsection{Consistency Assumption of Policy}

Given a general adversary, we observe the true state $s$ and the perturbed observation $s_\nu$ have the same optimal action in practice. Some examples are shown in Figure \ref{fig:intrisic state} and Appendix \ref{app: instrinsic state}. This enlightens Bellman optimal policy is robust, motivating us to consider how the adversary affects the optimal action in theory. We assume that the adversary perturbation set is a $\epsilon$-neighbourhood $B_\epsilon(s) = \left\{ \|s^\prime - s\| \le \epsilon \right\}$ for convenience of description and first define the intrinsic state neighborhood where the optimal action is consistent.
\begin{definition}[Intrinsic State Neighborhood]
    \label{defn: intrinsic state neighborhood}
    Given an SA-MDP, we define the intrinsic state $\epsilon$-neighbourhood for any state $s$ as 
\begin{equation}\notag
    \begin{aligned}
        B^*_\epsilon(s) := &\left\{ s^{\prime} \in \mathcal{S} | s^\prime \in B_\epsilon(s), \right.\\ 
        & \mathop{\arg\max}_a Q^*(s^{\prime},a) =   \mathop{\arg\max}_a Q^*(s,a) \}.
    \end{aligned}
\end{equation} 
\end{definition}
Further, we characterize the states where the state neighborhood is distinct from the intrinsic one and find their little in real environments, which lays the groundwork for the consistency assumption of policy we develop later.
\begin{theorem}[Rationality of CAP] \label{thm: consistency assumption reasonable}
     Given $\epsilon>0$, let $\mathcal{S}_{nu} = \left\{ s\in\mathcal{S} | \mathop{\arg\max}_a Q^*(s,a) \text{ is not a singleton} \right\}$, and $ \mathcal{S}_{nin} = \left\{ s\in\mathcal{S} |  B_\epsilon(s) \neq B^*_\epsilon(s) \right\}$. If $Q^*(\cdot,a)$ is continuous almost everywhere in $\mathcal{S}$ for all $ a\in\mathcal{A}$, we have that $\mathcal{S}_{nin} \subseteq \mathcal{S}_{nu} \cup \mathcal{S}_{0} + B_\epsilon $, where $\mathcal{S}_{0}$ is a zero measure set. 
\end{theorem}
Actually, $\mathcal{S}_{nu}$ is also close to an empty set in most practical complex environments, and $\mathcal{S}_{0}$ is a set of special and rare discontinuous points of $Q^*$ (as shown in our proof). Theorem \ref{thm: consistency assumption reasonable} essentially shows, for complicated tasks, $\mathcal{S}_{nin}$ is a quite small set and the magnitude of $\mu(\mathcal{S}_{nin})$ is around $O(\epsilon^d)$, where $\mu(A)$ represents the measure of set $A$ and $d$ is the dimension of state space. The Corollary in Appendix~\ref{app: reasonable of C assumption} illustrates better conclusions with stronger conditions.

Motivated by Theorem~\ref{thm: consistency assumption reasonable} and the above analysis, we assume that all states have a consistent intrinsic state neighborhood.

\begin{assumption}[Consistency Assumption of Policy (CAP)] 

\label{ass: consistency assumption}

    For all $s\in\mathcal{S}$, its adversary $\epsilon$-perturbation set is the same as the intrinsic state $\epsilon$-neighbourhood, \textit{i.e.},
    $B_\epsilon(s) = B^*_\epsilon(s)$.
\end{assumption}

\begin{figure}
    \centering
\includegraphics[width=\columnwidth]{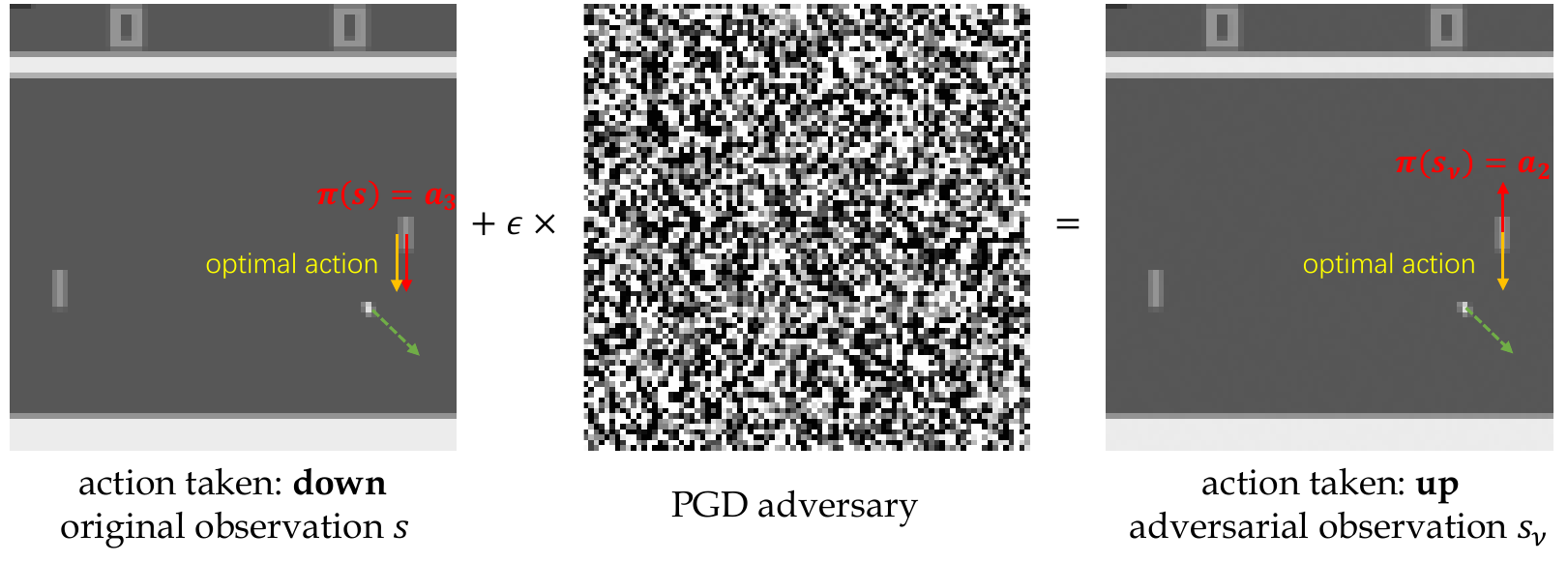} 
    \vspace{-1.7em}
    \caption{An example of state adversary in DQN.   While the adversary disrupts the policy performed by DQN, it does not impact the optimal action dictated by the Bellman optimal policy. This observation prompts the study of two key issues: whether the Bellman optimal policy serves as the ORP, and why vanilla DQN trained with Bellman error fails to ensure robustness.
    }
    \label{fig:intrisic state}
\end{figure}

\subsection{Consistent Optimal Robust Policy}

To establish the relation between the optimal $Q$ function before and after the perturbation,
we propose a consistent adversarial robust (CAR) operator.

\begin{definition}[CAR Operator $\mathcal{T}_{car}$]\label{thm: equivalence}
    Given an SA-MDP, the CAR operator is   $\mathcal{T}_{car}: L^p\left( \mathcal{S}\times\mathcal{A} \right) \rightarrow L^p\left( \mathcal{S}\times\mathcal{A} \right)$,
    \begin{equation} \notag
        \begin{aligned}
            &\left( \mathcal{T}_{car} Q \right) (s,a) = r(s,a) \\
            &+ \gamma \mathbb{E}_{ s^\prime \sim \mathbb{P}(\cdot|s,a)} \left[ \min _{s^\prime_\nu \in B_\epsilon(s^\prime)} Q \left(s^\prime,\mathop{\arg\max}_{a_{s^\prime_\nu}} Q\left(s^\prime_\nu, a_{s^\prime_\nu}\right)\right) \right].
        \end{aligned}
    \end{equation}
    
\end{definition}
 
Although $\mathcal{T}_{car}$ is not contractive (see Appendix \ref{app: not a contraction}), Theorem \ref{thm: fixed point} shows that under CAP, $\mathcal{T}_{car}$ has a fixed point, which corresponds to the optimal adversarial $Q$ function. 

\begin{theorem}[Relation between $Q^*$ and $Q^{\pi^*\circ \nu^*(\pi^*)}$]\label{thm: fixed point} \
        \begin{itemize}[leftmargin=0em, itemindent=2em,topsep=-0.1em,parsep=-0.1em]
            \item[(\textbf{1}).] If the optimal adversarial $Q$ function $Q^{\pi^*\circ \nu^*(\pi^*)}$ exists for all $s\in\mathcal{S}$ and $a\in\mathcal{A}$, then it is the fixed point of $\mathcal{T}_{car}$.
            \item[(\textbf{2}).] If the CAP holds, then $Q^*$ is the fixed point of $\mathcal{T}_{car}$ and it is also the optimal adversarial $Q$ function, \textit{i.e.}, $Q^*(s,a) = Q^{\pi^*\circ \nu^*(\pi^*)}(s,a)$, for all $s\in\mathcal{S}$ and $a\in\mathcal{A}$.
        \end{itemize}
    
\end{theorem}
\begin{remark}
    Note that the min and max operations in Definition~\ref{thm: equivalence} are not a normal minimax problem because the minimum and maximum objectives are different. It is a bilevel optimization problem. The min and max can not swapped under the general setting. However, they can swap when $\mathop{\arg\max}_{a_{s^\prime_\nu}} Q\left(s^\prime_\nu, a_{s^\prime_\nu}\right)$ is a singleton for all $s^\prime_\nu \in B_\epsilon(s^\prime)$, which is a mild condition in our training. Further, we validate that $Q^*$ is also the fixed point of the operator after swapping.
\end{remark}
We have demonstrated the convergence of $\mathcal{T}_{car}$ in a smooth environment (see Appendix \ref{app:convergence}), stating the fixed point iterations of $\mathcal{T}_{car}$ at least converge to a sub-optimal solution close to $Q^*$.
On this basis, it can be derived from Theorem \ref{thm: fixed point} that the greedy policy $\pi^*(s):=\mathop{\arg\max}_a Q^*(s,a)$, for all $s\in \mathcal{S}$, is exactly the ORP.

\begin{corollary}[Existence of ORP]
    If the CAP holds, there exists a deterministic and stationary policy $\pi^*$ which satisfies $V^{\pi^*\circ \nu^*(\pi^*)}(s) \ge V^{\pi\circ \nu^*(\pi)}(s)$ and $Q^{\pi^*\circ \nu^*(\pi^*)}(s, a) \ge Q^{\pi\circ \nu^*(\pi)}(s, a)$, for all $\pi\in\Pi$, $s\in\mathcal{S}$ and $a\in\mathcal{A}$.
\end{corollary}
The above theorems indicate that under the CAP, the ORP against the strongest adversary is actually the Bellman optimal policy derived from the Bellman optimality equations. This also suggests that the objectives in natural and adversarial environments are aligned, which is supported by our experiment results in Sec.~\ref{sec: Comparison}.

\section{Policy Robustness under Bellman $p$-error}
As the Bellman optimal policy is ORP, we further consider the reasons for the vulnerability of DRL agents: Although former methods, such as Q-learning, essentially take the Bellman optimal policy as a goal, why do they exhibit rather poor robustness? We approach this issue by examining the stability of policy across diverse Banach spaces.

Let $Q_\theta$ denote a parameterized $Q$ function. The value-based RL training theoretically requires minimizing  $\left\|Q_\theta - Q^*\right\|_{\mathcal{B}}$, where $\mathcal{B}$ is a Banach space. 
In practice, it is hard to make the distance between $Q_\theta$ and $Q^*$ vanish due to some limitations, such as the characterization capabilities of neural networks and the convergence of optimization algorithms. Therefore, we analyze the properties of $Q_\theta$ when the $\left\|Q_\theta - Q^*\right\|_{\mathcal{B}}$ is a small positive value on different spaces $\mathcal{B}$. 

\subsection{Necessity of $L^\infty$-norm for Adversarial Robustness}

We study the adversarial robustness of the greedy policy derived from $Q$ when $0 < \left\|Q - Q^*\right\|_{p}=\delta \ll 1$ for different $L^p$ spaces.
Given a function $Q$ and perturbation budget $\epsilon$, let $\mathcal{S}^Q_{sub} = \{ s | Q^*(s,\mathop{\arg\max}_a Q(s, a)) < \max_a Q^*(s, a) \}$ denote the set of states where the greedy policy according to $Q$ is suboptimal, and $\mathcal{S}^Q_{adv}$ denote the set of states in whose $\epsilon$-neighbourhood there exists the adversarial state, \textit{i.e.},
\begin{equation}\notag
    \begin{aligned}
        \mathcal{S}^Q_{adv} = &\{ s | \exists s_\nu \in B_\epsilon(s),\\
        &\text{s.t. } Q^*(s,\mathop{\arg\max}_a Q(s_\nu,a)) < \max_a Q^*(s,a) \}.        
    \end{aligned}
\end{equation}

\begin{theorem}[Necessity of $L^\infty$-norm]\label{thm:necessity of infty norm}
    There exists an MDP instance such that the following statements hold. 
        \begin{itemize}[leftmargin=0em, itemindent=2em,topsep=-0.1em,parsep=-0.1em]
            \item[(\textbf{1}).] For any $1\le p<\infty$ and $\delta>0$, there exists a function $Q\in L^p\left( \mathcal{S}\times\mathcal{A} \right)$ satisfying $\|Q-Q^*\|_{p} \leq \delta$ such that $\mu\left(\mathcal{S}^Q_{sub}\right) = O(\delta)$ yet $\mu\left( \mathcal{S}^Q_{adv} \right) =\mu \left(\mathcal{S}\right)$.
            \item[(\textbf{2}).] There exists a $\bar{\delta}>0$ such that for any $0< \delta \le \bar{\delta}$, and any function $Q\in L^\infty\left( \mathcal{S}\times\mathcal{A} \right)$ satisfying $\|Q-Q^*\|_{\infty} \leq  \delta$, we have that $\mu\left(\mathcal{S}^Q_{sub}\right) = O(\delta)$ and $\mu\left( \mathcal{S}^Q_{adv} \right) = 2\epsilon + O\left( \delta \right)$.
        \end{itemize}
\end{theorem}
The first statement indicates that when $\|Q-Q^*\|_p$ is a small value for $1\le p<\infty$, there always exist adversarial examples near almost all states, resulting in quite poor robustness, while the policy might exhibit excellent performance in a natural environment without adversarial attacks.
This observation sheds light on the vulnerability of DRL agents, aligning with findings across various studies~\cite{huang2017adversarial, ilahi2021challenges}. Importantly, the second statement points out that through minimizing $\|Q-Q^*\|$ in the $L^\infty$-norm space, we can avoid the vulnerability and attain a policy with both natural and robust capabilities. This inspires to optimize $\|Q_{\theta}-Q^*\|_{\infty}$ in DRL algorithms. Intuitive examples of Theorem \ref{thm:necessity of infty norm} are shown in Figure \ref{fig:example}.


\begin{figure}[!t]
\centering
\includegraphics[width=0.99\columnwidth]{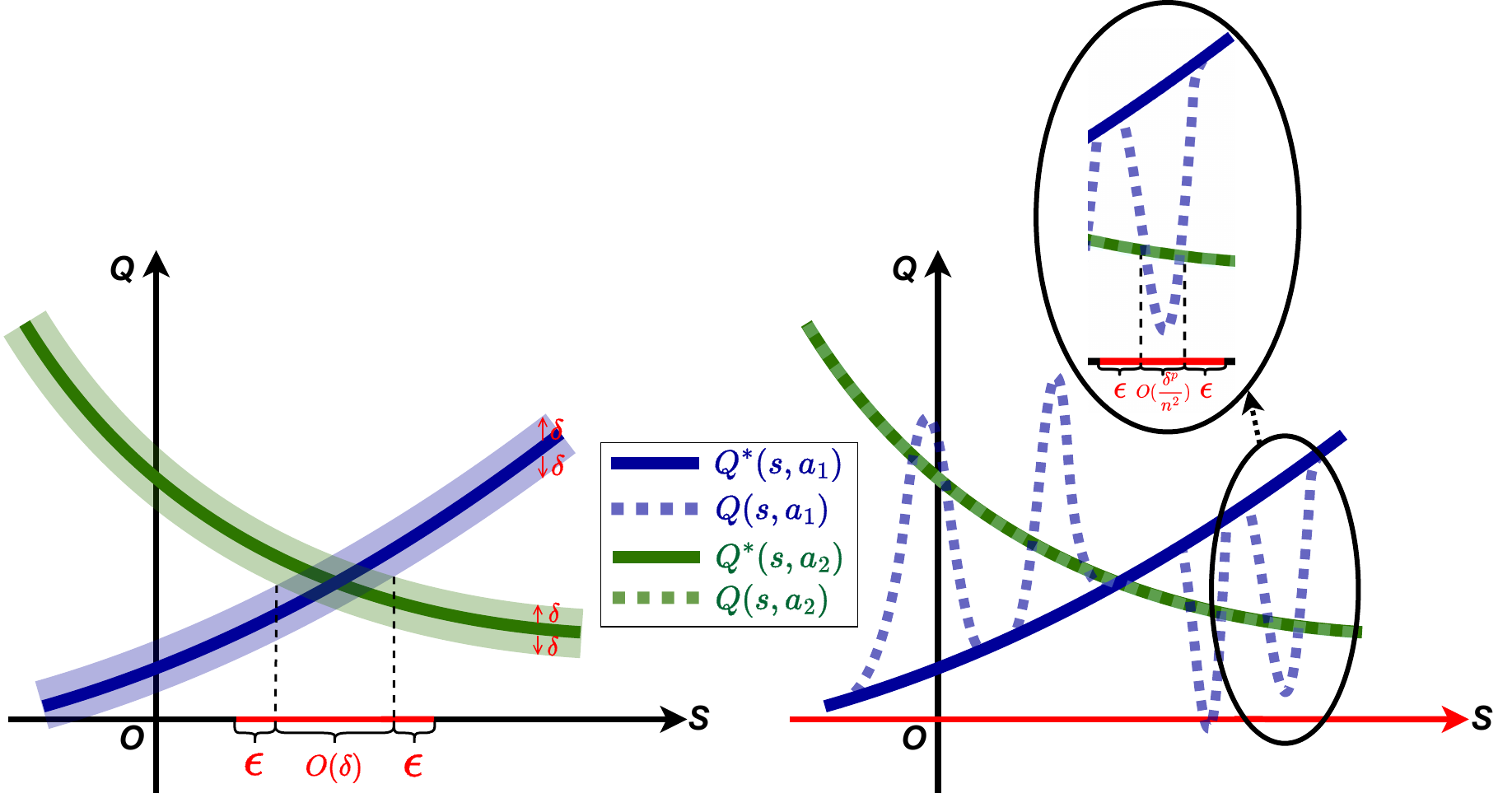}
\vspace{-1em}
\caption{
Examples of adversarial robustness for $Q$ satisfying $\|Q-Q^*\|_p\le\delta$.
Given a perturbation radius $\epsilon$, the red line represents the set $\mathcal{S}_{adv}^Q$, in which states have adversarial states. The left panel depicts the case of $p=\infty$, where all $Q$ is distributed in the shadow area with the measure of $\mathcal{S}_{adv}^Q$ being a small value $2 \epsilon + O\left( \delta \right)$. 
 The right panel shows that for $1\le p<\infty$, there always exists $Q$ such that $\mathcal{S}_{adv}^Q = \mathcal{S}$, indicating poor robustness.
 }
 \label{fig:example}
\end{figure}
\subsection{Stability of Bellman Optimality Equations}

Unfortunately, it is infeasible to measure $\|Q_{\theta}-Q^*\|$ within a practical DRL procedure due to the unknown nature of $Q^*$. Most methods train $Q_\theta$ via optimizing the Bellman error $\| \mathcal{T}_B Q_\theta - Q_\theta\|_{\mathcal{B}'}$, where $\mathcal{T}_B$ is the Bellman optimal operator  
$$\left(\mathcal{T}_B Q\right) (s, a)=r(s, a)+\gamma \mathbb{E}_{s^{\prime} \sim \mathbb{P}(\cdot \mid s, a)}\left[\max _{a^{\prime} \in \mathcal{A}} Q\left(s^{\prime}, a^{\prime}\right)\right] .$$
Similar to Theorem~\ref{thm:necessity of infty norm}, we need to figure out which Banach space $\mathcal{B}^\prime$ is the best to train DRL agents which can keep the fewest adversarial states. 
To analyze this issue, we introduce a concept of functional equations stability drawing on relevant research about physics-informed neural networks for partial differential equations~\cite{wang20222}.

\begin{definition}[Stability of Functional Equations]\label{def: stability}
    Given two Banach spaces $\mathcal{B}_1$ and $\mathcal{B}_2$, if there exist $\delta>0$ and $C>0$ such that for all $Q\in \mathcal{B}_1 \cap \mathcal{B}_2$ satisfying $\|\mathcal{T}Q - Q\|_{\mathcal{B}_1} < \delta$, we have that $\|Q - Q^*\|_{\mathcal{B}_2} < C \|\mathcal{T}Q - Q\|_{\mathcal{B}_1}$, where $Q^*$ is the exact solution of this functional equation, then we say a nonlinear functional equation $\mathcal{T}Q = Q$ is $\left( \mathcal{B}_1, \mathcal{B}_2 \right)$-stable. For simplicity, we call that functional $\mathcal{T}$ is $\left( \mathcal{B}_1, \mathcal{B}_2 \right)$-stable.
\end{definition}

The above Definition implies that  if there exists a space $\mathcal{B}^\prime$ such that  $\mathcal{T}_{B}$ is $\left(  \mathcal{B}^\prime,  L^\infty\left( \mathcal{S}\times\mathcal{A} \right) \right)$-stable, 
then $\left\|Q_\theta - Q^*\right\|_{\infty}$ will be controlled when minimizing the Bellman error in $\mathcal{B}^\prime$ space, making DRL agents robust according to  Theorem~\ref{thm:necessity of infty norm}.(2).  
The following theorems illustrate the conditions that affect the stability of $\mathcal{T}_{B}$ and guide for selecting a suitable $\mathcal{B}^\prime$.


\begin{theorem}[Stable Properties of $\mathcal{T}_{B}$ in $L^p$ Spaces]\label{thm: stability} \
\begin{itemize}[leftmargin=0em, itemindent=2em,topsep=-0.1em,parsep=-0.1em]
            \item[(\textbf{1}).] 
    There exists an MDP such that for all $1 \le q < p\le \infty$, the Bellman optimality equations $\mathcal{T}_B Q = Q$ is not $\left(  L^q\left( \mathcal{S}\times\mathcal{A} \right),  L^p \left( \mathcal{S}\times\mathcal{A} \right) \right)$-stable.
            \item[(\textbf{2}).] For any MDP, if the following conditions hold: 
    \begin{equation}\label{eq: stability condition} \notag
        \begin{aligned}
            &C_{\mathbb{P},p}< \frac{1}{\gamma};\quad p \le q \le \infty
            ; \\
            &p \ge \max\left\{1, \frac{\log \left( \left| \mathcal{A}\right|\right) + \log \left( \mu\left( \mathcal{S} \right) \right)}{\log \frac{1}{\gamma C_{\mathbb{P},p}} } \right\},
        \end{aligned}
    \end{equation}
     where $C_{\mathbb{P},p}:= \sup_{(s,a)\in\mathcal{S}\times \mathcal{A}} \left\| \mathbb{P}(\cdot \mid s, a) \right\|_{L^{\frac{p}{p-1}}\left( \mathcal{S} \right)}$, then we have that the Bellman optimality equations $\mathcal{T}_B Q = Q$ is  $\left(  L^q\left( \mathcal{S}\times\mathcal{A} \right),  L^p\left( \mathcal{S}\times\mathcal{A} \right) \right)$-stable.
        \end{itemize}
\end{theorem}


We note that $\lim_{p\rightarrow\infty} C_{\mathbb{P},p} =1 < \frac{1}{\gamma}$ and thus the first condition holds when $p$ is large enough. The second condition suggests it is available for stability that $q$ is larger than $p$, and the last condition reveals that $p$ is relevant to the size of the state and action spaces.  Further, we have that $\mathcal{T}_{B}$ is $\left(L^\infty\left( \mathcal{S}\times\mathcal{A} \right),  L^\infty\left( \mathcal{S}\times\mathcal{A} \right) \right)$-stable and thus
we can optimize DRL agents in space $\mathcal{B}^\prime = L^\infty\left( \mathcal{S}\times\mathcal{A} \right)$  for adversarial robustness. Moreover, $\mathcal{B}^\prime$ cannot be $L^q\left( \mathcal{S}\times\mathcal{A} \right)$ for all $1\le q < \infty$ according to Theorem \ref{thm: stability}.(1).

\section{Consistent Adversarial Robust DQN}

Our theoretical analysis has revealed the feasibility of training a deep Q-network (DQN) by minimizing the Bellman error in $L^{\infty}$ space to achieve the ORP. However, the exact computation of the $L^{\infty}$-norm is intractable because of the unknown environment and continuous state space. Therefore, we introduce a surrogate objective based on the $L^{\infty}$-norm and present the Consistent Adversarial Robust deep Q-network (CAR-DQN), enhancing both the natural and robust performance of agents.

\subsection{Stability of Deep Q-learning}

Define the state-action visitation distribution as
$$d_{\mu_0}^\pi(s,a)=\mathbb{E}_{s_0\sim \mu_0}  (1-\gamma) \sum_{t=0}^{\infty} \gamma^t \operatorname{Pr}^\pi(s_t=s, a_t=a|s_0).$$
Deep Q-learning algorithms, \textit{e.g.}, DQN, leverage the following objective due to interactions with the environment:
\begin{equation}\notag
    \begin{aligned}
    \mathcal{L}(Q_{\theta};\pi_{\theta})=\mathbb{E}_{(s,a)\sim d_{\mu_0}^{\pi_{\theta}}} \left| \mathcal{T}_B Q_{\theta}(s,a) - Q_{\theta}(s,a)  \right|. 
\end{aligned}
\end{equation}
The former theoretical analysis of functional equations stability can be extended to $\mathcal{L}(Q_{\theta};\pi_{\theta})$ by integrating sampling probability into a seminorm. 

\begin{definition}[$(p,d_{\mu_0}^\pi)$-seminorm]
    Given a policy $\pi$,  $f:\mathcal{S}\times\mathcal{A}\rightarrow \mathbb{R}$ and $1\le p\le\infty$, if $d_{\mu_0}^\pi$ is a probability density function, we define the $(p,d_{\mu_0}^\pi)$-seminorm as the following, which satisfies the absolute homogeneity and triangle inequality:
    \begin{equation}\notag
        \begin{aligned}
            \left\| f \right\|_{p,d_{\mu_0}^\pi} :=
            (\int_{(s,a)\in\mathcal{S}\times\mathcal{A}} \left| d_{\mu_0}^\pi(s,a)  f(s,a) \right|^p d \mu(s,a))^{\frac{1}{p}}.
        \end{aligned}
    \end{equation}
\end{definition}

We note that $(p,d_{\mu_0}^\pi)$-seminorm will be upgraded to a norm, if $d_{\mu_0}^\pi(s,a)>0$ almost everywhere for  $(s,a)\in\mathcal{S}\times\mathcal{A}$. The deep Q-learning objective can be streamlined as $\mathcal{L}(Q_{\theta};\pi_{\theta}) = \left\|\mathcal{T}_B Q_{\theta} - Q_{\theta} \right\|_{1,d_{\mu_0}^{\pi_{\theta}}}$ based on the definition. Similar to Theorem~\ref{thm: stability}, we prove this objective cannot ensure robustness, while $(\infty,d_{\mu_0}^\pi)$-norm is necessary for adversarial robustness (see Appendix \ref{app: Stability of DQN: the Good}).
\begin{theorem}
     In the practical DQN procedure, the Bellman optimality equations $\mathcal{T}_B Q = Q$ is $(  L^{\infty,d_{\mu_0}^\pi}( \mathcal{S}\times\mathcal{A}),  L^p( \mathcal{S}\times\mathcal{A}))$-stable for all $1\le p \le \infty$, while it is not $(  L^{q,d_{\mu_0}^\pi}( \mathcal{S}\times\mathcal{A}),  L^p( \mathcal{S}\times\mathcal{A}))$-stable for all $1 \le q < p\le \infty$. 
\end{theorem}
We also investigate the stability when $d_{\mu_0}^\pi$ is a probability mass function in Appendix \ref{app:seminorm} and \ref{app:Stability of DQN: the Bad}.

\subsection{Consistent Adversarial Robust Objective}

Inspired by the theoretical analysis, we propose to train robust DQNs with
$\mathcal{L}_{car}(\theta) = \| \mathcal{T}_B Q_{\theta} - Q_{\theta}\|_{\infty,d^{\pi_{\theta}}_{\mu_0}}$, which is equal to (see Appendix \ref{app:derive car-dqn})
$$\sup_{(s,a)\in\mathcal{S}\times\mathcal{A}} d_{\mu_0}^{\pi_\theta}(s,a) \max_{s_\nu \in B_\epsilon(s)}  \left| \mathcal{T}_{B} Q_\theta(s_\nu,a) - Q_\theta(s_\nu,a) \right|,$$
where $\pi_\theta$ is a behavior policy related to $Q_\theta$ and it is usually $\epsilon$-greedy policy derived from $Q_\theta$. Since interactions with the environment in an SA-MDP happen based on the true state $s$ rather than $s_\nu$, $\mathcal{T}_{B}Q_\theta(s_\nu, a)$ cannot be directly estimated. Thus, we exploit $\mathcal{T}_{B}Q_\theta(s,a)$ to substituted it, attaining a surrogate objective $\mathcal{L}_{car}^{train}(\theta)$:
$$\sup_{(s,a)\in\mathcal{S}\times\mathcal{A}} d_{\mu_0}^{\pi_\theta}(s,a) \max_{s_\nu \in B_\epsilon(s)}  \left| \mathcal{T}_{B} Q_\theta(s,a) - Q_\theta(s_\nu,a) \right|,$$
which can bound $\mathcal{L}_{car}$, especially in smooth environments.
Denote $\mathcal{L}_{car}^{diff}(\theta)$ as the following:
$$\sup_{(s,a)\in\mathcal{S}\times\mathcal{A}} d_{\mu_0}^{\pi_\theta}(s,a) \max_{s_\nu \in B_\epsilon(s)} \left| \mathcal{T}_{B} Q_\theta(s_\nu,a) - \mathcal{T}_{B}Q_\theta(s,a) \right|.$$

\begin{theorem}[Bounding $\mathcal{L}_{car}$  with  $\mathcal{L}_{car}^{train}$]\label{thm:bound car objective} \
We have that
    \begin{equation}\notag
        \left| \mathcal{L}_{car}^{train}(\theta) -  \mathcal{L}_{car}^{diff}(\theta)\right| \le \mathcal{L}_{car}(\theta) \le \mathcal{L}_{car}^{train}(\theta) + \mathcal{L}_{car}^{diff}(\theta).
    \end{equation}    
Further, suppose the environment is $\left(L_r, L_{\mathbb{P}}\right)$-smooth and suppose $Q_\theta$ and $r$ are uniformly bounded, i.e. $\exists\ M_Q,M_r >0$ such that $\left|Q_\theta(s,a)\right| \le M_Q,\ \left|r(s,a)\right| \le M_r\ \forall s\in\mathcal{S}, a\in\mathcal{A}$. If $M:=\sup_{\theta,(s,a)\in\mathcal{S}\times\mathcal{A}} d_{\mu_0}^{\pi_\theta}(s,a) <\infty$, then we have
    \begin{equation}\notag
        \mathcal{L}_{car}^{diff}(\theta) \le C_{\mathcal{T}_{B}} \epsilon,
    \end{equation}
    where $C_{\mathcal{T}_{B}}=L_{\mathcal{T}_{B}} M$, $L_{\mathcal{T}_{B}} =  L_r + \gamma C_Q L_{\mathbb{P}}$ and $C_Q = \max\left\{ M_Q, \frac{M_r}{1-\gamma} \right\}$. The definition of $\left(L_r, L_{\mathbb{P}}\right)$-smooth environment is shown in Appendix \ref{app:convergence}.

\end{theorem}

Theorem \ref{thm:bound car objective} suggests that $\mathcal{L}_{car}^{train}(\theta)$ is a proper surrogate objective from the optimization perspective. It also provides an insight into potential instability during robust training:
If $\mathcal{L}_{car}^{train}(\theta)$ is minimized to a small value yet less than $\mathcal{L}_{car}^{diff}(\theta)$, $\mathcal{L}_{car}(\theta)$ may tend to increase or overfit.


In implementation, to fully utilize each sample in the batch, we derive the soft version $\mathcal{L}_{car}^{soft}(\theta)$ of the CAR objective (derivation seen in Appendix \ref{app:derive car-dqn}):
\begin{align}\notag
    &\sum_{i\in \mathcal{\left|B\right|}} \alpha_i \max_{s_\nu \in B_\epsilon(s_i)} \left|r_i + \gamma \max_{a^\prime} Q_{\bar{\theta}}(s_i^\prime,a^\prime) - Q_{\theta}(s_\nu,a_i)  \right|, \\
    &\text{where} \ \alpha_i = \frac{e^{\frac{1}{\lambda} \max_{s_\nu } \left|r_i + \gamma \max_{a^\prime} Q_{\bar{\theta}}(s_i^\prime,a^\prime) - Q_{\theta}(s_\nu,a_i)  \right|}}{\sum_{i\in \mathcal{\left|B\right|}} e^{\frac{1}{\lambda} \max_{s_\nu } \left|r_i + \gamma \max_{a^\prime} Q_{\bar{\theta}}(s_i^\prime,a^\prime) - Q_{\theta}(s_\nu,a_i)  \right|}}. \notag
\end{align}
$\mathcal{B}$ represents a batch of transition pairs sampled from the replay buffer. $\bar{\theta}$ is the parameter of target network and $\lambda$ is the coefficient to control the level of softness.

\section{Experiments}\label{sec:exp}

In this section, we conduct extensive comparison and ablation experiments to validate the rationality of our theoretical analysis and the effectiveness of CAR-DQN. 
\begin{table*}[!t]
\caption{Average episode rewards $\pm$ standard error of the mean over 50 episodes on baselines and CAR-DQN. The best results of the algorithm with the same type of solver are highlighted in bold. CAR-DQN (PGD) outperforms SA-DQN (PGD) in all metrics and achieves remarkably better robustness ($110\%$ higher reward) on RoadRunner. CAR-DQN (cov) outperforms baselines in a majority of cases.}
\label{table: compare}
\vskip 0.15in
\resizebox{\textwidth}{!}{%
\begin{tabular}{cc|cccc||cccc}
\hline\hline
\multicolumn{2}{c|}{\multirow{3}{*}{Model}}                                                                         & \multicolumn{4}{c||}{\textbf{Pong}}                                                                                                                                    & \multicolumn{4}{c}{\textbf{BankHeist}}                                                                                                                                    \\ \cline{3-10} 
\multicolumn{2}{c|}{}                                                                                               & \multicolumn{1}{c|}{\multirow{2}{*}{\begin{tabular}[c]{@{}c@{}}Natural\\ Reward\end{tabular}}} & PGD                     & MinBest                 & ACR     & \multicolumn{1}{c|}{\multirow{2}{*}{\begin{tabular}[c]{@{}c@{}}Natural\\ Reward\end{tabular}}} & PGD                       & MinBest                   & ACR     \\ \cline{4-6} \cline{8-10} 
\multicolumn{2}{c|}{}                                                                                               & \multicolumn{1}{c|}{}                                                                          & \multicolumn{3}{c||}{$\epsilon=1/255$}                       & \multicolumn{1}{c|}{}                                                                          & \multicolumn{3}{c}{$\epsilon=1/255$}                            \\ \hline
\multicolumn{1}{c|}{Standard}                                                                      & DQN            & \multicolumn{1}{c|}{$21.0 \pm 0.0$}                                                            & $-21.0 \pm 0.0$         & $-21.0 \pm 0.0$         & $0$     & \multicolumn{1}{c|}{$1317.2 \pm 4.2$}                                                          & $22.2 \pm 1.9$            & $0.0 \pm 0.0$             & $0$     \\ \hline
\multicolumn{1}{c|}{\multirow{2}{*}{PGD}}                                                          & SA-DQN         & \multicolumn{1}{c|}{$21.0 \pm 0.0$}                                                            & $21.0 \pm 0.0$          & $21.0 \pm 0.0$          & $0$     & \multicolumn{1}{c|}{$1248.8 \pm 1.4$}                                                          & $965.8 \pm 35.9$          & $1118.0 \pm 6.3$          & $0$     \\
\multicolumn{1}{c|}{}                                                                              & CAR-DQN (Ours) & \multicolumn{1}{c|}{$21.0 \pm 0.0$}                                                            & $21.0 \pm 0.0$          & $21.0 \pm 0.0$          & $0$     & \multicolumn{1}{c|}{\textbf{$\bf 1307.0 \pm 6.1$}}                                                 & \textbf{$\bf 1243.2 \pm 7.4$} & \textbf{$\bf 1242.6 \pm 8.4$} & $0$     \\ \hline
\multicolumn{1}{c|}{\multirow{4}{*}{\begin{tabular}[c]{@{}c@{}}Convex \\ Relaxation\end{tabular}}} & SA-DQN         & \multicolumn{1}{c|}{$21.0 \pm 0.0$}                                                            & $21.0 \pm 0.0$          & $21.0 \pm 0.0$          & $1.000$ & \multicolumn{1}{c|}{$1236.0 \pm 1.4$}                                                          & $1232.2 \pm 2.5$          & $1232.2 \pm 2.5$          & $0.991$ \\
\multicolumn{1}{c|}{}                                                                              & RADIAL-DQN     & \multicolumn{1}{c|}{$21.0 \pm 0.0$}                                                            & $21.0 \pm 0.0$          & $21.0 \pm 0.0$          & $0.898$ & \multicolumn{1}{c|}{$1341.8 \pm 3.8$}                                                          & $1341.8 \pm 3.8$          & $1341.8 \pm 3.8$          & $0.982$ \\
\multicolumn{1}{c|}{}                                                                              & WocaR-DQN      & \multicolumn{1}{c|}{$21.0 \pm 0.0$}                                                            & $21.0 \pm 0.0$          & $21.0 \pm 0.0$          & $0.979$ & \multicolumn{1}{c|}{$1315.0 \pm 6.1$}                                                          & $1312.0 \pm 6.1$          & $1312.0 \pm 6.1$          & $0.987$ \\
\multicolumn{1}{c|}{}                                                                              & CAR-DQN (Ours) & \multicolumn{1}{c|}{$21.0 \pm 0.0$}                                                            & $21.0 \pm 0.0$          & $21.0 \pm 0.0$          & $0.986$ & \multicolumn{1}{c|}{\textbf{$\bf 1349.6 \pm 3.0$}}                                                 & \textbf{$\bf 1347.6 \pm 3.6$} & \textbf{$\bf 1347.4 \pm 3.6$} & $0.974$ \\ \hline
\hline
\multicolumn{2}{c|}{\multirow{3}{*}{Model}}                                                                         & \multicolumn{4}{c||}{\textbf{Freeway}}                                                                                                                                 & \multicolumn{4}{c}{\textbf{RoadRunner}}                                                                                                                                   \\ \cline{3-10} 
\multicolumn{2}{c|}{}                                                                                               & \multicolumn{1}{c|}{\multirow{2}{*}{\begin{tabular}[c]{@{}c@{}}Natural\\ Reward\end{tabular}}} & PGD                     & MinBest                 & ACR     & \multicolumn{1}{c|}{\multirow{2}{*}{\begin{tabular}[c]{@{}c@{}}Natural\\ Reward\end{tabular}}} & PGD                       & MinBest                   & ACR     \\ \cline{4-6} \cline{8-10} 
\multicolumn{2}{c|}{}                                                                                               & \multicolumn{1}{c|}{}                                                                          & \multicolumn{3}{c||}{$\epsilon=1/255$}                       & \multicolumn{1}{c|}{}                                                                          & \multicolumn{3}{c}{$\epsilon=1/255$}                            \\ \hline
\multicolumn{1}{c|}{Standard}                                                                      & DQN            & \multicolumn{1}{c|}{$33.9 \pm 0.0$}                                                            & $0.0 \pm 0.0$           & $0.0 \pm 0.0$           & $0$     & \multicolumn{1}{c|}{$41492 \pm 903$}                                                           & $0 \pm 0$                 & $0 \pm 0$                 & $0$     \\ \hline
\multicolumn{1}{c|}{\multirow{2}{*}{PGD}}                                                          & SA-DQN         & \multicolumn{1}{c|}{$33.6 \pm 0.1$}                                                            & $23.4 \pm 0.2$          & $21.1 \pm 0.2$          & $0.250$ & \multicolumn{1}{c|}{$33380 \pm 611$}                                                           & $20482 \pm 1087$          & $24632 \pm 812$           & $0$     \\
\multicolumn{1}{c|}{}                                                                              & CAR-DQN (Ours) & \multicolumn{1}{c|}{\textbf{$\bf 34.0 \pm 0.0$}}                                                   & \textbf{$\bf 33.7 \pm 0.1$} & \textbf{$\bf 33.7 \pm 0.1$} & $0$     & \multicolumn{1}{c|}{\textbf{$\bf 49700 \pm 1015$}}                                                 & \textbf{$\bf 43286 \pm 801$}  & \textbf{$\bf 48908 \pm 1107$} & $0$     \\ \hline
\multicolumn{1}{c|}{\multirow{4}{*}{\begin{tabular}[c]{@{}c@{}}Convex \\ Relaxation\end{tabular}}} & SA-DQN         & \multicolumn{1}{c|}{$30.0 \pm 0.0$}                                                            & $30.0 \pm 0.0$          & $30.0 \pm 0.0$          & $1.000$ & \multicolumn{1}{c|}{$46372 \pm 882$}                                                           & $44960\pm 1152$           & $45226 \pm 1102$          & $0.819$ \\
\multicolumn{1}{c|}{}                                                                              & RADIAL-DQN     & \multicolumn{1}{c|}{$33.1 \pm 0.1$}                                                            & \textbf{$\bf 33.3 \pm 0.1$} & \textbf{$\bf 33.3 \pm 0.1$} & $0.998$ & \multicolumn{1}{c|}{$46224\pm 1133$}                                                           & $45990 \pm 1112$          & $46082 \pm 1128$          & $0.994$ \\
\multicolumn{1}{c|}{}                                                                              & WocaR-DQN      & \multicolumn{1}{c|}{$30.8 \pm 0.1$}                                                            & $31.0 \pm 0.0$          & $31.0 \pm 0.0$          & $0.992$ & \multicolumn{1}{c|}{$43686 \pm 1608$}                                                          & $45636 \pm 706$           & $45636 \pm 706$           & $0.956$ \\
\multicolumn{1}{c|}{}                                                                              & CAR-DQN (Ours) & \multicolumn{1}{c|}{\textbf{$\bf 33.2 \pm 0.1$}}                                                   & $33.2 \pm 0.1$          & $33.2 \pm 0.1$          & $0.981$ & \multicolumn{1}{c|}{\textbf{$\bf 49398 \pm 1106$}}                                                 & \textbf{$\bf 49456 \pm 992$}  & \textbf{$\bf 47526 \pm 1132$} & $0.760$ \\ \hline
\end{tabular}%
}
\end{table*}

\subsection{Implementation details}

\textbf{Environments.}\quad
 Following recent works~\cite{zhang2020robust, oikarinen2021robust, liang2022efficient},
we conduct experiments on four Atari video games~\cite{brockman2016openai}, including Pong, Freeway, BankHeist, and RoadRunner. These environments are characterized by high-dimensional pixel inputs and discrete action spaces. We pre-process the input images into $84\times 84$ grayscale images and normalize the pixel values to the range $[0,1]$. In each environment, agents execute an action every 4 frames, skipping the other frames without frame stacking. All rewards are clipped to the range $[-1,1]$.

\begin{figure}[t]
    \centering
    \begin{subfigure}
        \centering
\includegraphics[width=0.23\textwidth]{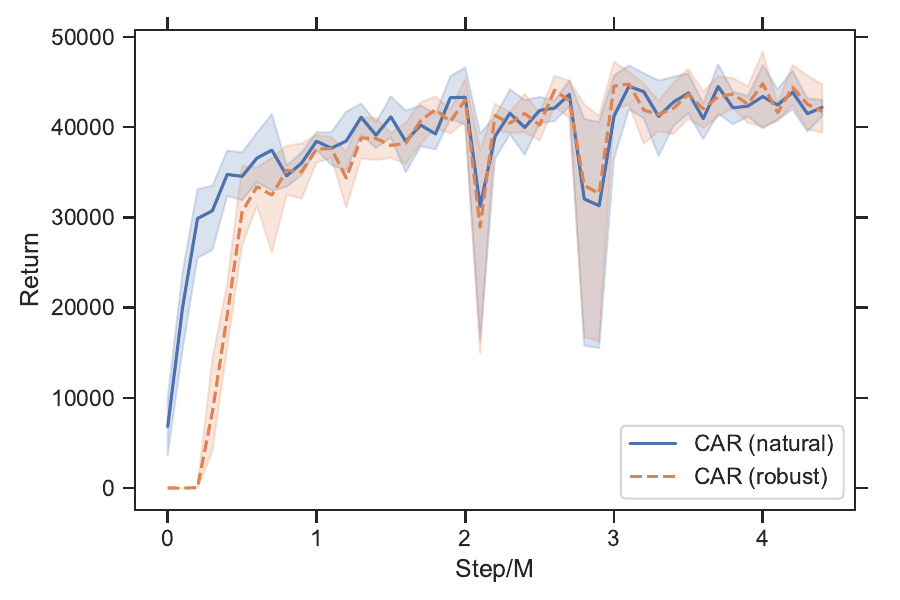} 
    \end{subfigure}
    \begin{subfigure}
        \centering
\includegraphics[width=0.23\textwidth]{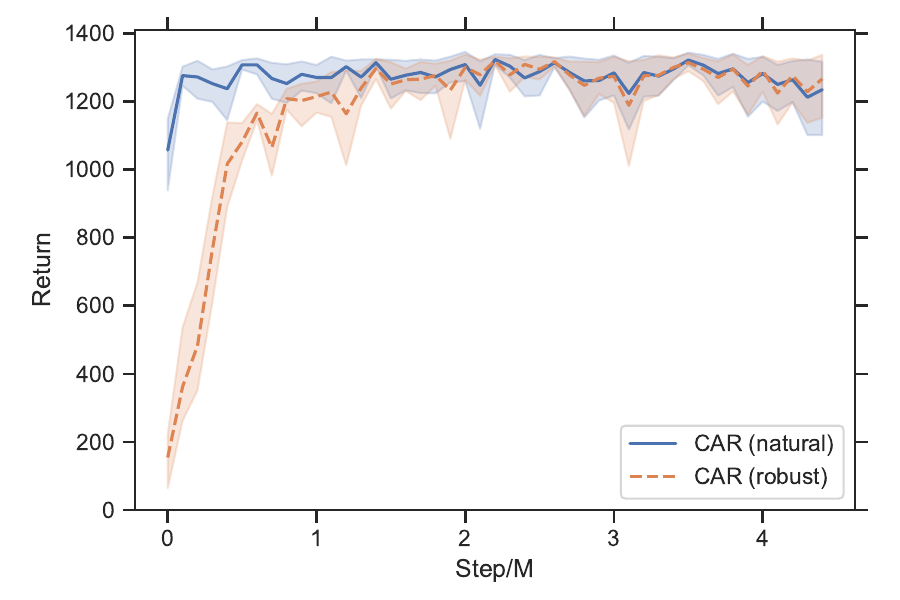}
    \end{subfigure}
    \\
    \vskip -0.1in
    \begin{minipage}{0.23\textwidth}
        \centering
        \quad \scriptsize{RoadRunner} 
    \end{minipage}
    \begin{minipage}{0.23\textwidth}
        \centering
        \quad \scriptsize{BankHeist}
    \end{minipage}
    \\
    \vskip 0.05in
    \begin{subfigure}
        \centering
\includegraphics[width=0.23\textwidth]{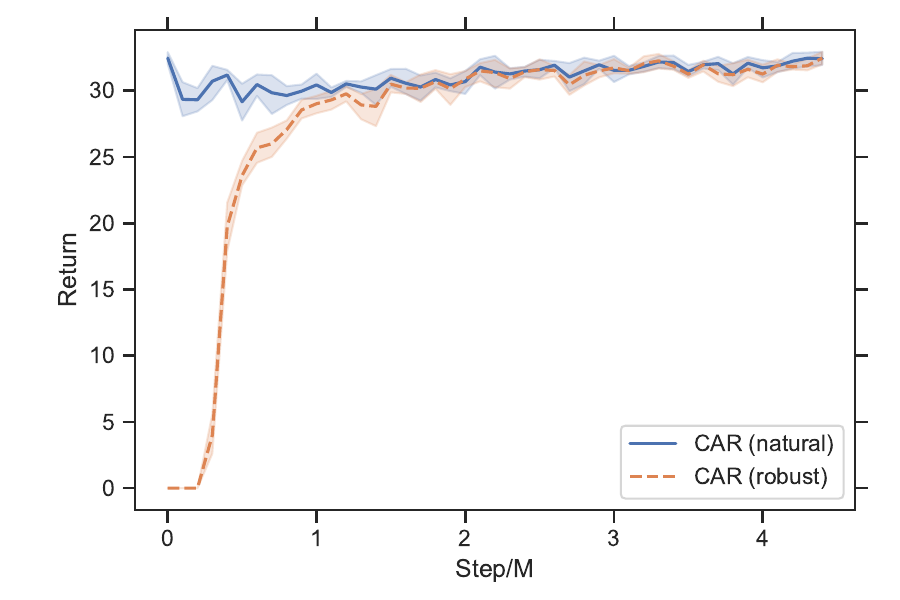}
    \end{subfigure}
    \begin{subfigure}
        \centering
\includegraphics[width=0.23\textwidth]{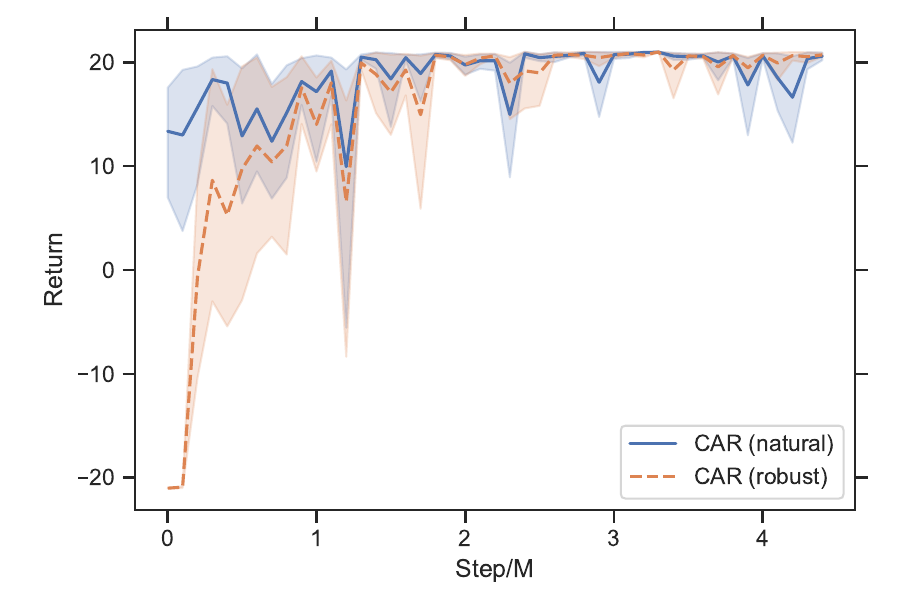}
    \end{subfigure}
    \\
    \vskip -0.1in
    \begin{minipage}{0.23\textwidth}
        \centering
        \quad \scriptsize{Freeway}
    \end{minipage}
    \begin{minipage}{0.23\textwidth}
        \centering
        \setlength{\parindent}{0.5em}
        \quad \scriptsize{Pong}
    \end{minipage}
    \vskip -0.05in
    \caption{Episode rewards of CAR-DQN agents with and without 10-step PGD attacks on 4 Atari games and 5 random seeds. As evidenced by the overlap of the two curves,  CAR-DQN achieves the consistency between Bellman optimal policy and ORP. }
    \label{fig:consitency}
    \vskip -0.1in
\end{figure}

\textbf{Baselines.}\quad
We compare CAR-DQN with several state-of-the-art robust training methods. SA-DQN~\cite{zhang2020robust} incorporates a KL-based regularization and solves the inner maximization problem using PGD~\cite{madry2017towards} and CROWN-IBP~\cite{zhang2019towards}, respectively. RADIAL-DQN~\cite{oikarinen2021robust} applies adversarial regularizations based on robustness verification bounds computed by IBP~\cite{gowal2018effectiveness}. We utilize the official code of SA-DQN and RADIAL-DQN and replicate WocaR-DQN, as its official implementation uses a different environment wrapper compared to SA-DQN and RADIAL.

\textbf{Evaluations.}\quad
We evaluate the robustness of agents using three metrics on Atari games: (1) episode return under a 10-steps untargeted PGD attack~\cite{madry2017towards}, (2) episode return under MinBest~\cite{huang2017adversarial}, both of which minimize the probability of selecting the learned best action, and (3) Action Certification Rate (ACR)~\cite{zhang2020robust}. ACR uses relaxation bounds to estimate the percentage of frames where the learned best action is guaranteed to remain unchanged during rollouts under attacks.


\textbf{CAR-DQN.}\quad
We implement CAR-DQN based on Double Dueling DQN~\cite{van2016deep,wang2016dueling} and train all baselines and CAR-DQN for 4.5 million steps, based on the same standard model released by~\citet{zhang2020robust}, which has been trained for 6 million steps. We increase the attack $\epsilon$ from $0$ to $1/255$ in the first 4 million steps using the same smoothed schedule as in~\citet{zhang2020robust, oikarinen2021robust, liang2022efficient}, and then continue training with a fixed $\epsilon$  for the remaining 0.5 million steps. We use Huber loss to replace the absolute value function and separately apply classic gradient-based methods (PGD) and cheap convex relaxation (IBP) to resolve the inner optimization in $\mathcal{L}_{car}^{soft}(\theta)$.
For CAR-DQN with PGD solver, hyper-parameters are the same as SA-DQN~\cite{zhang2020robust}. For CAR-DQN with IBP solver, we update the target network every $2000$ steps, and set learning rate as $1.25\times 10^{-4}$, batch size as $32$, exploration $\epsilon_{exp}$-end as $0.01$, soft coefficient $\lambda=1.0$ and discount factor as $0.99$. We use a replay buffer with a capacity of $2\times 10^{5}$ and Adam optimizer~\cite{kingma2014adam} with $\beta_1=0.9$ and $\beta_2=0.999$.

\subsection{Comparison Results}\label{sec: Comparison}

\begin{figure*}[!t]
    \centering
    \begin{minipage}{0.238\textwidth}
        \centering
        \quad \scriptsize{RoadRunner}
    \end{minipage}
    \begin{minipage}{0.238\textwidth}
        \centering
        \quad \scriptsize{BankHeist}
    \end{minipage}
    \begin{minipage}{0.238\textwidth}
        \centering
        \setlength{\parindent}{0.5em}
        \quad \scriptsize{Freeway}
    \end{minipage}
    \begin{minipage}{0.238\textwidth}
        \centering
        \setlength{\parindent}{1em}
        \quad \scriptsize{Pong}
    \end{minipage}
    \\
    \rotatebox{90}{\scriptsize{\qquad \qquad \quad Robust}}
    \begin{subfigure}
        \centering
        \includegraphics[width=0.238\textwidth]{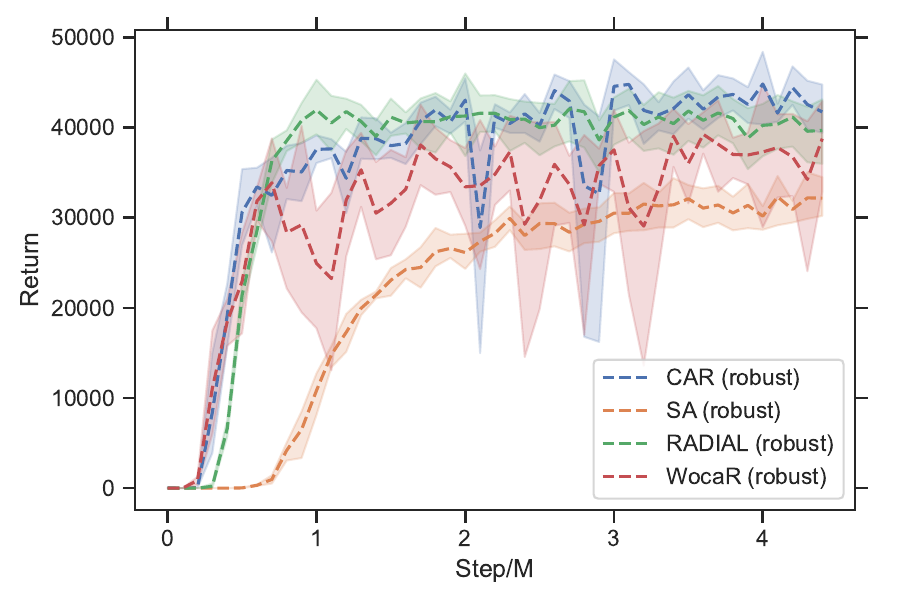}
    \end{subfigure}
    \begin{subfigure}
        \centering
        \includegraphics[width=0.238\textwidth]{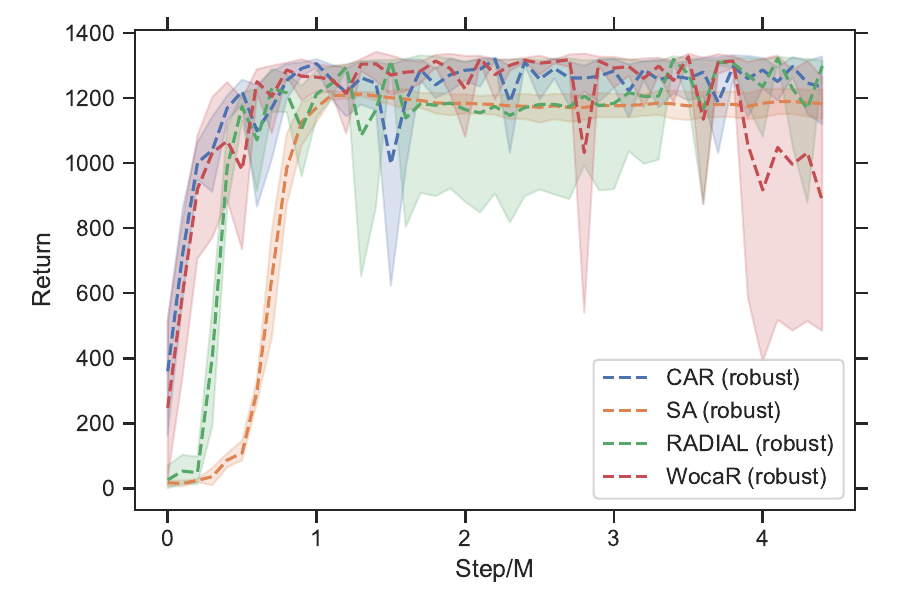}
    \end{subfigure}
    \begin{subfigure}
        \centering
        \includegraphics[width=0.238\textwidth]{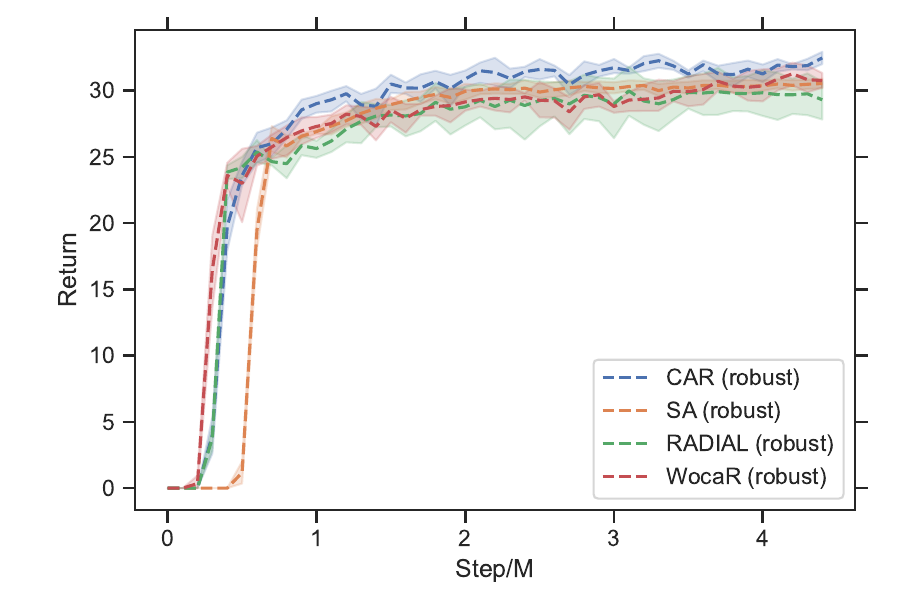}
    \end{subfigure}
    \begin{subfigure}
        \centering
        \includegraphics[width=0.238\textwidth]{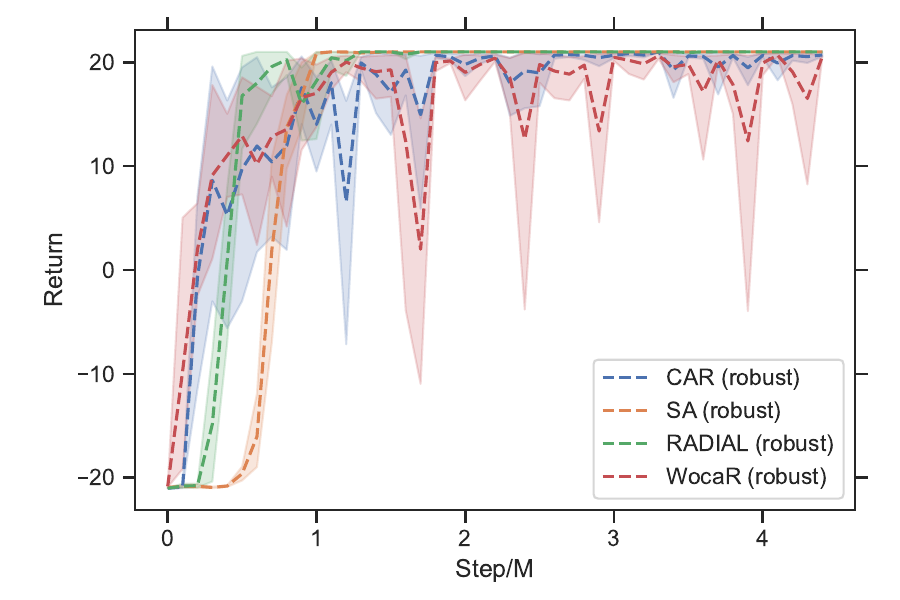}
    \end{subfigure}\\
    \rotatebox{90}{\scriptsize{\qquad \qquad \quad Natural}}
    \begin{subfigure}
        \centering
        \includegraphics[width=0.238\textwidth]{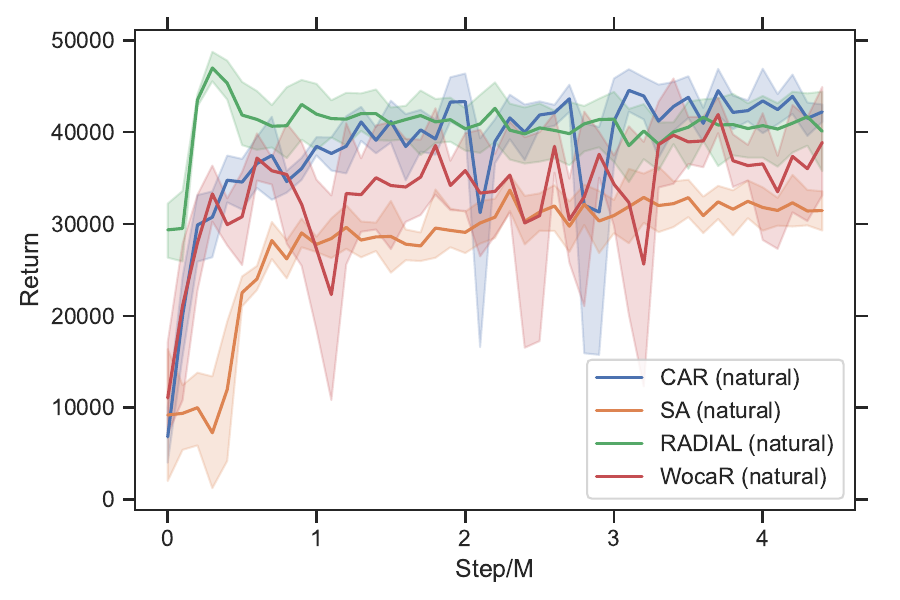}
    \end{subfigure}
    \begin{subfigure}
        \centering
        \includegraphics[width=0.238\textwidth]{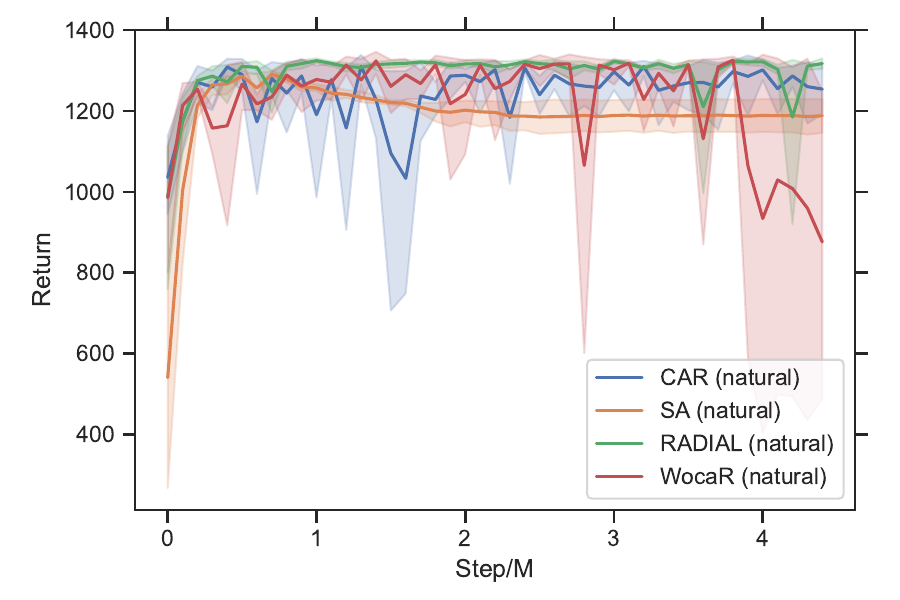}
    \end{subfigure}
    \begin{subfigure}
        \centering
        \includegraphics[width=0.238\textwidth]{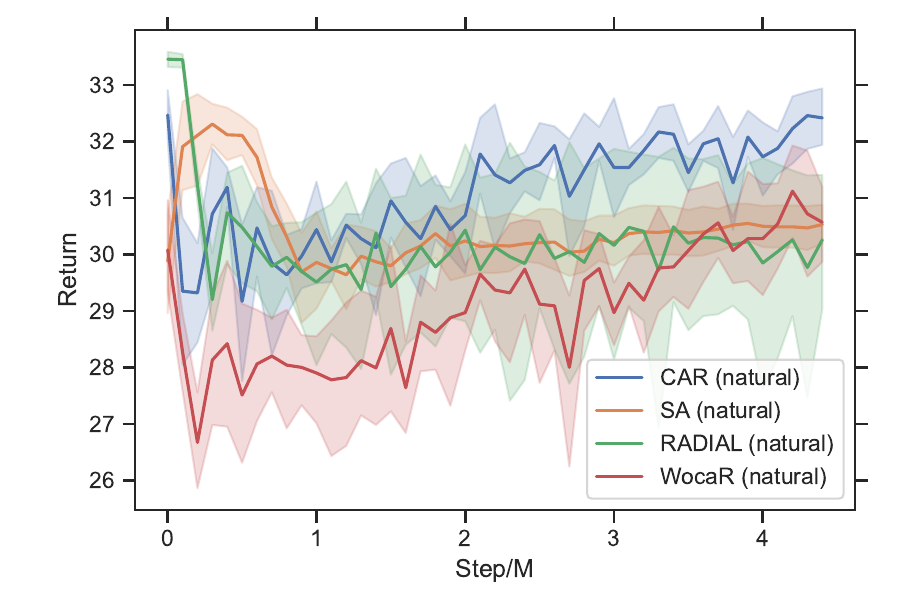}
    \end{subfigure}
    \begin{subfigure}
        \centering
        \includegraphics[width=0.238\textwidth]{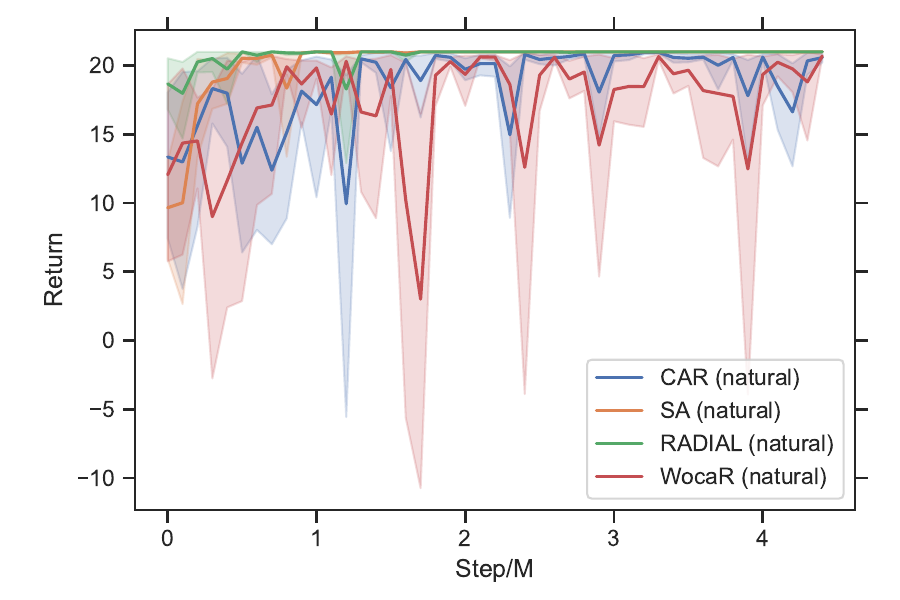}
    \end{subfigure}
    \vspace{-2em}
    \caption{Episode rewards of baselines and CAR-DQN with and without PGD attacks on 4 Atari games.
    Shaded regions are computed over 5 random seeds. CAR-DQN demonstrates superior natural and robust performance in all environments.
    }
    \label{fig: natural and robustness rewards during training}
    \vskip -0.05in
\end{figure*}

\textbf{Evaluation on benchmarks.}\quad Table \ref{table: compare} presents the natural and robust performance, with all agents trained and attacked using a perturbation radius of $\epsilon=1/255$. More results and discussion are provided in Appendix \ref{app: add exp} and \ref{app: exp with increasing perturbation radius}. Notably, CAR-DQN agents exhibit superior performance compared to baselines in the most challenging RoadRunner environment, achieving significant improvements in both natural and robust rewards. In the other three games, CAR-DQN can well match the performance of the baselines. Our proposed loss function coupled with the PGD solver, achieves a remarkable return of around 45000 on the RoadRunner environment, outperforming the SA-DQN with the PGD approach. It also attains $60\%$ higher robust rewards under the MinBest attack on the Freeway game.
In these two solvers, we observe that PGD exhibits relatively weaker robust performance compared to the convex relaxation, especially failing to ensure the ACR computed with relaxation bounds. This discrepancy can be attributed to that the PGD solver offers a lower bound surrogate function of the loss, while the IBP solver gives an upper bound.

\textbf{Consistency in natural and PGD attack returns.} \quad
Figure~\ref{fig:consitency} records the natural and PGD attack returns of CAR-DQN agents during training, showcasing a strong alignment between natural performance and robustness across all environments. This alignment validates our theory that the ORP is consistent with the Bellman optimal policy, and confirms the rationality of the proposed CAP. In addition,
Figure~\ref{fig: natural and robustness rewards during training} illustrates the natural episode return and robustness during training for different algorithms. It is worth noting that CAR-DQN agents can fast and stably converge to peak robustness and natural performance across all environments, while other algorithms exhibit unstable trends. For instance, 
the natural reward curves of SA-DQN and WocaR-DQN on BankHeist and RADIAL-DQN on RoadRunner distinctly tend to decrease, and the robust curves of SA-DQN and WocaR-DQN on BankHeist tend to decline. 
This discrepancy primarily stems from their robustness objectives, which diverge from the standard training loss and consequently result in learning sub-optimal actions. In contrast, the proposed consistent objective ensures that CAR-DQN always learns optimal actions in both natural and robust directions.

\textbf{Training efficiency.}\quad Training SA-DQN, RADIAL-DQN, WocaR-DQN, and CAR-DQN costs approximately 27, 12, 20, and 14 hours, respectively. All these models are trained for 4.5 million frames on identical hardware. Additionally, our proposed loss does not incur additional memory consumption compared to vanilla training.

\subsection{Ablation Studies}

\begin{table}[t]
\caption{Performance of DQN with different Bellman $p$-error.}
\label{table: p-error}
\vskip 0.15in
\resizebox{\columnwidth}{!}{%
\begin{tabular}{c|c|c|ccc}
\hline
Environment                 & Norm   & Natural   & PGD              & MinBest          & ACR     \\ \hline
\multirow{3}{*}{Pong}       & $L^1$      & $\bf 21.0 \pm 0.0$   & $-21.0 \pm 0.0$  & $-21.0 \pm 0.0$  & $0$     \\
                            & $L^2$      & $\bf 21.0 \pm 0.0$   & $-21.0 \pm 0.0$  & $-20.8 \pm 0.1$  & $0$     \\
                            & $L^\infty$ & $\bf 21.0 \pm 0.0$   & $\bf 21.0 \pm 0.0$   & $\bf 21.0 \pm 0.0$   & $0.985$ \\ \hline
\multirow{3}{*}{Freeway}    & $L^1$      & $\bf 33.9 \pm 0.1$   & $0.0 \pm 0.0$        & $0.0 \pm 0.0$        & $0$     \\
                            & $L^2$      & $21.8 \pm 0.3$   & $21.7 \pm 0.3$   & $22.1 \pm 0.3$   & $0$     \\
                            & $L^\infty$ & $33.3 \pm 0.1$   & $\bf 33.2 \pm 0.1$   & $\bf 33.2 \pm 0.1$   & $0.981$ \\ \hline
\multirow{3}{*}{BankHeist}  & $L^1$      & $1325.5 \pm 5.7$ & $27.0 \pm 2.0$   & $0.0 \pm 0.0$        & $0$     \\
                            & $L^2$      & $1314.5 \pm 4.0$ & $18.5 \pm 1.5$   & $22.5 \pm 2.6$   & $0$     \\
                            & $L^\infty$ & $\bf 1356.0 \pm 1.7$  & $\bf 1356.5 \pm 1.1$  & $\bf 1356.5 \pm 1.1$  & $0.969$ \\ \hline
\multirow{3}{*}{RoadRunner} & $L^1$      & $43795 \pm 1066$ & $0 \pm 0$        & $0 \pm 0$        & $0$     \\
                            & $L^2$      & $30620 \pm 990$  & $0 \pm 0$        & $0 \pm 0$        & $0$     \\
                            & $L^\infty$ & $\bf 49500 \pm 2106$ & $\bf 48230 \pm 1648$ & $\bf 48050 \pm 1642$ & $0.947$ \\ \hline
\end{tabular}%
}
\vskip -0.1in
\end{table}

\textbf{Necessity of infinity-norm.} \quad
To verify the necessity of the $(\infty, d_{\mu_0}^\pi)$-norm for adversarial robustness, we train DQN agents using the Bellman error under $(1, d_{\mu_0}^\pi)$-norm and $(2, d_{\mu_0}^\pi)$-norm, respectively. We then compare their performance with our CAR-DQN, which approximates the Bellman error under $(\infty, d_{\mu_0}^\pi)$-norm. As shown in Table \ref{table: p-error}, all agents perform well without attacks in the four games. However, the performance of $(1, d_{\mu_0}^\pi)$-norm and $(2, d_{\mu_0}^\pi)$-norm agents highly degrades under strong attacks, receiving episode rewards close to the lowest in each game. These empirical results are highly consistent with Theorem \ref{thm:necessity of infty norm}.

\begin{figure}[t]
    \centering
    \includegraphics[width=\columnwidth]{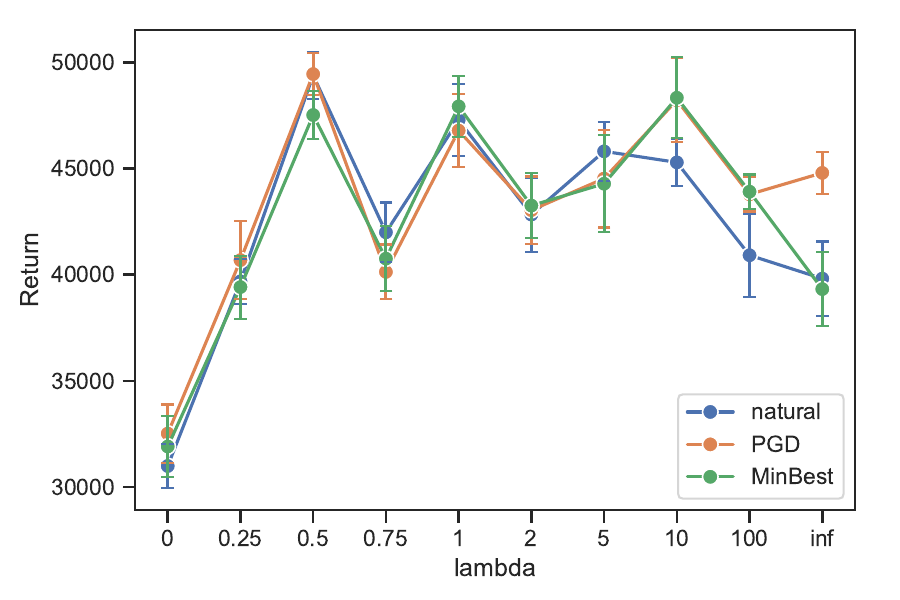}
    \vspace{-2em}
    \caption{Natural, PGD attack, and MinBest attack rewards of CAR-DQN with different soft coefficients on RoadRunner game.}
    \vspace{-1em}
    \label{fig:soft roadrunner}
\end{figure}

\begin{table}[t]
\caption{Ablation studies for soft coefficients on 4 Atari games.}
\label{table:soft coeff}
\vskip 0.15in
\resizebox{\columnwidth}{!}{%
\begin{tabular}{c|c|c|ccc}
\hline
Environment                 & $\lambda$ & Natural    & PGD               & MinBest           & ACR     \\ \hline
\multirow{3}{*}{Pong}       & $0$       & $\bf 21.0 \pm 0.0$    & $\bf 21.0 \pm 0.0$    & $\bf 21.0 \pm 0.0$    & $0.972$ \\
                            & $1$       & $\bf 21.0 \pm 0.0$    & $\bf 21.0 \pm 0.0$    & $\bf 21.0 \pm 0.0$    & $0.985$ \\
                            & $\infty$  & $20.6 \pm 0.1$   & $20.7 \pm 0.1$    & $20.7 \pm 0.1$    & $0.980$ \\ \hline
\multirow{3}{*}{Freeway}    & $0$       & $31.6 \pm 0.2$    & $31.5 \pm 0.1$    & $31.5 \pm 0.1$    & $0.966$ \\
                            & $1$       & $\bf 33.3 \pm 0.1$   & $\bf 33.2 \pm 0.1$   & $\bf 33.2 \pm 0.1$   & $0.981$ \\
                            & $\infty$  & $31.5 \pm 0.1$    & $30.9 \pm 0.3$    & $31.2 \pm 0.2$    & $0.967$ \\ \hline
\multirow{3}{*}{BankHeist}  & $0$       & $1307.5 \pm 11.0$ & $1288.5 \pm 14.0$ & $1284.0 \pm 13.8$ & $0.980$ \\
                            & $1$       & $\bf 1356.0 \pm 1.7$  & $\bf 1356.5 \pm 1.1$  & $\bf 1356.5 \pm 1.1$  & $0.969$ \\
                            & $\infty$  & $1326.0 \pm 4.8$  & $1316.0 \pm 6.8$  & $1314.0 \pm 6.6$  & $0.979$ \\ \hline
\multirow{3}{*}{RoadRunner} & $0$       & $25160 \pm 802$   & $24540 \pm 760$   & $26785 \pm 617$   & $0.007$ \\
                            & $1$       & $\bf 49500 \pm 2106$ & $\bf 48230 \pm 1648$ & $\bf 48050 \pm 1642$ & $0.947$ \\
                            & $\infty$  & $40890 \pm 2075$  & $36760 \pm 1874$  & $36740 \pm 2098$  & $0.940$ \\ \hline
\end{tabular}%
}
\vskip -0.1in
\end{table}

\textbf{Effects of soft coefficient.} \quad
We validate the effectiveness of the soft CAR-DQN loss by adjusting the soft coefficient $\lambda$. We train CAR-DQN agents on the RoadRunner environment with $\lambda$ values ranging from $0$ to $\infty$. When $\lambda=0$ we utilize the sample with the largest adversarial TD-error from a batch, while $\lambda=\infty$ corresponds to averaging over all samples in a batch. It is worth noting that a small $\lambda$ may lead to numerical instability. As depicted in Figure \ref{fig:soft roadrunner}, the agents exhibit similar capabilities when $0.5 \leq \lambda \leq 10$, indicating that the learned policies are not sensitive to the soft coefficient within this range. Table \ref{table:soft coeff} displays the performance of CAR-DQN agents with $\lambda=0,1,\infty$ across four Atari environments. In the RoadRunner, the case $\lambda=0$ yields poor performance, achieving returns around 25000 due to inadequate utilization of the samples. Interestingly, utilizing only the sample with the largest adversarial TD error from a batch achieves good robustness on the other three simpler games. The case $\lambda=\infty$ results in worse robustness compared to other cases with differentiated weights. This suggests that each sample in a batch plays a distinct role in robust training, and we can enhance robust performance by specifying weightings. These results further validate the efficacy of our CAR-DQN loss.

\section{Conclusion}
In this paper, we prove the alignment of the optimal robust policy with the Bellman optimal policy under the consistency assumption of policy. We show that measuring Bellman error in differed $L^p$ spaces yields varied performance, underscoring the necessity of Bellman infinity-error for robustness. We validate these findings through experiments with CAR-DQN, which optimizes a surrogate objective of Bellman infinity-error. We believe this work contributes significantly to unveiling the nature of robustness in Q-learning.
Since our work focuses on value-based DRL with discrete action space, we will extend future research into the policy-based DRL and continuous action space setting.

\section*{Acknowledgements}

This paper is supported by the National Key R\&D (research and development) Program of China (2022YFA1004001) and the National Natural Science Foundation of China (Nos. 11991022, U23B2012). We thank Anqi Li and Wenzhao Liu for valuable feedback on earlier drafts of the paper.

\section*{Impact Statement}

Reinforcing the robustness of Deep Reinforcement Learning agents is crucial for commercial and industrial applications. Unlike previous approaches that focus on imposing smoothness or stability on neural networks through elaborated regularization techniques, this study uncovers the inherent robustness of the Bellman optimal policy under adversarial attack scenarios. We believe this discovery can offer fresh insights to the research community, potentially driving advancements in related applications. 


\bibliography{car_dqn}
\bibliographystyle{icml2024}

\newpage
\appendix
\onecolumn


\section{Theorems and Proofs of Optimal Adversarial Robustness}

\subsection{Reasonablity of the Consistency Assumption} \label{app: reasonable of C assumption}

\begin{theorem}\label{continuous a.e.}
    For any MDP $\mathcal{M}$, let $\mathcal{S}_{nu}$ denote the state set where the optimal action is not unique, i.e. $\mathcal{S}_{nu} = \left\{ s\in\mathcal{S} | \mathop{\arg\max}_a Q^*(s, a) \text{ is not a singleton} \right\}$. If $Q^*(\cdot,a)$ is continuous almost everywhere in $\mathcal{S}$ for all $a\in\mathcal{A}$, we have the following conclusions:
        \begin{itemize}
            \item For almost everywhere $s\in \mathcal{S} \setminus \mathcal{S}_{nu} $, there exists $\epsilon > 0$ such that $B_\epsilon(s) = B^*_\epsilon(s)$.
            \item Given $\epsilon > 0$, let $\mathcal{S}_{nin}$ denote the set of states where \textit{the intrinsic state $\epsilon$-neighourhood} is not the same as \textit{the $\epsilon$-neighourhood}, i.e. $\mathcal{S}_{nin} = \left\{ s\in\mathcal{S} |  B_\epsilon(s) \neq B^*_\epsilon(s) \right\}$. Then, we have $\mathcal{S}_{nin} \subseteq \mathcal{S}_{nu} \cup \mathcal{S}_{0} + B_\epsilon  = \left\{ s_1 + s_2 | s_1 \in \mathcal{S}_{nu} \cup \mathcal{S}_{0},\ \|s_2\| \le \epsilon \right\}$, where $\mathcal{S}_{0}$ is a zero measure set. 
        \end{itemize}
\end{theorem}
\begin{proof}
   (1)  Let $\mathcal{S}^{\prime}=\{ s \in \mathcal{S}  |  \exists a \in \mathcal{A},  \text{s.t. $Q^*(s,a)$  is not continuous at } s \}$. Then $\mu(\mathcal{S}^{\prime})=0$ because $Q^*(\cdot,a)$ is continuous almost everywhere in $\mathcal{S}$ for all $a\in\mathcal{A}$ and $\mathcal{A}$ is a finite discrete set. And $Q^*(s,a)$ is continuous in $\left(\mathcal{S} \setminus \mathcal{S}_{nu} \right) \setminus  \mathcal{S}^{\prime} $. Because $ \arg\max_{a}Q^*(s,a) $ is a singleton  for $s\in \left(\mathcal{S} \setminus \mathcal{S}_{nu} \right) \setminus  \mathcal{S}^{\prime}$, define $\arg\max_{a}Q^*(s,a) = \{a_s^*\} $ for any $s\in \left(\mathcal{S} \setminus \mathcal{S}_{nu} \right) \setminus  \mathcal{S}^{\prime} $ . Then $Q^*(s,a_s^*) > Q^*(s,a)$ for a fixed $a\in \mathcal{A} \setminus \{a_{s}^*\}$.  According to continuity of $Q^*(\cdot,a)$ for all $a\in \mathcal{A}$,
   there exists $ \epsilon_a >0$, such that $Q^*(s^{\prime},a_{s}^*) > Q^*(s^{\prime},a)$ for all $s^{\prime} \in  B_{\epsilon_a}(s)$. Because $\mathcal{A}$ is a finite discrete set, let $\epsilon=\min_{a\in \mathcal{A} \setminus {a_s^*}}\{\epsilon_a\}$, then $Q^*(s^{\prime},a_{s^{\prime}}^*) > Q^*(s^{\prime},a)$ for all $s^{\prime} \in  B_\epsilon(s)$  and for all $a\in \mathcal{A} \setminus \{a_{s^{\prime}}^*\}$, i.e. $B_\epsilon(s)=B_\epsilon^*(s)$.

   (2)Let 
    $\mathcal{S}_n=\{ s \in \mathcal{S}  | \forall \epsilon_1>0, \exists s^{\prime} \in B_{\epsilon_1}(s), \ \text{s.t. }    \arg\max_{a}Q^*(s^{\prime},a) \neq \arg\max_{a}Q^*(s,a) \}$ 
    and 
    $\mathcal{S}_0=\mathcal{S}_n \cap \mathcal{S}^{\prime}$. Then $S_0$ is the set of discontinuous points that cause the optimal action to change.  
    And $\mu(\mathcal{S}_0)=\mu(\mathcal{S}_n \cap \mathcal{S}^{\prime})=0$ because $\mu(\mathcal{S}^{\prime})=0$. 
    
    For any $s\in \mathcal{S}_{nin} = \left\{ s\in\mathcal{S} |  B_\epsilon(s) \neq B^*_\epsilon(s) \right\}$, we have the following two cases.

    \textbf{Case 1}. $\exists s^\prime \in B_\epsilon(s)$ s.t. $s^\prime \in \mathcal{S}_{nu}$, then $s \in \mathcal{S}_{nu}+B_\epsilon$, i.e. 
    \begin{equation}
        s \in \mathcal{S}_{nu} \cup \mathcal{S}_{0} + B_\epsilon.
        \label{th2_case1}
    \end{equation}

    \textbf{Case 2}. $\forall s^\prime \in B_\epsilon(s) , s^\prime \notin \mathcal{S}_{nu}$, which means that  $\arg\max_{a}Q^*(s^\prime,a)$ is a singleton for all $s^\prime\in B_{\epsilon}(s)$. Define $\arg\max_{a}Q^*(s^{\prime},a) = \{a_{s^\prime}^*\} $ for any $s^\prime\in B_\epsilon(s)$. 
     
     Because $s \in \mathcal{S}_{nin}$, there exist a $s^\prime \in B_\epsilon(s)$ such that $a_{s^\prime}^* \neq a_s^*$. Let $s_1$ be the point that closest to $s$  satisfing $a_{s_1}^* \neq a_s^*$, then $s_1 \in B_\epsilon(s)$. We have 
     \begin{equation}
         s_1 \in \mathcal{S}_n.
         \label{th2_case2_1}
     \end{equation}
     
     Otherwise $s_1 \notin \mathcal{S}_n$ means that $\exists \epsilon_1 >0, \forall s^\prime \in B_{\epsilon_1}(s_1)$, $a_{s^\prime}^*=a_{s_1}^*$, then $s_1$ is not the point that closest to $s$  satisfing $a_{s_1}^* \neq a_s^*$, which is a contradiction. We also have 
     \begin{equation}
         s_1\in \mathcal{S}^\prime.
         \label{th2_case2_2}
     \end{equation}
     
     Otherwise $s_1\notin \mathcal{S}^\prime$ means that $\forall a \in \mathcal{A}$, $Q^*(\cdot,a) $ is continuous in $s_1$. First, we have 
     \begin{equation}
         Q^*(s_1,a_{s_1}^*) > Q^*(s_1,a), \forall a \in \mathcal{A} \setminus \{a_{s_1}^*\}. 
     \end{equation}

     Then $\exists \epsilon_2 >0$, $\forall s \in B_{\epsilon_2}(s_1)$, s.t.  
     \begin{equation}
         Q^*(s,a_{s_1}^*) > Q^*(s,a), \forall a \in \mathcal{A} \setminus \{a_{s_1}^*\}. 
     \end{equation}
     because of the continuity of point $s_1$. This contradicts the definition of $s_1$. 

     According to (\ref{th2_case2_1}) and (\ref{th2_case2_2}), we have 
     $s_1\in \mathcal{S}^{\prime} \cap \mathcal{S}_n$ i.e. $s_1 \in \mathcal{S}_0$. Then $s\in \mathcal{S}_0 + B_\epsilon(s)$, i.e.
     \begin{equation}
          s \in \mathcal{S}_{nu} \cup \mathcal{S}_{0} + B_\epsilon.
        \label{th2_case2_3} 
     \end{equation}

     Thus 
     \begin{equation}
        \mathcal{S}_{nin} \subseteq \mathcal{S}_{nu} \cup \mathcal{S}_{0} + B_\epsilon.
     \end{equation}
    
\end{proof}
\begin{remark}
    In practical and complex tasks, we can view $\mathcal{S}_{nu}$ as an empty set.
\end{remark}

\begin{remark}
    Except for the smooth environment, many tasks can be modeled as environments with sparse rewards. Further, the value function and action-value function in these environments are almost everywhere continuous.
\end{remark}

\begin{remark}
    According to the construction in the above proof, we know that $\mathcal{S}_0$ is a set of special discontinuous points and its elements are rare in practical complex environments. 
\end{remark}

Further, we can get the following corollary in the setting of continuous functions and there are better conclusions.
\begin{corollary}\label{Th1}
    For any MDP $\mathcal{M}$, let $\mathcal{S}_{nu}$ denote the state set where the optimal action is not unique, i.e. $\mathcal{S}_{nu} = \left\{ s\in\mathcal{S} | \mathop{\arg\max}_a Q^*(s, a) \text{ is not a singleton} \right\}$. If $Q^*(\cdot,a)$ is continuous in $\mathcal{S}$ for all $a\in\mathcal{A}$, we have the following conclusions:
        \begin{itemize}
            \item For $s\in \mathcal{S} \setminus \mathcal{S}_{nu}$, there exists $\epsilon > 0$ such that $B_\epsilon(s) = B^*_\epsilon(s)$.
            \item Given $\epsilon > 0$, let $\mathcal{S}_{nin}$ denote the set of states where \textit{the intrinsic state $\epsilon$-neighourhood} is not the same as \textit{the $\epsilon$-neighourhood}, i.e. $\mathcal{S}_{nin} = \left\{ s\in\mathcal{S} |  B_\epsilon(s) \neq B^*_\epsilon(s) \right\}$. Then, we have $\mathcal{S}_{nin} \subseteq \mathcal{S}_{nu} + B_\epsilon = \left\{ s_1 + s_2 | s_1 \in \mathcal{S}_{nu},\ \|s_2\| \le \epsilon \right\}$. Especially, when $\mathcal{S}_{nu}$ is a finite set, we have $\mu\left( \mathcal{S}_{nin} \right) \le \left| \mathcal{S}_{nu} \right| \mu\left( B_\epsilon \right) = C_d \left| \mathcal{S}_{nu} \right| \epsilon^d $, where $C_d$ is a constant with respect to dimension $d$ and norm.
        \end{itemize}
\end{corollary}
    \begin{proof}
        Corollary \ref{Th1} can be derived from Theorem \ref{continuous a.e.}
        because we have the following conclusion in continuous case.
        \begin{equation}
            \mathcal{S}_{0} \subseteq \{ s \in \mathcal{S}  |  \exists a \in \mathcal{A},  \text{s.t. $Q^*(s,a)$  is not continuous at } s \}=\emptyset
        \end{equation}
    \end{proof}
\begin{remark}
    Certain natural environments show smooth reward function and transition dynamics, especially in continuous control tasks where the transition dynamics come from some physical laws. Further, the value function and action-value function in these environments is continuous.
\end{remark}

\subsection{ORP and CAR Operator}
Define the consistent adversarial robust operator for adversarial action-value function:
\begin{equation} 
    \left( \mathcal{T}_{car} Q \right) (s,a) = r(s,a)  + \gamma \mathbb{E}_{ s^\prime \sim \mathbb{P}(\cdot|s,a)} \left[ \min _{s^\prime_\nu \in B_\epsilon(s^\prime)} Q \left(s^\prime,\mathop{\arg\max}_{a_{s^\prime_\nu}} Q\left(s^\prime_\nu, a_{s^\prime_\nu}\right)\right) \right].
\end{equation}

\subsubsection{Equivalence with Optimal Adversarial Value Function}

\begin{lemma}[Bellman equations for fixed $\pi$ and $\nu$ in SA-MDP, \citet{zhang2020robust}]\label{lem:bellman equations in samdp}
    Given $\pi: \mathcal{S}\rightarrow\Delta(\mathcal{A})$ and $\nu: \mathcal{S}\rightarrow\mathcal{S}$, we have 
    \begin{align}
        V^{\pi\circ \nu}(s)&= \mathbb{E}_{a\sim \pi\left(\cdot|{\nu\left(s\right)}\right)}  Q^{\pi\circ \nu}(s,a)\\
        &= \mathbb{E}_{a\sim \pi\left(\cdot|{\nu\left(s\right)}\right)} \left[ r(s,a)  + \gamma \mathbb{E}_{ s^\prime \sim \mathbb{P}(\cdot|s,a)} V^{\pi\circ \nu}(s^\prime) \right],\\
        Q^{\pi\circ \nu}(s,a) &= r(s,a)  + \gamma \mathbb{E}_{ s^\prime \sim \mathbb{P}(\cdot|s,a)} V^{\pi\circ \nu}(s^\prime) \\
        &= r(s,a)  + \gamma \mathbb{E}_{ s^\prime \sim \mathbb{P}(\cdot|s,a), a^\prime\sim \pi\left(\cdot|{\nu\left(s^\prime\right)}\right)} Q^{\pi\circ \nu}(s^\prime,a^\prime).
    \end{align}
\end{lemma}
\begin{lemma}[Bellman equation for strongest adversary $\nu^*$ in SA-MDP, \citet{zhang2020robust}]\label{lem:bellman equation for strongest adversary}
    \begin{equation}
        V^{\pi\circ \nu^*(\pi)}(s) = \min_{\nu(s)\in B_\epsilon(s)} \mathbb{E}_{a\sim \pi\left(\cdot|{\nu\left(s\right)}\right)}Q^{\pi\circ \nu^*(\pi)}(s,a).
    \end{equation}
\end{lemma}

\begin{definition}
    Define the linear functional $\mathcal{L}^{\pi\circ \nu}: L^p\left( \mathcal{S}\times\mathcal{A} \right) \rightarrow L^p\left( \mathcal{S}\times\mathcal{A} \right)$ for fixed $\pi$ and $\nu$:
    \begin{equation}
        \left(\mathcal{L}^{\pi\circ \nu}Q\right)(s,a):=\mathbb{E}_{ s^\prime \sim \mathbb{P}(\cdot|s,a), a^\prime\sim \pi\left(\cdot|{\nu\left(s^\prime\right)}\right)} Q(s^\prime,a^\prime).
    \end{equation}
\end{definition}
Then, by lemma \ref{lem:bellman equations in samdp}, we have that
\begin{equation}
    Q^{\pi\circ \nu} = r + \gamma \mathcal{L}^{\pi\circ \nu}Q^{\pi\circ \nu}.
\end{equation}

\begin{lemma}\label{lem: bounded inverse functional}
    $\mathcal{T}:\mathcal{X} \rightarrow \mathcal{X}$ is a linear functional where $\mathcal{X}$ are normed vector space. If there exists $ m >0$ such that
    \begin{equation}
        \|\mathcal{T}x\| \ge m\|x\| \quad \forall x\in \mathcal{X},
    \end{equation}
    then $\mathcal{T}$ has a bounded inverse operator $\mathcal{T}^{-1}$.
\end{lemma}
\begin{proof}
    If  $\mathcal{T}x_1=\mathcal{T}x_2$, then $\mathcal{T}(x_1-x_2)=0$. While $0=\|\mathcal{T}(x_1-x_2)\| \geq m\|x_1-x_2\|$, thus $x_1=x_2$. Then $\mathcal{T}$ is a bijection and thus the inverse operator of $\mathcal{T}$ exists.

    For any $y\in\mathcal{X}$, $\mathcal{T}^{-1}y\in\mathcal{X}$. We have that
    \begin{align}
        \|y\| = \| \mathcal{T} \left( \mathcal{T}^{-1}y \right) \| \ge m \|  \mathcal{T}^{-1}y \|.
    \end{align}
    Thus, we attain that
    \begin{equation}
        \|  \mathcal{T}^{-1}y \| \le \frac{1}{m} \|y\|,\quad \forall y\in\mathcal{X},
    \end{equation}
    which shows that $\mathcal{T}^{-1}$ is bounded.
\end{proof}

\begin{lemma}\label{lem:invertible lemma}
    $I-\gamma \mathcal{L}^{\pi\circ \nu}$ is invertible and thus we have that
    \begin{equation}
        Q^{\pi\circ \nu} = \left(I-\gamma \mathcal{L}^{\pi\circ \nu}\right)^{-1} r.
    \end{equation}
\end{lemma}
\begin{proof}
    Firstly, for all $(s,a)\in \mathcal{S}\times\mathcal{A}$, we have
    \begin{align}
        \left(\mathcal{L}^{\pi\circ \nu}Q\right)(s,a)&=\mathbb{E}_{ s^\prime \sim \mathbb{P}(\cdot|s,a), a^\prime\sim \pi\left(\cdot|{\nu\left(s^\prime\right)}\right)} Q(s^\prime,a^\prime) \\
        &\le \mathbb{E}_{ s^\prime \sim \mathbb{P}(\cdot|s,a), a^\prime\sim \pi\left(\cdot|{\nu\left(s^\prime\right)}\right)} \left\| Q\right\|_{L^\infty\left( \mathcal{S}\times\mathcal{A} \right)} \\
        &= \left\| Q\right\|_{L^\infty\left( \mathcal{S}\times\mathcal{A} \right)}
    \end{align}
    Thus, we have that
    \begin{equation}\label{eq:nonexpansion of pi nu bellman operator}
        \left\|\mathcal{L}^{\pi\circ \nu} Q \right\|_{L^\infty\left( \mathcal{S}\times\mathcal{A} \right)} \le \left\| Q\right\|_{L^\infty\left( \mathcal{S}\times\mathcal{A} \right)}.
    \end{equation}
    For any $Q\in L^p\left( \mathcal{S}\times\mathcal{A} \right)$, we have
    \begin{align}
        \left\| \left(I-\gamma \mathcal{L}^{\pi\circ \nu}\right) Q \right\|_{L^\infty\left( \mathcal{S}\times\mathcal{A} \right)} &= \left\| Q-\gamma \mathcal{L}^{\pi\circ \nu} Q \right\|_{L^\infty\left( \mathcal{S}\times\mathcal{A} \right)}\\
        &\ge \left\| Q\right\|_{L^\infty\left( \mathcal{S}\times\mathcal{A} \right)}-\gamma \left\|\mathcal{L}^{\pi\circ \nu} Q \right\|_{L^\infty\left( \mathcal{S}\times\mathcal{A} \right)} \\
        &\ge \left\| Q\right\|_{L^\infty\left( \mathcal{S}\times\mathcal{A} \right)}-\gamma \left\| Q\right\|_{L^\infty\left( \mathcal{S}\times\mathcal{A} \right)} \\
        &= \left(1-\gamma\right)\left\| Q\right\|_{L^\infty\left( \mathcal{S}\times\mathcal{A} \right)},
    \end{align}
    where the first inequality comes from the triangle inequality and the second inequality comes from \eqref{eq:nonexpansion of pi nu bellman operator}. Then, according to lemma \ref{lem: bounded inverse functional}, we attain that $I-\gamma \mathcal{L}^{\pi\circ \nu}$ is invertible.
\end{proof}

\begin{lemma} \label{lem:nonegative}
    If $Q>0$ for all $(s,a)\in \mathcal{S}\times\mathcal{A}$, then we have that $\left(I-\gamma \mathcal{L}^{\pi\circ \nu}\right)^{-1} Q > 0$ for all $(s,a)\in \mathcal{S}\times\mathcal{A}$.
\end{lemma}
\begin{proof}
    At first, we have
    \begin{align}
        &\quad \left(I-\gamma \mathcal{L}^{\pi\circ \nu}\right) \left( \sum_{t=0}^\infty \gamma^t \left(\mathcal{L}^{\pi\circ \nu}\right)^t \right) \\
        &=\sum_{t=0}^\infty \gamma^t \left(\mathcal{L}^{\pi\circ \nu}\right)^t - \sum_{t=1}^\infty \gamma^t \left(\mathcal{L}^{\pi\circ \nu}\right)^t\\
        &= I.
    \end{align}
    Thus, we get that
    \begin{equation}
        \left(I-\gamma \mathcal{L}^{\pi\circ \nu}\right)^{-1} = \sum_{t=0}^\infty \gamma^t \left(\mathcal{L}^{\pi\circ \nu}\right)^t.
    \end{equation}
    If $Q(s,a)>0$ for all $(s,a)\in \mathcal{S}\times\mathcal{A}$, then for all $(s,a)\in \mathcal{S}\times\mathcal{A}$, we have
    \begin{equation}
        \left(\mathcal{L}^{\pi\circ \nu}Q\right)(s,a)=\mathbb{E}_{ s^\prime \sim \mathbb{P}(\cdot|s,a), a^\prime\sim \pi\left(\cdot|{\nu\left(s^\prime\right)}\right)} Q(s^\prime,a^\prime) \ge 0.
    \end{equation}
    Further, we have that $\left(\left(\mathcal{L}^{\pi\circ \nu}\right)^k Q\right) (s,a)> 0$ for all $k\in\mathbb{N}$ and $(s,a)\in \mathcal{S}\times\mathcal{A}$. Thus, we have 
    \begin{align}
        &\quad \left(I-\gamma \mathcal{L}^{\pi\circ \nu}\right)^{-1} Q (s,a) \\
        &=\sum_{t=0}^\infty \gamma^t \left(\left(\mathcal{L}^{\pi\circ \nu}\right)^t Q\right) (s,a) \\
        &>0.
    \end{align}
    
\end{proof}

\begin{theorem}
    If the optimal adversarial action-value function under the strongest adversary $Q_0(s,a):=\max_\pi \min_\nu Q^{\pi\circ \nu}(s,a)$ exists for all $s\in\mathcal{S}$ and $a\in\mathcal{A}$, then it is the fixed point of CAR operator.
\end{theorem}
\begin{proof}
    Denote $V_0(s) = \max_\pi \min_\nu V^{\pi\circ \nu}(s)$.
    For all $s\in\mathcal{S}$ and $a\in\mathcal{A}$, we have 
    \begin{align}
        Q_0(s,a) &= \max_\pi \min_\nu Q^{\pi\circ \nu}(s,a) \\
        &= r(s,a)  + \gamma \max_\pi \min_\nu \mathbb{E}_{ s^\prime \sim \mathbb{P}(\cdot|s,a)} V^{\pi\circ \nu}(s^\prime) \\
        &= r(s,a)  + \gamma \mathbb{E}_{ s^\prime \sim \mathbb{P}(\cdot|s,a)} V_0(s^\prime) \\
        &=r(s,a)  + \gamma \mathbb{E}_{ s^\prime \sim \mathbb{P}(\cdot|s,a)} \min_{\nu(s)\in B_\epsilon(s)} \max_\pi \mathbb{E}_{a\sim \pi\left(\cdot|{\nu\left(s\right)}\right)}Q_0(s,a)\\
        &= r(s,a)  + \gamma \mathbb{E}_{ s^\prime \sim \mathbb{P}(\cdot|s,a)} \left[ \min _{s^\prime_\nu \in B_\epsilon(s^\prime)} Q_0 \left(s^\prime,\mathop{\arg\max}_{a_{s^\prime_\nu}} Q_0\left(s^\prime_\nu, a_{s^\prime_\nu}\right)\right) \right]\\
        &= \left( \mathcal{T}_{car} Q \right)(s,a),
    \end{align}
    where the fourth equation comes from lemma \eqref{lem:bellman equation for strongest adversary}. This completes the proof.
\end{proof}

\begin{theorem}\label{app thm: fixed point}
    If the consistency assumption holds, then $Q^*$ is the fixed point of the CAR operator. Further, $Q^*$ is the optimal adversarial action-value function under the strongest adversary, i.e. $Q^*(s, a) = \max_\pi \min_\nu Q^{\pi\circ \nu}(s, a)$, for all $s\in\mathcal{S}$ and $a\in\mathcal{A}$.
\end{theorem}
\begin{proof}
    \begin{align}
        \left(\mathcal{T}_{car} Q^*\right)(s,a) &= r(s,a)  + \gamma \mathbb{E}_{ s^\prime \sim \mathbb{P}(\cdot|s,a)} \left[ \min _{s^\prime_\nu \in B^*_\epsilon(s^\prime)} Q^* \left(s^\prime,\mathop{\arg\max}_{a_{s^\prime_\nu}} Q^*\left(s^\prime_\nu, a_{s^\prime_\nu}\right)\right) \right]\\
        &=  r(s,a)  + \gamma \mathbb{E}_{ s^\prime \sim \mathbb{P}(\cdot|s,a)} \left[ \min _{s^\prime_\nu \in B^*_\epsilon(s^\prime)} \max_{a^\prime} Q^* \left(s^\prime,a^\prime \right) \right]\\
        &= r(s,a)  + \gamma \mathbb{E}_{ s^\prime \sim \mathbb{P}(\cdot|s,a)} \left[ \max_{a^\prime} Q^* \left(s^\prime,a^\prime \right) \right]\\
        &= Q^*(s,a),
    \end{align}
    where the second equality utilizes the definition of $B^*_\epsilon(s^\prime)$. Thus, $Q^*$ is a fixed point of the CAR operator.
    
    Define $\pi$ and $\nu$ as the following:
    \begin{align}
        \pi(s) &:= \mathop{\arg\max}_a Q^*(s,a), \label{eq:pi}\\ 
        \nu(s) &:= \mathop{\arg\min}_{s_\nu \in B_\epsilon(s)} Q^* \left(s,\mathop{\arg\max}_{a_{s_\nu}} Q^*\left(s_\nu, a_{s_\nu}\right)\right). \label{eq:nu}
    \end{align}
    Then, we have 
    \begin{align}
        \left( \mathcal{T}_{car} Q^* \right) (s,a) &= r(s,a)  + \gamma \mathbb{E}_{ s^\prime \sim \mathbb{P}(\cdot|s,a)} \left[ \min _{s^\prime_\nu \in B_\epsilon(s^\prime)} Q^* \left(s^\prime,\mathop{\arg\max}_{a_{s^\prime_\nu}} Q^*\left(s^\prime_\nu, a_{s^\prime_\nu}\right)\right) \right] \\
        &= r(s,a)  + \gamma \mathbb{E}_{ s^\prime \sim \mathbb{P}(\cdot|s,a)} \left[  Q^* \left(s^\prime,\mathop{\arg\max}_{a_{\nu(s^\prime)}} Q^*\left(\nu(s^\prime), a_{\nu(s^\prime)}\right)\right) \right] \\
        &= r(s,a)  + \gamma \mathbb{E}_{ s^\prime \sim \mathbb{P}(\cdot|s,a)} \left[  Q^* \left(s^\prime, \pi\left(\nu(s^\prime)\right)\right) \right] \\
        &=  r(s,a) + \gamma\left(\mathcal{L}^{\pi\circ \nu}Q^*\right)(s,a).
    \end{align}
    Thus, we have 
    \begin{equation}
        Q^* = \left(I-\gamma \mathcal{L}^{\pi\circ \nu}\right)^{-1} r = Q^{\pi\circ \nu},
    \end{equation}
    where equations comes from lemma \ref{lem:invertible lemma}. Further, according to the consistency assumption, we attain $Q^{\pi\circ \nu}(s,a)= Q^{\pi\circ \nu^*(\pi)}$.
    This shows that $Q^*$ is the action-value adversarial function of policy $\pi$ under the strongest adversary $\nu=\nu^*(\pi)$.

    According to the consistency assumption and the definition of $B_\epsilon^*$, we have that 
    \begin{equation}\label{eq: consistency}
        \pi(\nu(s)) = \pi(s), \quad \forall s\in\mathcal{S}.
    \end{equation}
    Then, for any stationary policy $\pi^\prime$, we have that
    \begin{align}
        &\quad \left[\left(\mathcal{L}^{\pi\circ \nu} - \mathcal{L}^{\pi^\prime\circ \nu^*(\pi^\prime)} \right) Q^{\pi\circ \nu}\right] (s,a)\\
        &= \mathbb{E}_{ s^\prime \sim \mathbb{P}(\cdot|s,a)} \left[  Q^{\pi\circ \nu} \left(s^\prime, \pi\left(\nu(s^\prime)\right)\right) - \mathbb{E}_{  a^\prime\sim \pi^\prime\left(\cdot|{\nu^*\left(s^\prime; \pi^\prime\right)}\right)} Q^{\pi\circ \nu}(s^\prime,a^\prime) \right] \\
        &= \mathbb{E}_{ s^\prime \sim \mathbb{P}(\cdot|s,a)} \left[  Q^{\pi\circ \nu} \left(s^\prime, \pi\left(s^\prime\right)\right) - \mathbb{E}_{  a^\prime\sim \pi^\prime\left(\cdot|{\nu^*\left(s^\prime; \pi^\prime\right)}\right)} Q^{\pi\circ \nu}(s^\prime,a^\prime) \right] \\
        &= \mathbb{E}_{ s^\prime \sim \mathbb{P}(\cdot|s,a),   a^\prime\sim \pi^\prime\left(\cdot|{\nu^*\left(s^\prime; \pi^\prime\right)}\right)} \left[  Q^{\pi\circ \nu} \left(s^\prime, \pi\left(s^\prime\right)\right) - Q^{\pi\circ \nu}(s^\prime,a^\prime) \right] \\
        &\ge 0, \label{eq: nonegative 2}
    \end{align}
    where the second equality comes from \eqref{eq: consistency} and the last inequality comes from \eqref{eq:pi}.

    Further, we have that
    \begin{align}
        Q^* - Q^{\pi^\prime\circ \nu^*(\pi^\prime)} &= Q^{\pi\circ \nu} - Q^{\pi^\prime\circ \nu^*(\pi^\prime)} \\
        &= Q^{\pi\circ \nu} - \left(I-\gamma \mathcal{L}^{\pi^\prime\circ \nu^*(\pi^\prime)}\right)^{-1} r \\
        &= Q^{\pi\circ \nu} - \left(I-\gamma \mathcal{L}^{\pi^\prime\circ \nu^*(\pi^\prime)}\right)^{-1} \left(I-\gamma \mathcal{L}^{\pi\circ \nu}\right) Q^{\pi\circ \nu} \\
        &= \left(I-\gamma \mathcal{L}^{\pi^\prime\circ \nu^*(\pi^\prime)}\right)^{-1} \left( \left(I-\gamma \mathcal{L}^{\pi^\prime\circ \nu^*(\pi^\prime)}\right) - \left(I-\gamma \mathcal{L}^{\pi\circ \nu}\right)  \right) Q^{\pi\circ \nu} \\
        &= \gamma \left(I-\gamma \mathcal{L}^{\pi^\prime\circ \nu^*(\pi^\prime)}\right)^{-1} \left( \mathcal{L}^{\pi\circ \nu} - \mathcal{L}^{\pi^\prime\circ \nu^*(\pi^\prime)} \right) Q^{\pi\circ \nu} \\
        &\ge 0,
    \end{align}
    where the last inequality comes from \eqref{eq: nonegative 2} and lemma \ref{lem:nonegative}. Thus, we have that $Q^{\pi\circ \nu} = Q^* \ge Q^{\pi^\prime\circ \nu^*(\pi^\prime)}$ for all policy $\pi^\prime$ which shows that $\pi$ is the optimal robust policy under strongest adversary.
\end{proof}

\begin{corollary}
    If the consistency assumption holds, there exists a deterministic and stationary policy $\pi^*$ which satisfies $V^{\pi^*\circ \nu^*(\pi^*)}(s) \ge V^{\pi\circ \nu^*(\pi)}(s)$ and $Q^{\pi^*\circ \nu^*(\pi^*)}(s, a) \ge Q^{\pi\circ \nu^*(\pi)}(s, a)$ for all $\pi\in\Pi$, $s\in\mathcal{S}$ and $a\in\mathcal{A}$.
\end{corollary}
\begin{proof}
    According to theorem \ref{app thm: fixed point}, we have that $Q^*(s,a) = \max_\pi \min_\nu Q^{\pi\circ \nu}(s,a)$, for all $s\in\mathcal{S}$ and $a\in\mathcal{A}$.
    Define $\pi^*$ and $\nu^*$ as the following:
    \begin{align}
        \pi^*(s) &:= \mathop{\arg\max}_a Q^*(s,a),\label{eq:pi*}\\ 
        \nu^*(s) &:= \mathop{\arg\min}_{s_\nu \in B_\epsilon(s)} Q^* \left(s,\mathop{\arg\max}_{a_{s_\nu}} Q^*\left(s_\nu, a_{s_\nu}\right)\right). \label{eq:nu*}
    \end{align}
    Then, we have that 
    \begin{align}
        \left( \mathcal{T}_{car} Q^* \right) (s,a) &= r(s,a)  + \gamma \mathbb{E}_{ s^\prime \sim \mathbb{P}(\cdot|s,a)} \left[ \min _{s^\prime_\nu \in B_\epsilon(s^\prime)} Q^* \left(s^\prime,\mathop{\arg\max}_{a_{s^\prime_\nu}} Q^*\left(s^\prime_\nu, a_{s^\prime_\nu}\right)\right) \right] \\
        &= r(s,a)  + \gamma \mathbb{E}_{ s^\prime \sim \mathbb{P}(\cdot|s,a)} \left[  Q^* \left(s^\prime,\mathop{\arg\max}_{a_{\nu(s^\prime)}} Q^*\left(\nu^*(s^\prime), a_{\nu(s^\prime)}\right)\right) \right] \\
        &= r(s,a)  + \gamma \mathbb{E}_{ s^\prime \sim \mathbb{P}(\cdot|s,a)} \left[  Q^* \left(s^\prime, \pi^*\left(\nu^*(s^\prime)\right)\right) \right] \\
        &=  r(s,a) + \gamma\left(\mathcal{L}^{\pi^*\circ \nu^*}Q^*\right)(s,a).
    \end{align}
    Thus, we have 
    \begin{equation}
        Q^* = \left(I-\gamma \mathcal{L}^{\pi\circ \nu}\right)^{-1} r = Q^{\pi^*\circ \nu^*},
    \end{equation}
    where equations comes from lemma \ref{lem:invertible lemma}. Further, according to the consistency assumption, we attain $Q^{\pi^*\circ \nu^*}(s,a)= Q^{\pi^*\circ \nu^*(\pi^*)}$.
    This shows that $Q^*$ is the action-value adversarial function of policy $\pi^*$ under the strongest adversary $\nu*=\nu^*(\pi^*)$. Thus, we have that 
    \begin{equation}\label{eq:pi* nu* Q}
        Q^{\pi^*\circ \nu^*(\pi^*)}(s,a) \ge Q^{\pi\circ \nu^*(\pi)}(s,a),\quad \forall s\in\mathcal{S}, a\in\mathcal{A}.       
    \end{equation}
    
    For any policy $\pi$ and $s\in\mathcal{S}$, we have that 
    \begin{align}
        V^{\pi^*\circ \nu^*(\pi^*)}(s)&= \mathbb{E}_{a\sim \pi^*\left(\cdot|{\nu^*\left(s;\pi^*\right)}\right)}  Q^{\pi^*\circ \nu^*(\pi^*)}(s,a) \\
        &= \max_a Q^{\pi^*\circ \nu^*(\pi^*)}(s,a) \\
        &\ge \mathbb{E}_{a\sim \pi\left(\cdot|{\nu^*\left(s;\pi\right)}\right)}  Q^{\pi^*\circ \nu^*}(s,a) \\
        &\ge \mathbb{E}_{a\sim \pi\left(\cdot|{\nu^*\left(s;\pi\right)}\right)}  Q^{\pi\circ \nu^*(\pi)}(s,a) \\
        &= V^{\pi\circ \nu^*(\pi)}(s),
    \end{align}
    where the first and last equations come from lemma \ref{lem:bellman equations in samdp} and the last inequality comes from \eqref{eq:pi* nu* Q}.
\end{proof}

\subsubsection{Not a Contraction}\label{app: not a contraction}

\begin{theorem}
    $\mathcal{T}_{car}$ is \textit{not} a contraction.
\end{theorem}
\begin{proof}
     Let $\mathcal{S}=[-1,1]$, $\mathcal{A}=\{a_1,a_2\}$, $0<\epsilon \ll 1$ and dynamic transition $\mathbb{P}(\cdot|s,a)$ be a determined function. Let $n > \max\{ \frac{\delta}{\gamma},2\delta\}$, $\delta > 0$ and
     \begin{equation}
      \begin{aligned}
         Q_1(s,a_1)&=2n \cdot \mathbbm{1}_{\left\{ s\in \left[ -1,0\right) \right\}}
         +\left[ 2n-\frac{2n-2\delta}{\frac{1}{8}\epsilon}s\right] \cdot \mathbbm{1}_{\left\{ s\in \left[ 0,\frac{1}{8}\epsilon \right) \right\}}
         +2\delta \cdot \mathbbm{1}_{\left\{ s\in \left[ \frac{1}{8}\epsilon,\frac{3}{8}\epsilon \right) \right\} }\\
         &+\left[ 2\delta+ \frac{n-2\delta}{\frac{1}{8}\epsilon}\left(  s-\frac{3\epsilon}{8} \right) \right] \cdot \mathbbm{1}_{\left\{ s\in \left[ \frac{3}{8}\epsilon,\frac{1}{2}\epsilon \right) \right\}}
         +n \cdot \mathbbm{1}_{\left\{ s\in \left[ \frac{1}{2}\epsilon,1\right] \right\}},
     \end{aligned}
    \end{equation}
    
    \begin{equation}
      \begin{aligned}
         Q_1(s,a_2)&=n \cdot \mathbbm{1}_{\left\{ s\in \left[ -1,0\right) \right\}}
         +\left[ n-\frac{n-\delta}{\frac{1}{8}\epsilon}s\right] \cdot \mathbbm{1}_{\left\{ s\in \left[ 0,\frac{1}{8}\epsilon \right) \right\}}
         +\delta \cdot \mathbbm{1}_{\left\{ s\in \left[ \frac{1}{8}\epsilon,\frac{3}{8}\epsilon \right) \right\} }\\
         &+\left[ \delta+ \frac{2n-\delta}{\frac{1}{8}\epsilon}\left(  s-\frac{3\epsilon}{8} \right) \right] \cdot \mathbbm{1}_{\left\{ s\in \left[ \frac{3}{8}\epsilon,\frac{1}{2}\epsilon \right) \right\}}
         +2n \cdot \mathbbm{1}_{\left\{ s\in \left[ \frac{1}{2}\epsilon,1\right] \right\}},
     \end{aligned}
    \end{equation}

    \begin{equation}
      \begin{aligned}
         Q_2(s,a_1)&=2n \cdot \mathbbm{1}_{\left\{ s\in \left[ -1,0\right) \right\}}
         +\left[ 2n-\frac{2n-\delta}{\frac{1}{8}\epsilon}s\right] \cdot \mathbbm{1}_{\left\{ s\in \left[ 0,\frac{1}{8}\epsilon \right) \right\}}
         +\delta \cdot \mathbbm{1}_{\left\{ s\in \left[ \frac{1}{8}\epsilon,\frac{3}{8}\epsilon \right) \right\} }\\
         &+\left[ \delta+ \frac{n-\delta}{\frac{1}{8}\epsilon}\left(  s-\frac{3\epsilon}{8} \right) \right] \cdot \mathbbm{1}_{\left\{ s\in \left[ \frac{3}{8}\epsilon,\frac{1}{2}\epsilon \right) \right\}}
         +n \cdot \mathbbm{1}_{\left\{ s\in \left[ \frac{1}{2}\epsilon,1\right] \right\}},
     \end{aligned}
    \end{equation}

    \begin{equation}
      \begin{aligned}
         Q_2(s,a_2)&=n \cdot \mathbbm{1}_{\left\{ s\in \left[ -1,0\right) \right\}}
         +\left[ n-\frac{n-2\delta}{\frac{1}{8}\epsilon}s\right] \cdot \mathbbm{1}_{\left\{ s\in \left[ 0,\frac{1}{8}\epsilon \right) \right\}}
         +2\delta \cdot \mathbbm{1}_{\left\{ s\in \left[ \frac{1}{8}\epsilon,\frac{3}{8}\epsilon \right) \right\} }\\
         &+\left[ 2\delta+ \frac{2n-2\delta}{\frac{1}{8}\epsilon}\left(  s-\frac{3\epsilon}{8} \right) \right] \cdot \mathbbm{1}_{\left\{ s\in \left[ \frac{3}{8}\epsilon,\frac{1}{2}\epsilon \right) \right\}}
         +2n \cdot \mathbbm{1}_{\left\{ s\in \left[ \frac{1}{2}\epsilon,1\right] \right\}}.
     \end{aligned}
    \end{equation}

     Then 
     \begin{equation}
         \|Q_1-Q_2\|_{L^\infty\left( \mathcal{S}\times\mathcal{A} \right)}=\delta.
     \end{equation}
     We have
     \begin{equation}
     \begin{aligned}
         \mathcal{T}_{car} Q_1(s,a) - \mathcal{T}_{car} Q_2(s,a)&= \gamma \mathbb{E}_{ s^\prime \sim \mathbb{P}(\cdot|s,a)} \left[ \min _{s^1_\nu \in B_\epsilon(s^\prime)} Q_1 \left(s^\prime,\mathop{\arg\max}_{a_{s^1_\nu}} Q_1\left(s^1_\nu, a_{s^1_\nu}\right)\right) - \right.\\
         &\left. \min _{s^2_\nu \in B_\epsilon(s^\prime)} Q_2 \left(s^\prime,\mathop{\arg\max}_{a_{s^2_\nu}} Q_2\left(s^2_\nu, a_{s^2_\nu}\right)\right)  \right].
     \end{aligned}
     \end{equation}
     Let  $\mathbb{P}(s^\prime=-\frac{\epsilon}{2}|s,a)=1$ and $s^\prime=-\frac{\epsilon}{2}$, then 
     \begin{equation}
         \begin{aligned}
             \min _{s^1_\nu \in B_\epsilon(s^\prime)} Q_1 \left(s^\prime,\mathop{\arg\max}_{a_{s^1_\nu}} Q_1\left(s^1_\nu, a_{s^1_\nu}\right)\right)=Q_1(s^{\prime},a_1),
         \end{aligned}
     \end{equation}
     \begin{equation}
         \begin{aligned}
             \min _{s^1_\nu \in B_\epsilon(s^\prime)} Q_2 \left(s^\prime,\mathop{\arg\max}_{a_{s^2_\nu}} Q_2\left(s^2_\nu, a_{s^2_\nu}\right)\right)=Q_2(s^{\prime},a_2).
         \end{aligned}
     \end{equation}
     Thus 
     \begin{equation}
     \begin{aligned}
        \mathcal{T}_{car} Q_1(s,a) - \mathcal{T}_{car} Q_2(s,a)=\gamma \left [ Q_1(s^{\prime},a_1)-Q_2(s^{\prime},a_2)\right ]
        =\gamma n
        >\delta,
     \end{aligned}
     \end{equation}
     which means that
     \begin{align}
         \|\mathcal{T}_{car} Q_1 - \mathcal{T}_{car} Q_2\|_{L^\infty\left( \mathcal{S}\times\mathcal{A} \right)} >\|Q_1-Q_2\|_{L^\infty\left( \mathcal{S}\times\mathcal{A} \right)}.
     \end{align}
     Therefore, $\mathcal{T}_{car}$ is not a contraction. 
\end{proof}

\subsubsection{Convergence}\label{app:convergence}
In this section, we prove a conclusion for convergence of the fixed point iterations of the CAR operator under the $\left(L_r, L_{\mathbb{P}}\right)$-smooth environment assumption.
\begin{definition}[\citet{bukharin2023robust}]
    Let $\mathcal{S} \subseteq \mathbb{R}^d$. We say the environment is $\left(L_r, L_{\mathbb{P}}\right)$-smooth, if the reward function $r: \mathcal{S} \times \mathcal{A} \rightarrow \mathbb{R}$, and the transition dynamics $\mathbb{P}: \mathcal{S} \times \mathcal{A} \rightarrow \Delta\left(\mathcal{S}\right)$ satisfy
    $$
    \left|r(s, a)-r\left(s^{\prime}, a\right)\right| \leq L_r\left\|s-s^{\prime}\right\| \text { and }\left\|\mathbb{P}(\cdot \mid s, a)-\mathbb{P}\left(\cdot \mid s^{\prime}, a\right)\right\|_{L^1\left(\mathcal{S}\right)} \leq L_{\mathbb{P}}\left\|s-s^{\prime}\right\|,
    $$
    for $\left(s, s^{\prime}, a\right) \in \mathcal{S} \times \mathcal{S} \times \mathcal{A}$. $\|\cdot\|$ denotes a metric on $\mathbb{R}^d$. 
    
\end{definition}
The definition is motivated by observations that certain natural environments show smooth reward function and transition dynamics, especially in continuous control tasks where the transition dynamics come from some physical laws.

The following lemma shows that $\mathcal{T}_{car}^{k} Q$ is uniformly bounded.
\begin{lemma} \label{lem: uniform bound}
    Suppose $Q$ and $r$ are uniformly bounded, i.e. $\exists\ M_Q,M_r >0$ such that $\left|Q(s,a)\right| \le M_Q,\ \left|r(s,a)\right| \le M_r\ \forall s\in\mathcal{S}, a\in\mathcal{A}$. Then $\mathcal{T}_{car} Q(\cdot,a)$ is uniformly bounded, i.e.
    \begin{equation}
        \left|\mathcal{T}_{car} Q(s,a) \right| \le C_Q,\ \forall s\in\mathcal{S}, a\in\mathcal{A},
    \end{equation}
    where $C_Q = \max\left\{ M_Q, \frac{M_r}{1-\gamma} \right\}$. Further, for any $k\in\mathbb{N}$, $\mathcal{T}_{car}^{k} Q(\cdot,a)$ has the same uniform bound as $\mathcal{T}_{car} Q(\cdot,a)$, i.e.
    \begin{equation}\label{eq: uniform bound}
        \left|\mathcal{T}_{car}^{k} Q(s,a) \right| \le C_Q,\ \forall s\in\mathcal{S}, a\in\mathcal{A}.
    \end{equation}
\end{lemma}
\begin{proof}
    \begin{align}
        \left|\mathcal{T}_{car} Q(s,a) \right| &= \left|r(s,a)  + \gamma \mathbb{E}_{ s^\prime \sim \mathbb{P}(\cdot|s,a)} \left[ \min _{s^\prime_\nu \in B_\epsilon(s^\prime)} Q \left(s^\prime,\mathop{\arg\max}_{a_{s^\prime_\nu}} Q\left(s^\prime_\nu, a_{s^\prime_\nu}\right)\right) \right]\right| \\
        &\le \left|r(s,a) \right| + \gamma\mathbb{E}_{ s^\prime \sim \mathbb{P}(\cdot|s,a)} \left|\min _{s^\prime_\nu \in B_\epsilon(s^\prime)} Q \left(s^\prime,\mathop{\arg\max}_{a_{s^\prime_\nu}} Q\left(s^\prime_\nu, a_{s^\prime_\nu}\right)\right)\right| \\
        &\le M_r + \gamma M_Q \\
        &\le \max\left\{ M_Q, \frac{M_r}{1-\gamma} \right\}, \qquad \forall s\in\mathcal{S}, a\in\mathcal{A}.
    \end{align}

    Let $C_Q = \max\left\{ M_Q, \frac{M_r}{1-\gamma} \right\}$. Suppose the inequality (\ref{eq: uniform bound}) holds for $k=n$. Then, for $k=n+1$, we have
    \begin{align}
        \left|\mathcal{T}_{car}^{n+1} Q(s,a) \right| &= \left|r(s,a)  + \gamma \mathbb{E}_{ s^\prime \sim \mathbb{P}(\cdot|s,a)} \left[ \min _{s^\prime_\nu \in B_\epsilon(s^\prime)} \mathcal{T}_{car}^{n} Q \left(s^\prime,\mathop{\arg\max}_{a_{s^\prime_\nu}} \mathcal{T}_{car}^{n} Q\left(s^\prime_\nu, a_{s^\prime_\nu}\right)\right) \right]\right| \\
        &\le \left|r(s,a) \right| + \gamma\mathbb{E}_{ s^\prime \sim \mathbb{P}(\cdot|s,a)}  \left|\min _{s^\prime_\nu \in B_\epsilon(s^\prime)} \mathcal{T}_{car}^{n} Q \left(s^\prime,\mathop{\arg\max}_{a_{s^\prime_\nu}} \mathcal{T}_{car}^{n} Q\left(s^\prime_\nu, a_{s^\prime_\nu}\right)\right)\right| \\
        &\le M_r + \gamma C_Q \\
        &\le (1-\gamma) C_Q + \gamma C_Q \\
        &= C_Q.
    \end{align}
    By induction, we have $\left|\mathcal{T}_{car}^{k} Q(s,a) \right| \le C_Q,\ \forall s\in\mathcal{S}, a\in\mathcal{A}, k\in\mathbb{N}$.
\end{proof}

The following lemma shows that $\mathcal{T}_{car}^{k} Q$ is uniformly Lipschitz continuous in the $\left(L_r, L_{\mathbb{P}}\right)$-smooth environment.
\begin{lemma} \label{lem: lip}
    Suppose the environment is $\left(L_r, L_{\mathbb{P}}\right)$-smooth and suppose $Q$ and $r$ are uniformly bounded, i.e. $\exists\ M_Q,M_r >0$ such that $\left|Q(s,a)\right| \le M_Q,\ \left|r(s,a)\right| \le M_r\ \forall s\in\mathcal{S}, a\in\mathcal{A}$. Then $\mathcal{T}_{car} Q(\cdot,a)$ is Lipschitz continuous, i.e.
        \begin{equation}
            \left| \mathcal{T}_{car} Q(s,a) - \mathcal{T}_{car} Q(s^\prime,a) \right| \le L_{\mathcal{T}_{car}} \|s - s^\prime\|,
        \end{equation}
        where $L_{\mathcal{T}_{car}} =  L_r + \gamma C_Q L_{\mathbb{P}}$ and $C_Q = \max\left\{ M_Q, \frac{M_r}{1-\gamma} \right\}$. Further, for any $k\in\mathbb{N}$, $\mathcal{T}_{car}^{k} Q(\cdot,a)$ is Lipschitz continuous and has the same Lipschitz constant as $\mathcal{T}_{car} Q(\cdot,a)$, i.e.
        \begin{equation}
            \left| \mathcal{T}_{car}^{k} Q(s,a) - \mathcal{T}_{car}^{k} Q(s^\prime,a) \right| \le L_{\mathcal{T}_{car}} \|s - s^\prime\|.
        \end{equation}
\end{lemma}
\begin{proof}
    For all $s_1,s_2 \in \mathcal{S}$, we have
    \begin{align}
        &\quad \mathcal{T}_{car} Q(s_1,a) - \mathcal{T}_{car} Q(s_2,a) \\
        &= r(s_1,a)  + \gamma \mathbb{E}_{ s^\prime \sim \mathbb{P}(\cdot|s_1,a)} \left[ \min _{s^\prime_\nu \in B_\epsilon(s^\prime)} Q \left(s^\prime,\mathop{\arg\max}_{a_{s^\prime_\nu}} Q\left(s^\prime_\nu, a_{s^\prime_\nu}\right)\right) \right] \\
        &\quad - r(s_2,a) - \gamma \mathbb{E}_{ s^\prime \sim \mathbb{P}(\cdot|s_2,a)} \left[ \min _{s^\prime_\nu \in B_\epsilon(s^\prime)} Q \left(s^\prime,\mathop{\arg\max}_{a_{s^\prime_\nu}} Q\left(s^\prime_\nu, a_{s^\prime_\nu}\right)\right) \right] \\
        &= \left(r(s_1,a) - r(s_2,a)\right) \\
        &\quad + \gamma \int_{s^\prime} \left( \mathbb{P}(s^\prime|s_1,a) - \mathbb{P}(s^\prime|s_2,a)\right) \min _{s^\prime_\nu \in B_\epsilon(s^\prime)} Q \left(s^\prime,\mathop{\arg\max}_{a_{s^\prime_\nu}} Q\left(s^\prime_\nu, a_{s^\prime_\nu}\right)\right) ds^\prime.
    \end{align}

    Then, we have 
    \begin{align}
        &\quad \left| \mathcal{T}_{car} Q(s_1,a) - \mathcal{T}_{car} Q(s_2,a)\right| \\
        &\le \left| \left(r(s_1,a) - r(s_2,a)\right) \right| \\
        &\quad + \left| \gamma \int_{s^\prime} \left( \mathbb{P}(s^\prime|s_1,a) - \mathbb{P}(s^\prime|s_2,a)\right) \min _{s^\prime_\nu \in B_\epsilon(s^\prime)} Q \left(s^\prime,\mathop{\arg\max}_{a_{s^\prime_\nu}} Q\left(s^\prime_\nu, a_{s^\prime_\nu}\right)\right) ds^\prime \right| \\    
        &\le L_r \|s_1 - s_2\| \\
        &\quad + \gamma \int_{s^\prime} \left| \mathbb{P}(s^\prime|s_1,a) - \mathbb{P}(s^\prime|s_2,a)\right| \left| \min _{s^\prime_\nu \in B_\epsilon(s^\prime)} Q \left(s^\prime,\mathop{\arg\max}_{a_{s^\prime_\nu}} Q\left(s^\prime_\nu, a_{s^\prime_\nu}\right)\right) \right| ds^\prime\\
        &\le  L_r \|s_1 - s_2\| + \gamma C_Q \int_{s^\prime} \left| \mathbb{P}(s^\prime|s_1,a) - \mathbb{P}(s^\prime|s_2,a)\right| ds^\prime\\
        &\le  L_r \|s_1 - s_2\| + \gamma C_Q L_{\mathbb{P}} \|s_1 - s_2\| \\
        &= \left( L_r + \gamma C_Q L_{\mathbb{P}} \right) \|s_1 - s_2\|.
    \end{align}
    The second inequality comes from the Lipschitz property of $r$. The third inequality comes from the uniform boundedness of $Q$ and the last inequality utilizes the Lipschitz property of $\mathbb{P}$.

    Note that $\mathcal{T}_{car}^{k}$ and $\mathcal{T}_{car}$ have the same uniform boundedness $C_Q$. Then, due to lemma \ref{lem: uniform bound}, we can extend the above proof to $\mathcal{T}_{car}^{k}$.
\end{proof}

\begin{remark}
    Note that if replace the operator $\mathcal{T}_{car}$ in the Lemma \ref{lem: uniform bound} and Lemma \ref{lem: lip} with Bellman optimality operator $\mathcal{T}_B$, these lemmas still hold.
\end{remark}

The following lemma shows that the fixed point iteration has a property close to contraction.
\begin{lemma} \label{lem: near contraction}
    Suppose $Q$ and $r$ are uniformly bounded, i.e. $\exists\ M_Q,M_r >0$ such that $\left|Q(s,a)\right| \le M_Q,\ \left|r(s,a)\right| \le M_r\ \forall s\in\mathcal{S}, a\in\mathcal{A}$. Let $Q^*$ denote the Bellman optimality Q-function. If the consistency assumption holds, we have
    \begin{equation}
        \|\mathcal{T}_{car} Q - \mathcal{T}_{car} Q^*\|_\infty \le \gamma \left( \| Q - Q^* \|_\infty + 2 \max_s \max_{s_\nu \in B^*_\epsilon(s)} \max_{a} \left| Q \left(s,a\right) - Q \left(s_\nu,a\right) \right| \right).
    \end{equation}
    Further, if $Q(\cdot, a)$ is $L$-Lipschitz continuous with respect to $s\in\mathcal{S}$, i.e
    \begin{equation}
        \left| Q(s,a) - Q(s^\prime,a) \right| \le L \|s-s^\prime\|,\quad \forall s,s^\prime \in \mathcal{S},\ a\in \mathcal{A},
    \end{equation}
    we have
    \begin{equation}
        \|\mathcal{T}_{car} Q - \mathcal{T}_{car} Q^*\|_\infty \le \gamma \| Q - Q^* \|_\infty + 2 \gamma L \epsilon .
    \end{equation}
\end{lemma}
\begin{proof}
    Denote $a_{s^\prime_\nu,Q}^*=\mathop{\arg\max}_{a} Q\left(s^\prime_\nu, a\right)$ and $ s^{\prime,*}_\nu = \arg\min _{s^\prime_\nu \in B^*_\epsilon(s^\prime)} Q \left(s^\prime,a_{s^\prime_\nu,Q}^*\right)  $. If $\mathcal{T}_{car} Q > \mathcal{T}_{car} Q^*$, we have
    \begin{align}
        &\quad \left(\mathcal{T}_{car} Q\right)(s,a) - \left(\mathcal{T}_{car} Q^*\right)(s,a) \\
        &= \gamma \mathbb{E}_{ s^\prime \sim \mathbb{P}(\cdot|s,a)} \left[ \min _{s^\prime_\nu \in B^*_\epsilon(s^\prime)} Q \left(s^\prime,a_{s^\prime_\nu,Q}^*\right) - \min _{s^\prime_\nu \in B^*_\epsilon(s^\prime)} Q^* \left(s^\prime,a_{s^\prime_\nu,Q^*}^*\right) \right]\\
        &= \gamma \mathbb{E}_{ s^\prime \sim \mathbb{P}(\cdot|s,a)} \left[ Q \left(s^\prime,a_{s^{\prime,*}_\nu,Q}^*\right) - Q^* \left(s^\prime,a_{s^\prime,Q^*}^*\right) \right]\\
        &= \gamma \mathbb{E}_{ s^\prime \sim \mathbb{P}(\cdot|s,a)} \left[ Q \left(s^\prime,a_{s^{\prime,*}_\nu,Q}^*\right) - Q^* \left(s^\prime,a_{s^{\prime,*}_\nu,Q}^*\right) + Q^* \left(s^\prime,a_{s^{\prime,*}_\nu,Q}^*\right) - Q^* \left(s^\prime,a_{s^\prime,Q^*}^*\right) \right]\\
        &\le  \gamma \mathbb{E}_{ s^\prime \sim \mathbb{P}(\cdot|s,a)} \left[ Q \left(s^\prime,a_{s^{\prime,*}_\nu,Q}^*\right) - Q^* \left(s^\prime,a_{s^{\prime,*}_\nu,Q}^*\right) \right]\\
        &\le \gamma \mathbb{E}_{ s^\prime \sim \mathbb{P}(\cdot|s,a)} \left[ \max_{a^\prime} \left( Q \left(s^\prime,a^\prime\right) - Q^* \left(s^\prime,a^\prime\right) \right) \right]\\
        &\le \gamma \| Q - Q^* \|_\infty,
    \end{align}
    where the second equality utilize the definition of $B^*_\epsilon(s^\prime)$ and the first inequality comes from the optimality of $a_{s^\prime,Q^*}^*$.
    If $\mathcal{T}_{car} Q < \mathcal{T}_{car} Q^*$, we have
    \begin{align}
        &\quad \left(\mathcal{T}_{car} Q^*\right)(s,a) - \left(\mathcal{T}_{car} Q\right)(s,a) \\
        &= \gamma \mathbb{E}_{ s^\prime \sim \mathbb{P}(\cdot|s,a)} \left[  \min _{s^\prime_\nu \in B^*_\epsilon(s^\prime)} Q^* \left(s^\prime,a_{s^\prime_\nu,Q^*}^*\right) - \min _{s^\prime_\nu \in B^*_\epsilon(s^\prime)} Q \left(s^\prime,a_{s^\prime_\nu,Q}^*\right)  \right]\\
        &= \gamma \mathbb{E}_{ s^\prime \sim \mathbb{P}(\cdot|s,a)} \left[ Q^* \left(s^\prime,a_{s^\prime,Q^*}^*\right) - Q \left(s^\prime,a_{s^{\prime,*}_\nu,Q}^*\right) \right]\\
        &= \gamma \mathbb{E}_{ s^\prime \sim \mathbb{P}(\cdot|s,a)} \left[ Q^* \left(s^\prime,a_{s^\prime,Q^*}^*\right) - Q\left(s^\prime,a_{s^\prime,Q^*}^*\right)\right] \label{eq: convergence 1}\\
        &\quad + \gamma \mathbb{E}_{ s^\prime \sim \mathbb{P}(\cdot|s,a)}\left[Q\left(s^\prime,a_{s^\prime,Q^*}^*\right) - Q\left( s^{\prime,*}_\nu, a_{s^\prime,Q^*}^* \right) \right] \label{eq: convergence 2}\\
        &\quad + \gamma \mathbb{E}_{ s^\prime \sim \mathbb{P}(\cdot|s,a)}\left[Q\left( s^{\prime,*}_\nu, a_{s^\prime,Q^*}^* \right) - Q \left(s^\prime,a_{s^{\prime,*}_\nu,Q}^*\right) \right]. \label{eq: convergence 3}
    \end{align}
    
    We will separately analyze the items \ref{eq: convergence 1}, \ref{eq: convergence 2} and \ref{eq: convergence 3}. Firstly, we can bound the item \ref{eq: convergence 1} with $\| Q - Q^* \|_\infty$.
    \begin{align}
        &\quad \mathbb{E}_{ s^\prime \sim \mathbb{P}(\cdot|s,a)} \left[ Q^* \left(s^\prime,a_{s^\prime,Q^*}^*\right) - Q\left(s^\prime,a_{s^\prime,Q^*}^*\right)\right] \\
        & \le \mathbb{E}_{ s^\prime \sim \mathbb{P}(\cdot|s,a)} \left[ \max_{a^\prime} \left( Q \left(s^\prime,a^\prime\right) - Q^* \left(s^\prime,a^\prime\right) \right) \right] \\
        &\le \| Q - Q^* \|_\infty.
    \end{align}

    For the item \ref{eq: convergence 2}, we have
    \begin{align}
        &\quad \mathbb{E}_{ s^\prime \sim \mathbb{P}(\cdot|s,a)}\left[Q\left(s^\prime,a_{s^\prime,Q^*}^*\right) - Q\left( s^{\prime,*}_\nu, a_{s^\prime,Q^*}^* \right) \right]  \\
        &\le \mathbb{E}_{ s^\prime \sim \mathbb{P}(\cdot|s,a)} \left[ \max_{a^\prime} \left( Q \left(s^\prime,a^\prime\right) - Q \left(s^{\prime,*}_\nu,a^\prime\right) \right) \right] \\
        &\le \mathbb{E}_{ s^\prime \sim \mathbb{P}(\cdot|s,a)} \left[ \max_{s^\prime_\nu \in B^*_\epsilon(s^\prime)} \max_{a^\prime} \left| Q \left(s^\prime,a^\prime\right) - Q \left(s^\prime_\nu,a^\prime\right) \right| \right] \\
        &\le \max_s \max_{s_\nu \in B^*_\epsilon(s)} \max_{a} \left| Q \left(s,a\right) - Q \left(s_\nu,a\right) \right| .
    \end{align}

    Due to $a_{s^{\prime,*}_\nu,Q}^*=\mathop{\arg\max}_{a} Q\left(s^{\prime,*}_\nu, a\right)$, we have $Q\left(s^{\prime,*}_\nu, a\right) \le Q\left(s^{\prime,*}_\nu, a_{s^{\prime,*}_\nu,Q}^*\right),\ \forall a$. Then, for the item \ref{eq: convergence 3}, we have
    \begin{align}
        &\quad \mathbb{E}_{ s^\prime \sim \mathbb{P}(\cdot|s,a)}\left[Q\left( s^{\prime,*}_\nu, a_{s^\prime,Q^*}^* \right) - Q \left(s^\prime,a_{s^{\prime,*}_\nu,Q}^*\right) \right] \\
        &\le \mathbb{E}_{ s^\prime \sim \mathbb{P}(\cdot|s,a)}\left[Q\left(s^{\prime,*}_\nu, a_{s^{\prime,*}_\nu,Q}^*\right) - Q \left(s^\prime,a_{s^{\prime,*}_\nu,Q}^*\right) \right] \\
        &\le \mathbb{E}_{ s^\prime \sim \mathbb{P}(\cdot|s,a)} \left[ \max_{a^\prime} \left| Q \left(s^\prime,a^\prime\right) - Q \left(s^{\prime,*}_\nu,a^\prime\right) \right| \right]\\
        &\le \mathbb{E}_{ s^\prime \sim \mathbb{P}(\cdot|s,a)} \left[ \max_{s^\prime_\nu \in B^*_\epsilon(s^\prime)} \max_{a^\prime} \left| Q \left(s^\prime,a^\prime\right) - Q \left(s^\prime_\nu,a^\prime\right) \right| \right] \\
        &\le \max_s \max_{s_\nu \in B^*_\epsilon(s)} \max_{a} \left| Q \left(s,a\right) - Q \left(s_\nu,a\right) \right| .
    \end{align}

    Thus, we have 
    \begin{align}
        &\quad \left(\mathcal{T}_{car} Q^*\right)(s,a) - \left(\mathcal{T}_{car} Q\right)(s,a) \\
        &\le \gamma \left( \| Q - Q^* \|_\infty + 2 \max_s \max_{s_\nu \in B^*_\epsilon(s)} \max_{a} \left| Q \left(s,a\right) - Q \left(s_\nu,a\right) \right| \right).
    \end{align}
    In a summary, we get 
    \begin{equation}
        \|\mathcal{T}_{car} Q - \mathcal{T}_{car} Q^*\|_\infty \le \gamma \left( \| Q - Q^* \|_\infty + 2 \max_s \max_{s_\nu \in B^*_\epsilon(s)} \max_{a} \left| Q \left(s,a\right) - Q \left(s_\nu,a\right) \right| \right).
    \end{equation}

    Further, when $Q(\cdot, a)$ is $L$-Lipschitz continuous, i.e
    \begin{equation}
        \left| Q(s,a) - Q(s^\prime,a) \right| \le L \|s-s^\prime\|,\quad \forall s,s^\prime \in \mathcal{S},\ a\in \mathcal{A},
    \end{equation}
    we have 
    \begin{align}
        &\quad \max_s \max_{s_\nu \in B^*_\epsilon(s)} \max_{a} \left| Q \left(s,a\right) - Q \left(s_\nu,a\right) \right| \\
        &\le  \max_s \max_{s_\nu \in B^*_\epsilon(s)} L \|s-s_\nu\| \\
        &\le L \epsilon.
    \end{align}
    Then, we have
    \begin{equation}
        \|\mathcal{T}_{car} Q - \mathcal{T}_{car} Q^*\|_\infty \le \gamma \left( \| Q - Q^* \|_\infty + 2 L \epsilon \right).
    \end{equation}
\end{proof}
\begin{remark}
    We can relax the Lipschitz condition to local lipschitz continuous in the $B^*_\epsilon(s)$.
\end{remark}

We prove that the fixed point iterations of $\mathcal{T}_{car}$ at least converge to a sub-optimal solution close to $Q^*$ in the $\left(L_r, L_{\mathbb{P}}\right)$-smooth environment.
\begin{theorem}
    Suppose the environment is $\left(L_r, L_{\mathbb{P}}\right)$-smooth and suppose $Q_0$ and $r$ are uniformly bounded, i.e. $\exists\ M_{Q_0},M_r >0$ such that $\left|Q_0(s,a)\right| \le M_{Q_0},\ \left|r(s,a)\right| \le M_r\ \forall s\in\mathcal{S}, a\in\mathcal{A}$. Let $Q^*$ denote the Bellman optimality Q-function and $Q_{k+1} = \mathcal{T}_{car} Q_{k} = \mathcal{T}_{car}^{k+1} Q_0$ for all $k\in\mathbb{N}$. If the consistency assumption holds, we have
    \begin{equation}
        \|Q_{k+1} - Q^*\|_\infty \le  \gamma^{k+1} \| Q_0 - Q^*\|_\infty + \gamma^{k+1} D_{Q_0} + \frac{2 \gamma \epsilon }{1-\gamma}L_{\mathcal{T}_{car}},
    \end{equation}
    where $D_{Q_0} = 2 \max_s \max_{s_\nu \in B^*_\epsilon(s)} \max_{a} \left| Q_0 \left(s,a\right) - Q_0 \left(s_\nu,a\right) \right|$, $L_{\mathcal{T}_{car}} =  L_r + \gamma C_{Q_0} L_{\mathbb{P}}$ and $C_{Q_0} = \max\left\{ M_{Q_0}, \frac{M_r}{1-\gamma} \right\}$. 
\end{theorem}
\begin{proof}
    For any $k\in\mathbb{N}$, we have
    \begin{align}
        &\quad \|Q_{k+1} - Q^*\|_\infty \\
        &= \|\mathcal{T}_{car}^{k+1} Q_0 - \mathcal{T}_{car}^{k+1} Q^*\|_\infty \\
        &\le \gamma \|\mathcal{T}_{car}^{k} Q_0 - \mathcal{T}_{car}^{k} Q^*\|_\infty  + 2 \gamma  L_{\mathcal{T}_{car}} \epsilon \\
        &\le \gamma \left( \gamma \|\mathcal{T}_{car}^{k-1} Q_0 - \mathcal{T}_{car}^{k-1} Q^*\|_\infty  + 2 \gamma  L_{\mathcal{T}_{car}} \epsilon \right) + 2 \gamma  L_{\mathcal{T}_{car}} \epsilon \\
        &= \gamma^2 \|\mathcal{T}_{car}^{k-1} Q_0 - \mathcal{T}_{car}^{k-1} Q^*\|_\infty + 2 \epsilon L_{\mathcal{T}_{car}}  \sum_{l=1}^{2} \gamma^l \\
        &\le \cdots \\
        &\le \gamma^k \|\mathcal{T}_{car} Q_0 - \mathcal{T}_{car} Q^*\|_\infty + 2 \epsilon L_{\mathcal{T}_{car}}  \sum_{l=1}^{k} \gamma^l \\
        &\le \gamma^{k+1} \| Q_0 - Q^*\|_\infty + 2 \gamma^{k+1} \max_s \max_{s_\nu \in B^*_\epsilon(s)} \max_{a} \left| Q_0 \left(s,a\right) - Q_0 \left(s_\nu,a\right) \right|  + 2 \epsilon L_{\mathcal{T}_{car}}  \sum_{l=1}^{k} \gamma^l \\
        &\le \gamma^{k+1} \| Q_0 - Q^*\|_\infty + 2 \gamma^{k+1} \max_s \max_{s_\nu \in B^*_\epsilon(s)} \max_{a} \left| Q_0 \left(s,a\right) - Q_0 \left(s_\nu,a\right) \right| + \frac{2 \gamma \epsilon }{1-\gamma}L_{\mathcal{T}_{car}}.
    \end{align}
    The first and second inequalities come from Lemma \ref{lem: lip} and Lemma \ref{lem: near contraction}. The penultimate inequality comes from Lemma \ref{lem: near contraction}.
\end{proof}

\section{Theorems and Proofs of Policy Robustness under Bellman p-error}

\textbf{Banach Space} is a complete normed space $\left(X, \|\cdot\| \right)$, consisting of a vector space $X$ together with a norm $\|\cdot\| :X\rightarrow \mathbb{R}^+$. In this paper, we consider the setting where the continuous state space $\mathcal{S} \subset \mathbb{R}^d$ is a compact set and the action space $\mathcal{A}$ is a finite set. We discuss in the Banach space $\left(L^p\left( \mathcal{S}\times\mathcal{A} \right), \|\cdot\|_{p} \right),\ 1\le p\le\infty$. Define $L^p\left( \mathcal{S}\times\mathcal{A} \right) = \left\{ f| \|f\|_{p} <\infty \right\}$, where $\|f\|_{p} = \left( \int_{\mathcal{S}}\sum_{a\in\mathcal{A}} \left| f(s,a) \right|^p d\mu(s) \right)^{\frac{1}{p}}$ for $\ 1\le p <\infty$, $\mu$ is the measure over $\mathcal{S}$ and $\|f\|_{\infty} = \inf \left\{ M\in\mathbb{R}_{\ge 0}|\left| f(s,a) \right|\le M \text{ for almost every } (s,a)  \right\}$. For simplicity, we refer to this Banach space as $L^p\left( \mathcal{S}\times\mathcal{A} \right)$.

\subsection{$L^\infty$ is Necessary for Adversarial Robustness}
\begin{theorem}
    There exists an MDP instance $\mathcal{M}$ such that the following statements hold. Given a function $Q$ and adversary perturbation size $\epsilon$, let $\mathcal{S}^Q_{sub}$ denote the set of states where the greedy policy according to $Q$ is suboptimal, i.e. $\mathcal{S}^Q_{sub} = \left\{ s | Q^*(s,\mathop{\arg\max}_a Q(s, a)) < \max_a Q^*(s, a) \right\}$ and let $\mathcal{S}^Q_{adv}$ denote the set of states in whose $\epsilon$-neighbourhood there exists the adversarial state, i.e. $\mathcal{S}^Q_{adv} = \left\{ s | \exists s_\nu \in B_\epsilon(s),\ \text{s.t. } Q^*(s,\mathop{\arg\max}_a Q(s_\nu, a)) < \max_a Q^*(s, a) \right\}$, where $Q^*$ is the Bellman optimal $Q$-function.
        \begin{itemize}
            \item For any $1\le p<\infty$ and $\delta>0$, there exists a function $Q\in L^p\left( \mathcal{S}\times\mathcal{A} \right)$ satisfying $\|Q-Q^*\|_{L^p\left( \mathcal{S}\times\mathcal{A} \right)} \leq \delta$ such that $\mu\left(\mathcal{S}^Q_{sub}\right) = O(\delta)$ yet $\mu\left( \mathcal{S}^Q_{adv} \right) =\mu \left(\mathcal{S}\right)$.
            \item There exists a $\bar{\delta}>0$ such that for any $0< \delta \le \bar{\delta}$, for any function $Q\in L^\infty\left( \mathcal{S}\times\mathcal{A} \right)$ satisfying $\|Q-Q^*\|_{L^\infty\left( \mathcal{S}\times\mathcal{A} \right)} \leq  \delta$, we have that $\mu\left(\mathcal{S}^Q_{sub}\right) = O(\delta)$ and $\mu\left( \mathcal{S}^Q_{adv} \right) =2 \epsilon + O\left( \delta \right)$.
        \end{itemize}
\end{theorem}
\begin{proof}
   Given a MDP instance $\mathcal{M}$ such that $\mathcal{S}=[-1,1]$, $\mathcal{A}=\{a_1,a_2\}$ and 
\begin{equation}
    \mathbb{P}(s^\prime |s,a_1)=\left\{ \begin{aligned}
        \mathbbm{1}_{\left\{s^\prime=s-\epsilon_1\right\}}, &\quad s\in [-1+\epsilon_1,1]\\
        \mathbbm{1}_{\left\{s^\prime=-1\right\}}, &\quad s\in[-1,-1+\epsilon_1)
    \end{aligned} \right .
\end{equation}
\begin{equation}
    \mathbb{P}(s^\prime |s,a_2)=\left\{ \begin{aligned}
        \mathbbm{1}_{\left\{s^\prime=s+\epsilon_1\right\}}, &\quad s\in [-1,1-\epsilon_1]\\
        \mathbbm{1}_{\left\{s^\prime=1\right\}},&\quad s\in (1-\epsilon_1,1]
    \end{aligned} \right.
\end{equation}
\begin{equation}
    \begin{aligned}
        r(s,a_1)&=-ks,\\
        r(s,a_2)&=ks.
    \end{aligned}   
\end{equation}
where $\mathbb{P}$ is the transition dynamic, $r$ is the reward function, $k>0$, $0<\epsilon_1 \ll 1$ and $\mathbbm{1}_{\left\{\cdot\right\}}$ is the indicator function. Let $\gamma$ be the discount factor.

First, we prove that equation (\ref{eq: optimal policy in Th7}) is the optimal policy.
\begin{equation}
    \pi^{*}(s)=\arg\max_{a}Q^*(s,a)=\left\{ \begin{aligned}
        \{a_2\},s>0\\
        \{a_1\},s<0\\
        \{a_1,a_2\},s=0
    \end{aligned} \right.
    \label{eq: optimal policy in Th7}
\end{equation}

 Define $s_t^{\pi}$ is state rollouted by policy $\pi$ in  time step $t$.
 
 Let $s_0>0$, then 
\begin{itemize}
    \item If $s_t^{\pi^*}=1$ , while $s_t^{\pi} \in [-1,1]$, then $s_t^{\pi^*} \geq |s_t^{\pi}|$ hold for any policy $\pi$.
    \item If $s_t^{\pi^*}<1$ . First we have 
    \begin{equation}
        0< s_t^{\pi^*}=s_0+t\epsilon_1 <1,
    \end{equation}
     \begin{equation}
        -1<s_0-t\epsilon_1 .
    \end{equation}
    Then for any policy $\pi$, we have the following equation by  definition of transition,
    \begin{equation}
        s_t^{\pi}=s_0+\sum_{i=1}^{t}x_i \epsilon_1, 
    \end{equation}
    where $x_i \in \{-1,1\} , i=1,...,t$. Then 
    \begin{equation}
        s_t^{\pi^*} \geq |s_t^{\pi}|
    \end{equation}
\end{itemize}
Then for any policy $\pi$ ,$s_0>0$ and $t\geq 0$, we have 
\begin{equation}
    s_t^{\pi^*} \geq |s_t^{\pi}|
\end{equation}
and 
\begin{equation}
    \begin{aligned}
        &\quad \pi(a_2|s_t^{\pi})r(s_t^{\pi},a_2)+\pi(a_1|s_t^{\pi})r(s_t^{\pi},a_1)\\
        & \leq \max\{ \pi(a_2|s_t^{\pi})(ks_t^{\pi}) +\pi(a_1|s_t^{\pi})(-ks_t^{\pi}),\pi(a_2|s_t^{\pi})(-ks_t^{\pi}) +\pi(a_1|s_t^{\pi})(ks_t^{\pi})\}\\
        &=\left|ks_t^{\pi}\left(\pi(a_2|s_t^{\pi})-\pi(a_1|s_t^{\pi})\right)\right|\\
        &\leq \left| ks_t^{\pi} \right|\\
        &\leq ks_t^{\pi^{*}}\\
        &= r(s_t^{\pi^{*}},a_2)
        \label{Th7-0}
    \end{aligned}
\end{equation}
Let $s_0$ be the initial state, $\tau=(s_0,...)$ be the trajectory of policy $\pi$. Define $J(\pi,s_0)=\mathbb{E}_{\tau}\sum_{t} \gamma^t r(s_t,a_t) $ is expected reward in initial state $s_0$ about policy $\pi$. Then
\begin{align}
    J(\pi^*,s_0)-J(\pi,s_0)&=\mathbb{E}_{s_t^{\pi^*},a_t\sim \pi^*(\cdot|s_t^{\pi^*}) }\sum_{t} \gamma^t r(s_t^{\pi^*},a_t)- \mathbb{E}_{s_t^{\pi},a_t\sim \pi(\cdot|s_t^{\pi}) }\sum_{t} \gamma^t r(s_t^{\pi},a_t)\\
    \label{Th7-1}
    &=\sum_{t} \gamma^t r(s_t^{\pi^*},a_2)-\sum_{t}  \gamma^t \mathbb{E}_{s_t^{\pi},a_t\sim \pi(\cdot|s_t^{\pi}) } r(s_t^{\pi},a_t)\\
    &=\sum_{t} \gamma^t r(s_t^{\pi^*},a_2)-\sum_{t}  \gamma^t \mathbb{E}_{s_t^{\pi}} \left[\pi(a_2|s_t^{\pi})r(s_t^{\pi},a_2)+\pi(a_1|s_t^{\pi})r(s_t^{\pi},a_1)\right]\\
    &=\sum_{t} \gamma^t \mathbb{E}_{s_t^{\pi}} \left[r(s_t^{\pi^*},a_2)-\left[\pi(a_2|s_t^{\pi})r(s_t^{\pi},a_2)+\pi(a_1|s_t^{\pi})r(s_t^{\pi},a_1)\right] \right]\\
    \label{Th7-2}
    & \geq 0 .
\end{align}
For (\ref{Th7-1}), the policy $\pi^*$ and dynamic transition $\mathbb{P}$ are deterministic.  For (\ref{Th7-2}), We use property (\ref{Th7-0}).

Then for $s >0$, we get that the optimal policy is $\pi(\cdot|s)=a_2$.
By symmetry, we can also get that the optimal policy is $\pi(\cdot|s)=a_1$ for $s<0$ and  $a_1$,$a_2$ are also optimal action for $ s=0$. Thus we have proved equation (\ref{eq: optimal policy in Th7}) is the optimal policy.

First,we have the following equation according to (\ref{eq: optimal policy in Th7}) 
\begin{align}
        Q^*(0,a_2)=Q^*(0,a_1) .  
\end{align}

For $s> 0$, we have
\begin{equation}
    \begin{aligned}
         Q^*(s,a_2)&=ks+\gamma k(s+\epsilon_1)+\gamma^2 k(s+2\epsilon_1)+...+\gamma^{t_s}k(s+t_s\epsilon_1)+\sum_{n=1}^{\infty}\gamma^{t_s+n}k\times 1\\
         &=ks+k \left[ \sum_{t=1}^{t_s}\gamma^t\left(s+t\epsilon_1 \right) + \sum_{t=t_s+1}^{\infty}\gamma^{t}\right].
         \label{eq: Th8 Q(s,a_2)}
    \end{aligned}
\end{equation}

where $s+t_s\epsilon_1 \in (1-\epsilon_1,1]$, i.e. $t_s = \lfloor  \frac{1-s}{\epsilon_1} \rfloor$.

For $s\geq \epsilon_1$, we have 
\begin{equation}
    \begin{aligned}
    Q^*(s,a_1)&=
       -ks+\gamma k(s-\epsilon_1)+\gamma^2 ks+\dots+\gamma^{t_s+2}k(s+t_s\epsilon_1)+\sum_{n=1}^{\infty}\gamma^{t_s+2+n}k\times 1\\
       &=-ks+k \left[ \sum_{t=1}^{t_s+2}\gamma^t\left(s+(t-2)\epsilon_1 \right) + \sum_{t=t_s+3}^{\infty}\gamma^{t}\right].
       \label{eq: Th8 Q(s,a_1)_1}
\end{aligned}
\end{equation}

For $0<  s< \epsilon_1$, we have
\begin{equation}
    \begin{aligned}
    Q^*(s,a_1)&=-ks+\gamma(-k)(s-\epsilon_1)+\gamma^2(-k)(s-2\epsilon_1)+\dots+\gamma^{q_s}(-k)(s-q_s\epsilon_1)+\sum_{n=1}^{\infty}\gamma^{q_s+n}(-k)(-1)\\
    &=-ks+k \left[ \sum_{t=1}^{q_s}\gamma^t\left(t\epsilon_1-s \right) + \sum_{t=q_s+1}^{\infty}\gamma^{t}\right].
    \label{eq: Th8 Q(s,a_1)_2}
\end{aligned}
\end{equation}

where  $s-q_s\epsilon_1 \in [-1,-1+\epsilon_1)$ ,i.e. $q_s = \lfloor  \frac{1+s}{\epsilon_1} \rfloor >t_s$. 

According to (\ref{eq: Th8 Q(s,a_2)}), (\ref{eq: Th8 Q(s,a_1)_1}), (\ref{eq: Th8 Q(s,a_1)_2}) and $q_s>t_s$, we have

\begin{equation}
    Q^*(s,a_2) -Q^*(s,a_1) > 2ks ,s> 0 \label{Qgap1}
\end{equation}

By symmetry, we can also get 
\begin{equation}
    Q^*(s,a_1) -Q^*(s,a_2) > -2ks ,s< 0\label{Qgap2}
\end{equation}

(1)First, we have 
\begin{equation}
    0 < Q^{*}(s,a) <\sum_{t=0}^{\infty} \gamma^t =\frac{1}{1-\gamma}
    \label{Th8 Q_bound}
\end{equation}

For any $1\leq p < \infty$ , let $n> \max \left\{\frac{1}{\epsilon},\left(\frac{1}{1-\gamma}\right)^p,\delta^p,\delta^{p-1}\right\}$ , $n\in\mathbb{N}$  and
\begin{align}
    Q(s,a_2)&=\left\{ \begin{array}{l}
        Q^*(s,a_2)-n^{\frac{1}{p}} ,s\in [\frac{k}{n},\frac{k}{n}+\frac{\delta^p}{n^2}] , k=0,1,\dots ,n-1\\ 
        Q^*(s,a_2),  others\\
    \end{array}\right.\\
    Q(s,a_1)&=\left\{ \begin{array}{l}
        Q^*(s,a_1)-n^{\frac{1}{p}} ,s\in [-\frac{k+1}{n} ,-\frac{k+1}{n}+\frac{\delta^p}{n^2} ],k=0,1,\dots ,n-1\\ 
        Q^*(s,a_1),  others\\
    \end{array}\right.
\end{align}

Then
\begin{align}
    \|Q(s,a_1)-Q^*(s,a_1)\|_{L^p\left( \mathcal{S} \right)}=\|Q(s,a_2)-Q^*(s,a_2)\|_{L^p\left( \mathcal{S} \right)}=\left[n* \frac{\delta^p}{n^2}*\left(n^{\frac{1}{p}}\right)^p\right]^{\frac{1}{p}}\leq \delta.
\end{align}

And 
\begin{align}
    \|Q(s,a)\|_{L^p\left( \mathcal{S} \right)}&=\|Q(s,a)-Q^*(s,a)+Q^*(s,a)\|_{L^p\left( \mathcal{S} \right)}\\
    &\leq \|Q(s,a)-Q^*(s,a)\|_{L^p\left( \mathcal{S} \right)}+\|Q^*(s,a)\|_{L^p\left( \mathcal{S} \right)}\\
    &< \infty.
\end{align}
which means $Q\in L^p\left( \mathcal{S}\times\mathcal{A} \right)$.

We have the following two inequalities because  $n>\left(\frac{1}{1-\gamma}\right)^p$ and (\ref{Th8 Q_bound}),
\begin{equation}
    Q^*(s,a_2)-n^{\frac{1}{p}} <Q^*(s,a_1),  
\end{equation}
\begin{equation}   
     Q^*(s,a_1)-n^{\frac{1}{p}} <Q^*(s,a_2).
\end{equation}

Then 
\begin{equation}
    \mathcal{S}^Q_{sub}=\bigcup_{k=-n}^{n-1}  \left [\frac{k}{n},\frac{k}{n}+\frac{\delta^p}{n^2}\right]
    \label{eq: ThB1_defination}
\end{equation}
and 
\begin{align}
    \mu\left( \mathcal{S}^Q_{sub} \right) = 2n*\frac{\delta^p}{n^2} <2\delta= O\left(\delta\right)
\end{align}
because $n>\delta^{p-1}$.

 According to  (\ref{eq: ThB1_defination}), the distance between any two adjacent intervals of $\mathcal{S}^Q_{sub}$ is less than $\epsilon$. For any $s\in\mathcal{S}$, $\exists k\in \{-n,-n+1,...,n-1\} $ s.t. $s\in [\frac{k}{n},\frac{k+1}{n}]$. Because $n>\frac{1}{\epsilon}$
(i.e. $\frac{1}{n} < \epsilon$), then $d(s,\frac{k}{n})<\epsilon $ i.e. $d(s,\mathcal{S}^Q_{sub}) <\epsilon$, where $d(\cdot,\cdot)$ is Euclid distance. According to the definition of $\mathcal{S}^Q_{adv}$, we have $\mathcal{S}^Q_{adv}=\mathcal{S}$ and  
\begin{equation}
    \mu\left( \mathcal{S}^Q_{adv} \right)=\mu\left( \mathcal{S}\right).
\end{equation}

(2)Let $\bar{\delta}\in(0,k]$, for any $0< \delta \le \bar{\delta}$, for any state-action value function $Q\in L^\infty\left( \mathcal{S}\times\mathcal{A} \right)$ satisfying $\|Q-Q^*\|_{L^\infty\left( \mathcal{S}\times\mathcal{A} \right)} \leq  \delta$, we can get the following two inequalities by (\ref{Qgap1}) and (\ref{Qgap2}).
\begin{equation}
    Q(s,a_2)\geq Q^*(s,a_2) -\delta >  Q^*(s,a_1) +\delta \geq Q(s,a_1) ,s \in (\frac{\delta}{k},1],
\end{equation}
\begin{equation}
    Q(s,a_1)\geq Q^*(s,a_1) -\delta >  Q^*(s,a_2) +\delta \geq Q(s,a_2) ,s \in [-1,-\frac{\delta}{k}).
\end{equation}
Then
\begin{equation}
    \mu\left( \mathcal{S}^Q_{sub} \right) \leq \frac{2\delta}{k}=O\left( \delta \right),
\end{equation}
\begin{equation}
    \mu\left( \mathcal{S}^Q_{adv} \right) \leq \frac{2\delta}{k}+2\epsilon =2\epsilon+ O\left( \delta \right).
\end{equation}

\end{proof}

\subsection{Stability of Bellman Optimality Equations}

We propose the following concept of stability drawing on relevant research in the field of partial differential equations \cite{wang20222}.
\begin{definition}
    Given two Banach spaces $\mathcal{B}_1$ and $\mathcal{B}_2$, if there exist $\delta>0$ and $C>0$ such that for all $Q\in \mathcal{B}_1 \cap \mathcal{B}_2$ satisfying $\|\mathcal{T}Q - Q\|_{\mathcal{B}_1} < \delta$, we have that $\|Q - Q^*\|_{\mathcal{B}_2} < C \|\mathcal{T}Q - Q\|_{\mathcal{B}_1}$, where $Q^*$ is the exact solution of this functional equation. Then, we say that a nonlinear functional equation $\mathcal{T}Q = Q$ is $\left( \mathcal{B}_1, \mathcal{B}_2 \right)$-stable.
\end{definition}

\begin{remark}
    This definition indicates that if $\mathcal{T}Q = Q$ is $\left( \mathcal{B}_1, \mathcal{B}_2 \right)$-stable, then $\|Q - Q^*\|_{\mathcal{B}_2} = O\left(  \|\mathcal{T}Q - Q\|_{\mathcal{B}_1} \right)$, as $ \|\mathcal{T}Q - Q\|_{\mathcal{B}_1} \longrightarrow 0$, $\forall Q\in \mathcal{B}_1 \cap \mathcal{B}_2$.    
\end{remark}

\begin{lemma}\label{lem: basic ineq 1}
    For any functions $f, g: \mathcal{X}\rightarrow \mathbb{R}$, we have
    \begin{equation}
        \max_{x\in\mathcal{X}} f(x) - \max_{x\in\mathcal{X}} g(x)\le \max_{x\in\mathcal{X}} \left( f(x) - g(x) \right).
    \end{equation}
\end{lemma}
\begin{proof}
    \begin{equation}
        \max_{x\in\mathcal{X}} f(x) - \max_{x\in\mathcal{X}} g(x) = f(x_f^*) - \max_{x\in\mathcal{X}} g(x) \le f(x_f^*) - g(x_f^*) \le \max_{x\in\mathcal{X}} \left( f(x) - g(x) \right),
    \end{equation}
    where $x_f^*$ is the maximizer of function $f$, i.e. $x_f^* = \arg\max_{x\in\mathcal{X}} f(x)$.
\end{proof}

\begin{lemma}\label{lem: basic ineq 2}
    For any functions $f, g: \mathcal{X}\rightarrow \mathbb{R}$, we have
    \begin{equation}
        \left| \max_{x\in\mathcal{X}} \left( f+g \right)(x) - \max_{x\in\mathcal{X}} f(x) \right| \le \max_{x\in\mathcal{X}} \left| g(x) \right|.
    \end{equation}
\end{lemma}
\begin{proof}
    If $\max_{x\in\mathcal{X}} \left( f+g \right)(x) \ge \max_{x\in\mathcal{X}} f(x)$, we have
    \begin{align}
        &\quad \max_{x\in\mathcal{X}} \left( f+g \right)(x) - \max_{x\in\mathcal{X}} f(x) \\
        &\le \max_{x\in\mathcal{X}} f(x) + \max_{x\in\mathcal{X}} g(x) - \max_{x\in\mathcal{X}} f(x) \\
        &= \max_{x\in\mathcal{X}} g(x) \\
        &\le \max_{x\in\mathcal{X}} \left| g(x) \right|.
    \end{align}
    If $\max_{x\in\mathcal{X}} \left( f+g \right)(x) < \max_{x\in\mathcal{X}} f(x)$, we have
    \begin{equation}
        \max_{x\in\mathcal{X}} f(x) - \max_{x\in\mathcal{X}} \left( f+g \right)(x) \le \max_{x\in\mathcal{X}} \left( -g(x) \right) \le \max_{x\in\mathcal{X}} \left| g(x) \right|,
    \end{equation}
    where the first inequality comes from Lemma \ref{lem: basic ineq 1}.
\end{proof}
\begin{theorem}
    For any MDP $\mathcal{M}$, let $C_{\mathbb{P},p}:= \sup_{(s,a)\in\mathcal{S}\times \mathcal{A}} \left\| \mathbb{P}(\cdot \mid s, a) \right\|_{L^{\frac{p}{p-1}}\left( \mathcal{S} \right)}$. Assume $p$ and $q$ satisfy the following conditions:
    \begin{equation}
        C_{\mathbb{P},p}< \frac{1}{\gamma};\quad
        p \ge \max\left\{1, \frac{\log \left( \left| \mathcal{A}\right|\right) + \log \left( \mu\left( \mathcal{S} \right) \right)}{\log \frac{1}{\gamma C_{\mathbb{P},p}} } \right\}; \quad p \le q \le \infty.
    \end{equation}
    Then, Bellman optimality equation $\mathcal{T}_B Q = Q$ is $\left(  L^q\left( \mathcal{S}\times\mathcal{A} \right),  L^p\left( \mathcal{S}\times\mathcal{A} \right) \right)$-stable.
\end{theorem}

\begin{proof}
    For any $1 \le p \le q \le \infty$ and $Q \in L^p\left( \mathcal{S}\times\mathcal{A} \right) \cap L^q\left( \mathcal{S}\times\mathcal{A} \right)$, denote that 
    \begin{align}
        \mathcal{L}_0 Q(s,a)&:=\gamma \mathbb{E}_{s^{\prime} \sim \mathbb{P}(\cdot \mid s, a)}\left[\max _{a^{\prime} \in \mathcal{A}} Q\left(s^{\prime}, a^{\prime}\right)\right],\\
        \mathcal{L} Q &:= \mathcal{T}_B Q - Q = r + \mathcal{L}_0 Q - Q.
    \end{align}
    
    Let $Q^*$ denote the Bellman optimality Q-function. Note that $\mathcal{T}_B Q^* = Q^* $ and $\mathcal{L} Q^*=0$. Define 
    \begin{align}
        w &= w_Q := Q - Q^*, \\
        f &= f_Q := \mathcal{L} Q = \mathcal{L} Q - \mathcal{L} Q^*.
    \end{align}
    Based on the above notations, we have
    \begin{align}
        f & = \mathcal{L} Q - \mathcal{L} Q^* \\
        &= \mathcal{L}_0 Q - Q - \mathcal{L}_0 Q^* + Q^* \\
        &= \left( \mathcal{L}_0 Q - \mathcal{L}_0 Q^* \right) - \left( Q - Q^* \right) \\
        &= -w + \mathcal{L}_0 \left( Q^* + w \right) - \mathcal{L}_0 Q^*.
    \end{align}
    Then, we have
    \begin{align}
        \left| w(s,a) \right| &=  \left|-f + \mathcal{L}_0 \left( Q^* + w \right) - \mathcal{L}_0 Q^* \right| \bigg|_{(s,a)} \\
        &\le \left| f \right| + \left| \mathcal{L}_0 \left( Q^* + w \right) - \mathcal{L}_0 Q^* \right| \bigg|_{(s,a)}.
    \end{align}
    Thus, we obtain 
    \begin{align}
        \left\| w \right\|_{L^p\left( \mathcal{S}\times\mathcal{A} \right)} &\le \left\| \left| f \right| + \left| \mathcal{L}_0 \left( Q^* + w \right) - \mathcal{L}_0 Q^* \right| \right\|_{L^p\left( \mathcal{S}\times\mathcal{A} \right)} \\
        &\le \left\| f \right\|_{L^p\left( \mathcal{S}\times\mathcal{A} \right)} + \left\| \mathcal{L}_0 \left( Q^* + w \right) - \mathcal{L}_0 Q^* \right\|_{L^p\left( \mathcal{S}\times\mathcal{A} \right)},
    \end{align}
    where the last inequality comes from the Minkowski's inequality. In the following, we analyze the relation between $\left\| \mathcal{L}_0 \left( Q^* + w \right) - \mathcal{L}_0 Q^* \right\|_{L^p\left( \mathcal{S}\times\mathcal{A} \right)}$ and $\left\| w \right\|_{L^p\left( \mathcal{S}\times\mathcal{A} \right)}$.
    \begin{align}
        &\quad \left| \mathcal{L}_0 \left( Q^* + w \right) - \mathcal{L}_0 Q^* \right| \bigg|_{(s,a)} \\
        &= \left| \gamma \mathbb{E}_{s^{\prime} \sim \mathbb{P}(\cdot \mid s, a)}\left[\max _{a^{\prime} \in \mathcal{A}} \left( Q^*\left(s^{\prime}, a^{\prime}\right) + w\left(s^{\prime}, a^{\prime}\right) \right)- \max _{a^{\prime} \in \mathcal{A}} Q^*\left(s^{\prime}, a^{\prime}\right)\right] \right| \\
        &\le \gamma \mathbb{E}_{s^{\prime} \sim \mathbb{P}(\cdot \mid s, a)}\left|\max _{a^{\prime} \in \mathcal{A}} \left( Q^*\left(s^{\prime}, a^{\prime}\right) + w\left(s^{\prime}, a^{\prime}\right) \right)- \max _{a^{\prime} \in \mathcal{A}} Q^*\left(s^{\prime}, a^{\prime}\right)\right| \\
        &\le \gamma \mathbb{E}_{s^{\prime} \sim \mathbb{P}(\cdot \mid s, a)}\left[\max _{a^{\prime} \in \mathcal{A}} \left|  w\left(s^{\prime}, a^{\prime}\right) \right|\right] \\
        &= \gamma \int_{s^\prime } \max _{a^{\prime} \in \mathcal{A}} \left|  w\left(s^{\prime}, a^{\prime}\right) \right| \mathbb{P}(s^\prime \mid s, a) ds^\prime \\
        &\le \gamma \left\| \max _{a \in \mathcal{A}} \left|  w\left(s, a\right) \right| \right\|_{L^p\left( \mathcal{S} \right)} \left( \int_{s^\prime } \left(\mathbb{P}(s^\prime \mid s, a) \right)^{\frac{p}{p-1}} ds^\prime \right)^{1-\frac{1}{p}} \\
        &=  \gamma \left\| \mathbb{P}(\cdot \mid s, a) \right\|_{L^{\frac{p}{p-1}}\left( \mathcal{S} \right)} \left\| \max _{a \in \mathcal{A}} \left|  w\left(s, a\right) \right| \right\|_{L^p\left( \mathcal{S} \right)}. 
    \end{align}
    where the second inequality comes from Lemma \ref{lem: basic ineq 2} and the last inequality comes from the Holder's inequality. Let $C_{\mathbb{P},p}:= \sup_{(s,a)\in\mathcal{S}\times \mathcal{A}} \left\| \mathbb{P}(\cdot \mid s, a) \right\|_{L^{\frac{p}{p-1}}\left( \mathcal{S} \right)} $. Then, we have 
    \begin{align}
        &\quad \left\| \mathcal{L}_0 \left( Q^* + w \right) - \mathcal{L}_0 Q^* \right\|_{L^p\left( \mathcal{S}\times\mathcal{A} \right)} \\
        &\le \left( \int_{ \mathcal{S}\times\mathcal{A} } \left( \gamma \left\| \mathbb{P}(\cdot \mid s, a) \right\|_{L^{\frac{p}{p-1}}\left( \mathcal{S} \right)}  \left\| \max _{a \in \mathcal{A}} \left|  w\left(s, a\right) \right| \right\|_{L^p\left( \mathcal{S} \right)}  \right)^p d\mu(s,a) \right)^{\frac{1}{p}}\\
        &\le \left( \int_{ \mathcal{S}\times\mathcal{A} } 1 d\mu(s,a) \right)^{\frac{1}{p}}   \gamma C_{\mathbb{P},p} \left\| \max _{a \in \mathcal{A}} \left|  w\left(s, a\right) \right| \right\|_{L^p\left( \mathcal{S} \right)}  \\
        &= \left( \mu\left( \mathcal{S}\times\mathcal{A} \right) \right)^{\frac{1}{p}} \gamma C_{\mathbb{P},p} \left\| \max _{a \in \mathcal{A}} \left|  w\left(s, a\right) \right| \right\|_{L^p\left( \mathcal{S} \right)} \\
        &= \gamma C_{\mathbb{P},p} \left( \left| \mathcal{A} \right| \mu\left( \mathcal{S} \right) \right)^{\frac{1}{p}} \left\| \max _{a \in \mathcal{A}} \left|  w\left(s, a\right) \right| \right\|_{L^p\left( \mathcal{S} \right)} \\
        &\le \gamma C_{\mathbb{P},p} \left( \left| \mathcal{A} \right| \mu\left( \mathcal{S} \right) \right)^{\frac{1}{p}} \left\|   w  \right\|_{L^p\left( \mathcal{S}\times\mathcal{A} \right)},
    \end{align}
    where the last inequality comes from $\left\|w\right\|_{l^\infty\left( \mathcal{A}\right)}\le \left\|w\right\|_{l^p\left( \mathcal{A}\right)}$. Thus, when $C_{\mathbb{P},p} < \frac{1}{\gamma}$ and $p\ge \frac{\log \left( \left| \mathcal{A}\right|\right) + \log \left( \mu\left( \mathcal{S} \right) \right)}{\log \frac{1}{\gamma C_{\mathbb{P},p}} }$ and $q\ge p$, we have 
    \begin{equation} \label{eq: stability}
        \left\| w \right\|_{L^p\left( \mathcal{S}\times\mathcal{A} \right)} \le \frac{1}{1-\gamma C_{\mathbb{P},p} \left( \left| \mathcal{A} \right| \mu\left( \mathcal{S} \right) \right)^{\frac{1}{p}}}  \left\| f \right\|_{L^p\left( \mathcal{S}\times\mathcal{A} \right)} \le \frac{\left( \left| \mathcal{A} \right| \mu\left( \mathcal{S} \right) \right)^{\frac{1}{p} - \frac{1}{q}}}{1-\gamma C_{\mathbb{P},p}\left( \left| \mathcal{A} \right| \mu\left( \mathcal{S} \right) \right)^{\frac{1}{p}}}  \left\| f \right\|_{L^q\left( \mathcal{S}\times\mathcal{A} \right)},
    \end{equation}
    where the last inequality comes from $\left\| f \right\|_{L^p\left( \mathcal{S}\times\mathcal{A} \right)} \le  \mu\left( \mathcal{S}\times\mathcal{A} \right)^{\frac{1}{p} - \frac{1}{q}}  \left\| f \right\|_{L^q\left( \mathcal{S}\times\mathcal{A} \right)}$.
\end{proof}
\begin{remark}
    Note that we have proved a stronger conclusion than stability because the equation (\ref{eq: stability}) holds for all $Q$ rather than for $Q$ satisfying $ \|\mathcal{T}Q - Q\|_{\mathcal{B}_1} \longrightarrow 0$.
\end{remark}

\begin{remark}
    When $\mathbb{P}(\cdot \mid s, a)$ is a probability mass function, then we have that $C_{\mathbb{P},p} \le 1 < \frac{1}{\gamma}$ holds for all $1<p\le \infty$. Generally, note that $\lim_{p\rightarrow\infty} C_{\mathbb{P},p} =1$ and as a consequence, when $p$ is large enough, $C_{\mathbb{P},p} < \frac{1}{\gamma}$ holds.
\end{remark}

\subsection{Instability of Bellman Optimality Equations}

\begin{theorem}
    There exists a MDP $\mathcal{M}$ such that Bellman optimality equation $\mathcal{T}_B Q = Q$ is not $\left(  L^p\left( \mathcal{S}\times\mathcal{A} \right),  L^\infty \left( \mathcal{S}\times\mathcal{A} \right) \right)$-stable, for $1 \le p < \infty$.
\end{theorem}
Generally, we have the following theorem.
\begin{theorem}\label{thm: unstable of bellman opt eq}
    There exists an MDP $\mathcal{M}$ such that Bellman optimality equation $\mathcal{T}_B Q = Q$ is not $\left(  L^q\left( \mathcal{S}\times\mathcal{A} \right),  L^p \left( \mathcal{S}\times\mathcal{A} \right) \right)$-stable, for all $1 \le q < p\le \infty$.    
\end{theorem}
\begin{proof}
    In order to show a Bellman optimality equation $\mathcal{T}_B Q = Q$ is not $\left(  L^q\left( \mathcal{S}\times\mathcal{A} \right),  L^p \left( \mathcal{S}\times\mathcal{A} \right) \right)$-stable, it is sufficient and necessary to prove $\forall n\in\mathbb{N}, \forall \delta>0, \exists Q(s,a), \text{ such that }\|\mathcal{T}_BQ-Q\|_{L^q\left( \mathcal{S}\times\mathcal{A} \right)}<\delta, \text{ but }  \|Q-Q^*\|_{L^p\left( \mathcal{S}\times\mathcal{A} \right)} \ge n {\|\mathcal{T}_BQ-Q\|_{L^q\left( \mathcal{S}\times\mathcal{A} \right)}}.$
    
    Define an MDP $\mathcal{M}$ where $\mathcal{S}=[-1,1], \mathcal{A}=\{a_1,a_2\},$
	\begin{equation*}
			\mathbb{P}(s^\prime |s,a_1)=\left\{ \begin{aligned}
				&\mathbbm{1}_{\left\{s^\prime=s-0.1\right\}}, \quad && s\in[-0.9,1] \\
				&\mathbbm{1}_{\left\{s^\prime=s \right\}}    ,      && else 
			\end{aligned}\right.,\quad		
			\mathbb{P}(s^\prime |s,a_2)=\left\{ \begin{aligned}
				&\mathbbm{1}_{\left\{s^\prime=s+0.1\right\}}, \quad && s\in[-1,0.9] \\
				&\mathbbm{1}_{\left\{s^\prime=s \right\}}    ,      && else 
			\end{aligned}\right.,
	\end{equation*}
	$r(s,a_i)=k_is,\ k_2\ge k_1> 0$. The transition function is essentially a deterministic transition dynamic and for convenience, we denote that
        \begin{equation*}
            p(s,a_1)=\left\{ \begin{aligned}
                &s-0.1, \quad && s\in[-0.9,1] \\
                &s     ,      && else 
            \end{aligned}\right.,\quad		
            p(s,a_2)=\left\{ \begin{aligned}
                &s+0.1, \quad && s\in[-1,0.9] \\
                &s     ,      && else 
            \end{aligned}\right..
       \end{equation*}
	
	Let $Q^*(s,a)=Q^{\pi^*}(s,a)$ be the optimal Q-function, where $\pi^*$ is the optimal policy. 
 
 We have $Q^*(s,a_2)\geq Q^*(s,a_1),\ \forall s\ge 0$. To prove this, we define $\bar{\pi}(s)\equiv a_2,\forall s\ge 0$, and thus\begin{equation}\label{equ:barQ}
     Q^{\bar{\pi}}(s,a_2)=\sum_{i=0}^{\infty}\gamma^i r(s_i,a_2),
 \end{equation} where $s_0=s\ge 0,\  s_{i}=p(s_{i-1},a_2)\ge 0,\,i\ge 1.$ Consider Q-function of any policy $\pi$ 
 \begin{equation}\label{equ:Qpi}
     Q^{\pi}(s,\alpha_0)=\sum_{i=0}^{\infty}\gamma^i r(\tilde{s}_i,\pi(\tilde{s}_i)),
 \end{equation}
 where $\pi(\tilde{s}_0)=\alpha_0 \in \mathcal{A},\ \tilde{s}_0=s\ge 0,\ \tilde{s}_{i+1}=p(\tilde{s}_i,\pi(\tilde{s}_i))$ and $r(\tilde{s}_i,\pi(\tilde{s}_i)) = \pi(a_1|\tilde{s}_i) r(\tilde{s}_i,a_1) + \pi(a_2|\tilde{s}_i) r(\tilde{s}_i,a_2)$.
	
 We first notice that all $s_i$ and $\tilde{s}_i$ lie on the grid points $\mathcal{S}\cap \{s + 0.1z:\,z\in \mathbb{Z} \}$, actually, $-\lfloor \frac{1+s}{0.1} \rfloor\le z\le\lfloor \frac{1-s}{0.1} \rfloor$. In the following, we prove $s_{i}\geq \tilde{s}_{i},\ \forall i$. We consider the recursion method and suppose $s_i\geq \tilde{s}_i.$ Then, we have the following two cases. If $s_i\le 0.9$, we obtain
 $$s_{i+1}=s_i+0.1\ge \tilde{s}_i+0.1\ge \tilde{s}_{i+1}.$$
 If $s_i> 0.9$, it follows from $z\le\lfloor \frac{1-s}{0.1} \rfloor$ that $$s_{i+1}=s_i=s+0.1 \lfloor \frac{1-s}{0.1} \rfloor \ge \tilde{s}_{i+1}.$$ 
 Thus, we have $s_{i+1}\geq \tilde{s}_{i+1},\ \forall i$. Note that $s_0=\tilde{s}_0=s$, and by recursion, it can be obtained that $s_{i}\geq \tilde{s}_{i}$ holds for all $i$.
 
 Noticing that the reward $r(s,a)$ is an increasing function in terms of $s$ and satisfies $r(s,a_2)\ge r(s,a_1), \forall s\ge 0,$ we have $$r(s_{i},a_2)\ge r(s_{i},\alpha_{i})\ge r(\tilde{s}_{i},\alpha_{i}),\quad \forall\alpha_{i}\in\mathcal{A},i=0,1,2,\dots ,$$
 where the second inequality is due to $s_i\ge \tilde{s}_i$. As a consequence, \begin{equation}\label{equ:reward}
     r(s_{i},a_2)\ge r(\tilde{s}_i,\pi(\tilde{s}_i)),\quad\forall i=0,1,2,\dots.
 \end{equation}

Combining \eqref{equ:barQ}, \eqref{equ:Qpi}, and \eqref{equ:reward}, we obtain that $Q^{\bar{\pi}}(s,a_2)\ge Q^{\pi}(s,\alpha_0)$. Further, with $\alpha_0=a_2$, we derive $\bar{\pi}(s)={\pi^*}(s)$ on $s>0$. With $\alpha_0=a_1$, we derive $Q^*(s,a_2)=Q^{\bar{\pi}}(s,a_2)\geq Q^*(s,a_1),\ \forall s\ge 0.$
	
We then prove that given $1\le q<p$, $\forall n\in \mathbb{N}, \delta>0$, there exists $Q(s,a)$ with $\|\mathcal{T}_BQ-Q\|_{L^q\left( \mathcal{S}\times\mathcal{A} \right)}\leq \delta$, such that $\|Q-Q^*\|_{L^p\left( \mathcal{S}\times\mathcal{A} \right)}\ge n\|\mathcal{T}_BQ-Q\|_{L^q\left( \mathcal{S}\times\mathcal{A} \right)}$. Let $Q(s,a_1)=Q^*(s,a_1),$ 
\begin{equation*}
	Q(s,a_2)=Q^*(s,a_2)+h\cdot\mathbbm{1}_{(\frac{1}{4}\epsilon, \frac{3}{4}\epsilon)}+\frac{4h}{\epsilon}s\cdot\mathbbm{1}_{(0,\frac{1}{4}\epsilon]}+\left(-\frac{4h}{\epsilon}s+4h\right)\cdot\mathbbm{1}_{[\frac{3}{4}\epsilon,\epsilon)},
\end{equation*}
	where $h>0$, $\epsilon=\min\left\{\left(\frac{\delta}{3h}\right)^q,\left(3n\cdot2^{\frac{1}{p}}\right)^{-\frac{pq}{p-q}}\right\}$ and $\mathbbm{1}_A(s)=\left\{ \begin{aligned}
		&1, &&s\in A\\
		&0, && else
	\end{aligned}\right.$ denotes the indicator function. It can be seen from the definition that \begin{equation}\label{equ:Q-UL}
	    Q^*(s,a_2)\le Q^*(s,a_2)+h\cdot\mathbbm{1}_{(\frac{1}{4}\epsilon, \frac{3}{4}\epsilon)} \leq Q(s,a_2)\leq Q^*(s,a_2)+h\cdot\mathbbm{1}_{(0,\epsilon)}.
	\end{equation}

	We consider the following cases.
	\begin{itemize}
		\item When $s\in (-0.1,-0.1+\epsilon),a=a_2$, 
            \begin{align*}
              \mathcal{T}_BQ(s,a_2) =r(s,a_2)+\gamma\max_{a_i}Q(s+0.1,a_i)=r(s,a_2)+\gamma Q(s+0.1,a_2),
            \end{align*}
              Together with \eqref{equ:Q-UL}, we have
            \begin{align*}
               & Q(s,a_2)=Q^*(s,a_2)= r(s,a_2)+\gamma Q^*(s+0.1,a_2)\\
               \le & \mathcal{T}_BQ(s,a_2) \leq r(s,a_2)+\gamma[Q^*(s+0.1,a_2)+h]\\
               =& Q^*(s,a_2)+h\gamma= Q(s,a_2)+h\gamma,
            \end{align*}
               thus $|\mathcal{T}_BQ(s,a_2)-Q(s,a_2)|\leq h\gamma.$
		\item When $s\in (0,\epsilon),a=a_2$, 
            \begin{align*}                \mathcal{T}_BQ(s,a_2)&=r(s,a_2)+\gamma\max_{a_i}Q(s+0.1,a_i)=r(s,a_2)+\gamma Q(s+0.1,a_2)\\
              &=r(s,a_2)+\gamma Q^*(s+0.1,a_2)=Q^*(s,a_2),
            \end{align*}
               Again from \eqref{equ:Q-UL}, there is
               \begin{align*}
                   |\mathcal{T}_BQ(s,a_2)-Q(s,a_2)|=|Q^*(s,a_2)-Q(s,a_2)|\leq h.
               \end{align*} 
		\item When $s\in (0.1,0.1+\epsilon),a=a_1$, 
            \begin{align*}
              \mathcal{T}_BQ(s,a_1) =r(s,a_1)+\gamma\max_{a_i}Q(s-0.1,a_i)=r(s,a_1)+\gamma Q(s-0.1,a_2),
            \end{align*}
              Utilizing \eqref{equ:Q-UL}, we have
            \begin{align*}
               & Q(s,a_1)=Q^*(s,a_1)= r(s,a_1)+\gamma Q^*(s-0.1,a_2)\\
               \le & \mathcal{T}_BQ(s,a_1) \leq r(s,a_1)+\gamma[Q^*(s-0.1,a_2)+h]\\
               =& Q^*(s,a_1)+h\gamma= Q(s,a_1)+h\gamma,
            \end{align*}
               thus $|\mathcal{T}_BQ(s,a_1)-Q(s,a_1)|\leq h\gamma.$

		\item Otherwise, $$\mathcal{T}_BQ(s,a_i)=r(s,a_i)+\gamma Q\left(p(s,a_i),\pi^*(p(s,a_i))\right)=Q^*(s,a_i),$$
		also note that $Q(s,\cdot)=Q^*(s,\cdot)$ for $s\notin(0,\epsilon)$, thus $$|\mathcal{T}_BQ(s,a)-Q(s,a)|=|Q^*(s,a)-Q(s,a)|=0.$$
	\end{itemize}
From the analysis above, we have
	\begin{equation}\label{equ:TQ-Q}
		\|\mathcal{T}_BQ-Q\|_{L^q\left( \mathcal{S}\times\mathcal{A} \right)}\leq (2h\gamma+h)\epsilon^{\frac{1}{q}}\leq 3h\epsilon^{\frac{1}{q}} \le\delta,
	\end{equation}
and
	\begin{equation}\label{equ:Q-Q*}
		\|Q-Q^*\|_{L^p\left( \mathcal{S}\times\mathcal{A} \right)}\geq\|(Q-Q^*)\mathbbm{1}_{( \frac{1}{4}\epsilon, \frac{3}{4}\epsilon )}\|_{L^p\left( \mathcal{S}\times\mathcal{A} \right)}\geq h\left(\frac{\epsilon}{2}\right)^{\frac{1}{p}} \ge n \|\mathcal{T}_BQ-Q\|_{L^q\left( \mathcal{S}\times\mathcal{A} \right)}.
	\end{equation}
Inequality \eqref{equ:TQ-Q} and \eqref{equ:Q-Q*} come from $\epsilon=\min\left\{\left(\frac{\delta}{3h}\right)^q,\left(3n\cdot2^{\frac{1}{p}}\right)^{-\frac{pq}{p-q}}\right\}$, which prove the desired property.
\end{proof}

\section{Theorems and Proofs of Stability Analysis of DQN}
In practical DQN training, we use the following loss:
\begin{equation}
    \mathcal{L}(\theta)=\frac{1}{\left| \mathcal{B} \right|} \sum_{(s,a,r,s^\prime)\in \mathcal{B}} \left|r + \gamma \max_{a^\prime} Q(s^\prime,a^\prime; \bar{\theta}) - Q(s,a; \theta)  \right|,
\end{equation}
where $\mathcal{B}$ represents a batch of transition pairs sampled from the replay buffer and $\bar{\theta}$ is the parameter of target network.

$\mathcal{L}(\theta)$ is a approximation of the following objective:
    \begin{align}
        \mathcal{L}(Q;\pi)&=\mathbb{E}_{s\sim d_{\mu_0}^\pi(\cdot)} \mathbb{E}_{a\sim \pi(\cdot|s)} \left| \mathcal{T}_B Q(s,a) - Q(s,a)  \right| \\
        &=\mathbb{E}_{(s,a)\sim d_{\mu_0}^\pi(\cdot,\cdot)} \left| \mathcal{T}_B Q(s,a) - Q(s,a)  \right|,
    \end{align}
where $d_{\mu_0}^\pi(s)=\mathbb{E}_{s_0\sim \mu_0} \left[ (1-\gamma) \sum_{t=0}^{\infty} \gamma^t \operatorname{Pr}^\pi(s_t=s|s_0) \right]$ is the state visitation distribution
and $d_{\mu_0}^\pi(s,a)=\mathbb{E}_{s_0\sim \mu_0} \left[ (1-\gamma) \sum_{t=0}^{\infty} \gamma^t \operatorname{Pr}^\pi(s_t=s, a_t=a|s_0) \right]$ is the state-action visitation distribution.

\subsection{Definition and Properties of $L^{p,d_{\mu_0}^\pi}$} \label{app:seminorm}
\begin{definition}
    Given a policy $\pi$, for any function $f:\mathcal{S}\times\mathcal{A}\rightarrow \mathbb{R}$ and $1\le p\le\infty$, we define the seminorm $L^{p,d_{\mu_0}^\pi}$.
        \begin{itemize}
            \item If $d_{\mu_0}^\pi$ is a probability density function, we define
            \begin{equation}
                \begin{aligned}
                    \left\| f \right\|_{L^{p,d_{\mu_0}^\pi}\left( \mathcal{S}\times\mathcal{A} \right)} &:= \left\| d_{\mu_0}^\pi f \right\|_{L^{p}\left( \mathcal{S}\times\mathcal{A} \right)} \\
                    &= \left(\int_{(s,a)\in\mathcal{S}\times\mathcal{A}} \left| d_{\mu_0}^\pi(s,a)  f(s,a) \right|^p d \mu(s,a) \right)^{\frac{1}{p}}.
                \end{aligned}
            \end{equation}
            \item If $d_{\mu_0}^\pi$ is a probability mass function, we define
            \begin{equation}
                    \left\| f \right\|_{L^{p,d_{\mu_0}^\pi}\left( \mathcal{S}\times\mathcal{A} \right)} := \left(\sum_{(s,a)\in\mathcal{S}\times\mathcal{A}} \left| d_{\mu_0}^\pi(s,a)  f(s,a) \right|^p  \right)^{\frac{1}{p}}.          
            \end{equation}
        \end{itemize}
\end{definition}

\begin{remark}
    Note that $\mathcal{L}(Q;\pi) = \left\|\mathcal{T}_B Q - Q \right\|_{L^{1,d_{\mu_0}^\pi}\left( \mathcal{S}\times\mathcal{A} \right)}$.
\end{remark}
\begin{theorem}
    For any $d_{\mu_0}^\pi(s,a)$ and $1\le p \le \infty$, $L^{p,d_{\mu_0}^\pi}$ is a seminorm.
\end{theorem}
\begin{proof}
    Firstly, we show that $L^{p,d_{\mu_0}^\pi}$ satisfies the absolute homogeneity. For any function $f$ and $\lambda\in\mathbb{R}$, we have
    \begin{align}
        \left\| \lambda f \right\|_{L^{p,d_{\mu_0}^\pi}\left( \mathcal{S}\times\mathcal{A} \right)} = \left\| d_{\mu_0}^\pi \lambda f \right\|_{L^{p}\left( \mathcal{S}\times\mathcal{A} \right)} = \left|\lambda \right| \left\| d_{\mu_0}^\pi f \right\|_{L^{p}\left( \mathcal{S}\times\mathcal{A} \right)} = \left|\lambda \right|  \left\| f \right\|_{L^{p,d_{\mu_0}^\pi}\left( \mathcal{S}\times\mathcal{A} \right)}.
    \end{align}

    Next, we show that the triangle inequality holds. For any functions $f$ and $g$, we have
    \begin{align}
         \left\| f + g \right\|_{L^{p,d_{\mu_0}^\pi}\left( \mathcal{S}\times\mathcal{A} \right)} &= \left\| d_{\mu_0}^\pi (f+g) \right\|_{L^{p}\left( \mathcal{S}\times\mathcal{A} \right)} \\
         &\le \left\| d_{\mu_0}^\pi f \right\|_{L^{p}\left( \mathcal{S}\times\mathcal{A} \right)} + \left\| d_{\mu_0}^\pi g \right\|_{L^{p}\left( \mathcal{S}\times\mathcal{A} \right)} \\
         &= \left\| f \right\|_{L^{p,d_{\mu_0}^\pi}\left( \mathcal{S}\times\mathcal{A} \right)} + \left\| g \right\|_{L^{p,d_{\mu_0}^\pi}\left( \mathcal{S}\times\mathcal{A} \right)},
    \end{align}
    where the inequality comes from the triangle inequality of $L^{p}\left( \mathcal{S}\times\mathcal{A} \right)$.
\end{proof}
\begin{theorem}
    If $d_{\mu_0}^\pi(s,a)>0$ for almost everywhere $(s,a)\in\mathcal{S}\times\mathcal{A}$, then $L^{p,d_{\mu_0}^\pi}\left( \mathcal{S}\times\mathcal{A} \right) := \left\{f \left| \left\| f \right\|_{L^{p,d_{\mu_0}^\pi}\left( \mathcal{S}\times\mathcal{A} \right)} \le \infty \right. \right\}$ is a Banach space, for $1\le p \le \infty$.
\end{theorem}
\begin{proof}
    Firstly, we show the $L^{p,d_{\mu_0}^\pi}$ is positive definite. If $\left\| f \right\|_{L^{p,d_{\mu_0}^\pi}\left( \mathcal{S}\times\mathcal{A} \right)} = 0$, we have that $d_{\mu_0}^\pi(s,a) f(s,a) = 0$, for almost everywhere $(s,a) \in \mathcal{S}\times\mathcal{A}$. Due to the nonnegativity of $d_{\mu_0}^\pi(s,a)$, we have $ f(s,a) = 0$, for almost everywhere $(s,a) \in \mathcal{S}\times\mathcal{A}$.

    We show the completeness of $L^{p,d_{\mu_0}^\pi}\left( \mathcal{S}\times\mathcal{A} \right)$ in the following. For any Cauchy sequence $\left\{ f_i \right\} \subset L^{p,d_{\mu_0}^\pi}\left( \mathcal{S}\times\mathcal{A} \right)$, $\left\{ d_{\mu_0}^\pi f_i \right\}$ is a Cauchy sequence in $L^{p}\left( \mathcal{S}\times\mathcal{A} \right)$. Then, due to the completeness of $L^{p}\left( \mathcal{S}\times\mathcal{A} \right)$, there exists $g\in L^{p}\left( \mathcal{S}\times\mathcal{A} \right)$ such that
    $$\lim_{i\rightarrow \infty} \left\| d_{\mu_0}^\pi f_i \right\|_{L^{p}\left( \mathcal{S}\times\mathcal{A} \right)} = \left\| g \right\|_{L^{p}\left( \mathcal{S}\times\mathcal{A} \right)}.$$
    Let $f = \frac{g}{d_{\mu_0}^\pi}$ for almost everywhere $(s,a)\in\mathcal{S}\times\mathcal{A}$, we have $\left\| f \right\|_{L^{p,d_{\mu_0}^\pi}\left( \mathcal{S}\times\mathcal{A} \right)} =  \left\| g \right\|_{L^{p}\left( \mathcal{S}\times\mathcal{A} \right)} \le \infty$ and 
    $$ \lim_{i\rightarrow \infty}\left\| f_i \right\|_{L^{p,d_{\mu_0}^\pi}\left( \mathcal{S}\times\mathcal{A} \right)} = \left\| f \right\|_{L^{p,d_{\mu_0}^\pi}\left( \mathcal{S}\times\mathcal{A} \right)} .$$
    Thus, $L^{p,d_{\mu_0}^\pi}\left( \mathcal{S}\times\mathcal{A} \right)$ is a Banach space.
\end{proof}
We analyze the properties of $L^{p,d_{\mu_0}^\pi}\left( \mathcal{S}\times\mathcal{A} \right)$ in the following lemma.
\begin{lemma}\label{lem: norm property}
    Given a policy $\pi$, for any function $f:\mathcal{S}\times\mathcal{A}\rightarrow \mathbb{R}$, then we have the following properties.
    \begin{itemize}
        \item If $M_{d_{\mu_0}^\pi}:=\sup_{(s,a)\in\mathcal{S}\times\mathcal{A}} d_{\mu_0}^\pi(s,a) <\infty$, then 
        $$ \left\| f \right\|_{L^{p,d_{\mu_0}^\pi}\left( \mathcal{S}\times\mathcal{A} \right)} \le M_{d_{\mu_0}^\pi} \left\|  f \right\|_{L^{p}\left( \mathcal{S}\times\mathcal{A} \right)} ,\ \forall 1\le p\le\infty.$$
        \item If $C_{d_{\mu_0}^\pi}:=\inf_{(s,a)\in\mathcal{S}\times\mathcal{A}} d_{\mu_0}^\pi(s,a)  >0$, then we have 
        $$  C_{d_{\mu_0}^\pi} \left\|  f \right\|_{L^{p}\left( \mathcal{S}\times\mathcal{A} \right)} \le \left\| f \right\|_{L^{p,d_{\mu_0}^\pi}\left( \mathcal{S}\times\mathcal{A} \right)} ,\ \forall 1\le p\le\infty.$$
        \item $\left\| f \right\|_{L^{1,d_{\mu_0}^\pi}\left( \mathcal{S}\times\mathcal{A} \right)}\le \left\| d_{\mu_0}^\pi \right\|_{L^{\frac{p}{p-1}}\left( \mathcal{S}\times\mathcal{A} \right)} \left\|  f \right\|_{L^{p}\left( \mathcal{S}\times\mathcal{A} \right)},\ \forall 1\le p\le\infty$.
        \item If $d_{\mu_0}^\pi(s,a) \neq 0$ for almost everywhere $(s,a)\in\mathcal{S}\times\mathcal{A}$, then we have 
        $$ \left\|  f \right\|_{L^{1}\left( \mathcal{S}\times\mathcal{A} \right)} \le C_{d_{\mu_0}^\pi,p} \left\| f \right\|_{L^{p,d_{\mu_0}^\pi}\left( \mathcal{S}\times\mathcal{A} \right)},\ \forall 1< p<\infty,$$ 
        where $C_{d_{\mu_0}^\pi,p}=\left( \int_{(s,a)\in\mathcal{S}\times\mathcal{A}} \left| d_{\mu_0}^\pi(s,a)  \right|^{-\frac{p}{p-1}} d \mu(s,a) \right)^{\frac{p-1}{p}}$.
        \item If $d_{\mu_0}^\pi(s,a) \neq 0$ for almost everywhere $(s,a)\in\mathcal{S}\times\mathcal{A}$, then we have 
        $$ \left\|  f \right\|_{L^{p}\left( \mathcal{S}\times\mathcal{A} \right)} \le C_{d_{\mu_0}^\pi,p} \left\| f \right\|_{L^{p^2,d_{\mu_0}^\pi}\left( \mathcal{S}\times\mathcal{A} \right)},\ \forall 1< p<\infty,$$ 
        where $C_{d_{\mu_0}^\pi,p}=\left( \int_{(s,a)\in\mathcal{S}\times\mathcal{A}} \left| d_{\mu_0}^\pi(s,a)  \right|^{-\frac{p^2}{p-1}} d \mu(s,a) \right)^{\frac{p-1}{p^2}}$.
    \end{itemize}
\end{lemma}
\begin{proof}
(1)
    If $M_{d_{\mu_0}^\pi}:=\sup_{(s,a)\in\mathcal{S}\times\mathcal{A}} d_{\mu_0}^\pi(s,a) <\infty$, we have
    \begin{align}
        \left\| f \right\|_{L^{p,d_{\mu_0}^\pi}\left( \mathcal{S}\times\mathcal{A} \right)} &= \left(\int_{(s,a)\in\mathcal{S}\times\mathcal{A}} \left| d_{\mu_0}^\pi(s,a) f(s,a) \right|^p d \mu(s,a) \right)^{\frac{1}{p}} \\
        &\le M_{d_{\mu_0}^\pi} \left(\int_{(s,a)\in\mathcal{S}\times\mathcal{A}} \left| f(s,a) \right|^p d \mu(s,a) \right)^{\frac{1}{p}} \\
        &=  M_{d_{\mu_0}^\pi} \left\|  f \right\|_{L^{p}\left( \mathcal{S}\times\mathcal{A} \right)}.
    \end{align}
    
(2)
    If $C_{d_{\mu_0}^\pi}:=\inf_{(s,a)\in\mathcal{S}\times\mathcal{A}} d_{\mu_0}^\pi(s,a) >0$, we have
    \begin{align}
        \left\| f \right\|_{L^{p,d_{\mu_0}^\pi}\left( \mathcal{S}\times\mathcal{A} \right)} &= \left(\int_{(s,a)\in\mathcal{S}\times\mathcal{A}} \left| d_{\mu_0}^\pi(s,a) f(s,a) \right|^p d \mu(s,a) \right)^{\frac{1}{p}} \\
        &\ge C_{d_{\mu_0}^\pi} \left\|  f \right\|_{L^{p}\left( \mathcal{S}\times\mathcal{A} \right)}.
    \end{align}
    
(3)
    \begin{align}
        \left\| f \right\|_{L^{1,d_{\mu_0}^\pi}\left( \mathcal{S}\times\mathcal{A} \right)} &= \int_{(s,a)\in\mathcal{S}\times\mathcal{A}} \left| d_{\mu_0}^\pi(s,a) f(s,a) \right| d \mu(s,a) \\
        &\le \left\|  f \right\|_{L^{p}\left( \mathcal{S}\times\mathcal{A} \right)} \left(\int_{(s,a)\in\mathcal{S}\times\mathcal{A}} \left| d_{\mu_0}^\pi(s,a) \right|^{\frac{p}{p-1}} d \mu(s,a) \right)^{1-\frac{1}{p}} \\
        &= \left\| d_{\mu_0}^\pi \right\|_{L^{\frac{p}{p-1}}\left( \mathcal{S}\times\mathcal{A} \right)} \left\|  f \right\|_{L^{p}\left( \mathcal{S}\times\mathcal{A} \right)},
    \end{align}
    where the first inequality comes from the Holder's inequality.
    
(4)
    For $1<p<\infty$, we have
    \begin{align}
         \left\| f \right\|_{L^{p,d_{\mu_0}^\pi}\left( \mathcal{S}\times\mathcal{A} \right)}^p &= \int_{(s,a)\in\mathcal{S}\times\mathcal{A}} \left| d_{\mu_0}^\pi(s,a) f(s,a) \right|^p d \mu(s,a)  \\
         &\ge \left\|  f \right\|_{L^{1}\left( \mathcal{S}\times\mathcal{A} \right)}^p \left( \int_{(s,a)\in\mathcal{S}\times\mathcal{A}} \left| d_{\mu_0}^\pi(s,a) \right|^{-\frac{p}{p-1}} d \mu(s,a) \right)^{-(p-1)},
    \end{align}
    where the inequality comes from reverse Holder's inequality.

(5)
    Further, we have
    \begin{align}
        \left\| f \right\|_{L^{{p^2},d_{\mu_0}^\pi}\left( \mathcal{S}\times\mathcal{A} \right)}^{p^2} &= \int_{(s,a)\in\mathcal{S}\times\mathcal{A}} \left| d_{\mu_0}^\pi(s,a) f(s,a) \right|^{p^2} d \mu(s,a)  \\
        &\ge \left\|  f \right\|_{L^{p}\left( \mathcal{S}\times\mathcal{A} \right)}^{p^2} \left( \int_{(s,a)\in\mathcal{S}\times\mathcal{A}} \left| d_{\mu_0}^\pi(s,a) \right|^{-\frac{p^2}{p-1}} d \mu(s,a) \right)^{-(p-1)},
    \end{align}
    where the inequality comes from reverse Holder's inequality.
\end{proof}

\begin{remark}
    Note that in a practical Q-learning scheme, we take the $\epsilon$-greedy policy for exploration and as a result, for any state-action pair $(s, a)$, we can visit it with positive probability, i.e. $d_{\mu_0}^\pi(s,a) >0$. Furthermore, the condition, $d_{\mu_0}^\pi(s,a)\neq 0$ for almost everywhere $(s,a)\in\mathcal{S}\times\mathcal{A}$, always holds.
\end{remark}

\subsection{Stability of DQN: the Good} \label{app: Stability of DQN: the Good}

\begin{theorem}
    For any MDP $\mathcal{M}$ and fixed policy $\pi$, assume $C_{d_{\mu_0}^\pi}:=\inf_{(s,a)\in\mathcal{S}\times\mathcal{A}} d_{\mu_0}^\pi(s,a)  >0$ and let $C_{\mathbb{P},p}:= \sup_{(s,a)\in\mathcal{S}\times \mathcal{A}} \left\| \mathbb{P}(\cdot \mid s, a) \right\|_{L^{\frac{p}{p-1}}\left( \mathcal{S} \right)}$. Assume $p$ and $q$ satisfy the following conditions:
    \begin{equation}
        C_{\mathbb{P},p}< \frac{1}{\gamma};\quad
        p \ge \max\left\{1, \frac{\log \left( \left| \mathcal{A}\right|\right) + \log \left( \mu\left( \mathcal{S} \right) \right)}{\log \frac{1}{\gamma C_{\mathbb{P},p}} } \right\}; \quad p \le q \le \infty.
    \end{equation}
    Then, Bellman optimality equation $\mathcal{T}_B Q = Q$ is $\left(  L^{q,d_{\mu_0}^\pi}\left( \mathcal{S}\times\mathcal{A} \right),  L^p\left( \mathcal{S}\times\mathcal{A} \right) \right)$-stable.
\end{theorem}
\begin{proof}
    For any $1 \le p \le q \le \infty$ and $Q \in L^p\left( \mathcal{S}\times\mathcal{A} \right) \cap L^q\left( \mathcal{S}\times\mathcal{A} \right)$, denote that 
    \begin{align}
        \mathcal{L}_0 Q(s,a)&:=\gamma \mathbb{E}_{s^{\prime} \sim \mathbb{P}(\cdot \mid s, a)}\left[\max _{a^{\prime} \in \mathcal{A}} Q\left(s^{\prime}, a^{\prime}\right)\right],\\
        \mathcal{L} Q &:= \mathcal{T}_B Q - Q = r + \mathcal{L}_0 Q - Q.
    \end{align}
    
    Let $Q^*$ denote the Bellman optimality Q-function. Note that $\mathcal{T}_B Q^* = Q^* $ and $\mathcal{L} Q^*=0$. Define 
    \begin{align}
        w &= w_Q := Q - Q^*, \\
        f &= f_Q := \mathcal{L} Q = \mathcal{L} Q - \mathcal{L} Q^*.
    \end{align}
    Based on the above notations, we have
    \begin{align}
        f & = \mathcal{L} Q - \mathcal{L} Q^* \\
        &= -w + \mathcal{L}_0 \left( Q^* + w \right) - \mathcal{L}_0 Q^*.
    \end{align}
    
    According to the inequality (\ref{eq: stability}), we have that when $C_{\mathbb{P},p} < \frac{1}{\gamma}$ and $p\ge \frac{\log \left( \left| \mathcal{A}\right|\right) + \log \left( \mu\left( \mathcal{S} \right) \right)}{\log \frac{1}{\gamma C_{\mathbb{P},p}} }$ and $q\ge p$, we have 
    \begin{equation} 
        \left\| w \right\|_{L^p\left( \mathcal{S}\times\mathcal{A} \right)} \le \frac{\left( \left| \mathcal{A} \right| \mu\left( \mathcal{S} \right) \right)^{\frac{1}{p} - \frac{1}{q}}}{1-\gamma C_{\mathbb{P},p}\left( \left| \mathcal{A} \right| \mu\left( \mathcal{S} \right) \right)^{\frac{1}{p}}}  \left\| f \right\|_{L^q\left( \mathcal{S}\times\mathcal{A} \right)}.
    \end{equation}
    According to Lemma \ref{lem: norm property}, when $C_{d_{\mu_0}^\pi}:=\inf_{(s,a)\in\mathcal{S}\times\mathcal{A}} d_{\mu_0}^\pi(s,a)  >0$, we have
    \begin{equation}
        \left\| w \right\|_{L^p\left( \mathcal{S}\times\mathcal{A} \right)} \le \frac{\left( \left| \mathcal{A} \right| \mu\left( \mathcal{S} \right) \right)^{\frac{1}{p} - \frac{1}{q}}}{ C_{d_{\mu_0}^\pi} \left(1-\gamma C_{\mathbb{P},p}\left( \left| \mathcal{A} \right| \mu\left( \mathcal{S} \right) \right)^{\frac{1}{p}} \right) }  \left\| f \right\|_{L^{q,d_{\mu_0}^\pi}\left( \mathcal{S}\times\mathcal{A} \right)}.
    \end{equation}
\end{proof}
\begin{remark}
    Note that in a practical Q-learning scheme, we take the $\epsilon$-greedy policy for exploration and as a result, for any state-action pair $(s, a)$, we can visit it with positive probability, and thus the condition $C_{d_{\mu_0}^\pi}>0$ is fulfilled.
\end{remark}
We also demonstrate a theorem with better bound yet stronger condition. 

\begin{theorem}
    For any MDP $\mathcal{M}$ and fixed policy $\pi$, assume $C_{d_{\mu_0}^\pi}:=\inf_{(s,a)\in\mathcal{S}\times\mathcal{A}} d_{\mu_0}^\pi(s,a)  >0$. Assume $p$, $q$ and $\gamma$ satisfy the following conditions:
    \begin{equation}
        C_{d_{\mu_0}^\pi, \mathbb{P},p}:= \frac{\left\| d_{\mu_0}^\pi \right\|_{L^{p^2}\left( \mathcal{S}\times\mathcal{A} \right)} C_{\mathbb{P},p}}{C_{d_{\mu_0}^\pi}}<\frac{1}{\gamma} ;\quad p\ge \frac{\log \left( \left| \mathcal{A}\right|\right) + \log \left( \mu\left( \mathcal{S} \right) \right)}{\log \frac{1}{\gamma C_{d_{\mu_0}^\pi, \mathbb{P},p}} } - 1; \quad q\ge p^2.
    \end{equation}
    Then, Bellman optimality equation $\mathcal{T}_B Q = Q$ is $\left(  L^{q,d_{\mu_0}^\pi}\left( \mathcal{S}\times\mathcal{A} \right),  L^p\left( \mathcal{S}\times\mathcal{A} \right) \right)$-stable.
\end{theorem}
\begin{proof}
    For any $1 \le p \le q \le \infty$ and $Q \in L^p\left( \mathcal{S}\times\mathcal{A} \right) \cap L^q\left( \mathcal{S}\times\mathcal{A} \right)$, denote that 
    \begin{align}
        \mathcal{L}_0 Q(s,a)&:=\gamma \mathbb{E}_{s^{\prime} \sim \mathbb{P}(\cdot \mid s, a)}\left[\max _{a^{\prime} \in \mathcal{A}} Q\left(s^{\prime}, a^{\prime}\right)\right],\\
        \mathcal{L} Q &:= \mathcal{T}_B Q - Q = r + \mathcal{L}_0 Q - Q.
    \end{align}
    
    Let $Q^*$ denote the Bellman optimality Q-function. Note that $\mathcal{T}_B Q^* = Q^* $ and $\mathcal{L} Q^*=0$. Define 
    \begin{align}
        w &= w_Q := Q - Q^*, \\
        f &= f_Q := \mathcal{L} Q = \mathcal{L} Q - \mathcal{L} Q^*.
    \end{align}
    Based on the above notations, we have
    \begin{align}
        f & = \mathcal{L} Q - \mathcal{L} Q^* \\
        &= -w + \mathcal{L}_0 \left( Q^* + w \right) - \mathcal{L}_0 Q^*.
    \end{align}
    Then, we have
    \begin{align}
        \left| w(s,a) \right| &=  \left|-f + \mathcal{L}_0 \left( Q^* + w \right) - \mathcal{L}_0 Q^* \right| \bigg|_{(s,a)} \\
        &\le \left| f \right| + \left| \mathcal{L}_0 \left( Q^* + w \right) - \mathcal{L}_0 Q^* \right| \bigg|_{(s,a)}.
    \end{align}
    Thus, we obtain 
    \begin{align}
        \left\| w \right\|_{L^{p^2,d_{\mu_0}^\pi}\left( \mathcal{S}\times\mathcal{A} \right)} &\le \left\|d_{\mu_0}^\pi \left| f \right| + d_{\mu_0}^\pi \left| \mathcal{L}_0 \left( Q^* + w \right) - \mathcal{L}_0 Q^* \right| \right\|_{L^{p^2}\left( \mathcal{S}\times\mathcal{A} \right)} \\
        &\le \left\| f \right\|_{L^{p^2,d_{\mu_0}^\pi}\left( \mathcal{S}\times\mathcal{A} \right)} + \left\| \mathcal{L}_0 \left( Q^* + w \right) - \mathcal{L}_0 Q^* \right\|_{L^{p^2,d_{\mu_0}^\pi}\left( \mathcal{S}\times\mathcal{A} \right)},
    \end{align}
    where the last inequality comes from the Minkowski's inequality. Owing to Lemma \ref{lem: norm property}, we have
    \begin{align}
        \left\| w \right\|_{L^{p^2,d_{\mu_0}^\pi}\left( \mathcal{S}\times\mathcal{A} \right)} \ge \frac{1}{C_{d_{\mu_0}^\pi,p}} \left\|  w \right\|_{L^{p}\left( \mathcal{S}\times\mathcal{A} \right)},
    \end{align}
    where $C_{d_{\mu_0}^\pi,p}=\left( \int_{(s,a)\in\mathcal{S}\times\mathcal{A}} \left| d_{\mu_0}^\pi(s,a)  \right|^{-\frac{p^2}{p-1}} d \mu(s,a) \right)^{\frac{p-1}{p^2}}$.
    In the following, we analyze the relation between $\left\| \mathcal{L}_0 \left( Q^* + w \right) - \mathcal{L}_0 Q^* \right\|_{L^{p^2,d_{\mu_0}^\pi}\left( \mathcal{S}\times\mathcal{A} \right)}$ and $\left\| w \right\|_{L^p\left( \mathcal{S}\times\mathcal{A} \right)}$.
    \begin{align}
        &\quad \left| \mathcal{L}_0 \left( Q^* + w \right) - \mathcal{L}_0 Q^* \right| \bigg|_{(s,a)} \\
        &= \left| \gamma \mathbb{E}_{s^{\prime} \sim \mathbb{P}(\cdot \mid s, a)}\left[\max _{a^{\prime} \in \mathcal{A}} \left( Q^*\left(s^{\prime}, a^{\prime}\right) + w\left(s^{\prime}, a^{\prime}\right) \right)- \max _{a^{\prime} \in \mathcal{A}} Q^*\left(s^{\prime}, a^{\prime}\right)\right] \right| \\
        &\le \gamma \mathbb{E}_{s^{\prime} \sim \mathbb{P}(\cdot \mid s, a)}\left|\max _{a^{\prime} \in \mathcal{A}} \left( Q^*\left(s^{\prime}, a^{\prime}\right) + w\left(s^{\prime}, a^{\prime}\right) \right)- \max _{a^{\prime} \in \mathcal{A}} Q^*\left(s^{\prime}, a^{\prime}\right)\right| \\
        &\le \gamma \mathbb{E}_{s^{\prime} \sim \mathbb{P}(\cdot \mid s, a)}\left[\max _{a^{\prime} \in \mathcal{A}} \left|  w\left(s^{\prime}, a^{\prime}\right) \right|\right] \\
        &= \gamma \int_{s^\prime } \max _{a^{\prime} \in \mathcal{A}} \left|  w\left(s^{\prime}, a^{\prime}\right) \right| \mathbb{P}(s^\prime \mid s, a) ds^\prime \\
        &\le \gamma \left\| \max _{a \in \mathcal{A}} \left|  w\left(s, a\right) \right| \right\|_{L^p\left( \mathcal{S} \right)} \left( \int_{s^\prime } \left(\mathbb{P}(s^\prime \mid s, a) \right)^{\frac{p}{p-1}} ds^\prime \right)^{1-\frac{1}{p}} \\
        &= \gamma \left\| \mathbb{P}(\cdot \mid s, a) \right\|_{L^{\frac{p}{p-1}}\left( \mathcal{S} \right)} \left\| \max _{a \in \mathcal{A}} \left|  w\left(s, a\right) \right| \right\|_{L^p\left( \mathcal{S} \right)}, 
    \end{align}
    where the second inequality comes from Lemma \ref{lem: basic ineq 2} and the third inequality comes from the Holder's inequality. let $C_{\mathbb{P},p}:= \sup_{(s,a)\in\mathcal{S}\times \mathcal{A}} \left\| \mathbb{P}(\cdot \mid s, a) \right\|_{L^{\frac{p}{p-1}}\left( \mathcal{S} \right)}$ and then, we have 
    \begin{align}
        &\quad \left\| \mathcal{L}_0 \left( Q^* + w \right) - \mathcal{L}_0 Q^* \right\|_{L^{p^2,d_{\mu_0}^\pi}\left( \mathcal{S}\times\mathcal{A} \right)} \\
        &\le \left( \int_{\mathcal{S}\times\mathcal{A} } \left( \gamma \left\| \mathbb{P}(\cdot \mid s, a) \right\|_{L^{\frac{p}{p-1}}\left( \mathcal{S} \right)} \left\| \max _{a \in \mathcal{A}} \left|  w\left(s, a\right) \right| \right\|_{L^p\left( \mathcal{S} \right)} d_{\mu_0}^\pi(s,a) \right)^{p^2} d\mu(s,a) \right)^{\frac{1}{p^2}}\\
        &= \left( \int_{ \mathcal{S}\times\mathcal{A} } \left(d_{\mu_0}^\pi(s,a)\right)^{p^2} d\mu(s,a) \right)^{\frac{1}{p^2}}   \gamma  C_{\mathbb{P},p} \left\| \max _{a \in \mathcal{A}} \left|  w\left(s, a\right) \right| \right\|_{L^p\left( \mathcal{S} \right)}  \\
        &\le \gamma \left\| d_{\mu_0}^\pi \right\|_{L^{p^2}\left( \mathcal{S}\times\mathcal{A} \right)}    C_{\mathbb{P},p} \left\| \max _{a \in \mathcal{A}} \left|  w\left(s, a\right) \right| \right\|_{L^p\left( \mathcal{S} \right)} \\
        &\le \gamma \left\| d_{\mu_0}^\pi \right\|_{L^{p^2}\left( \mathcal{S}\times\mathcal{A} \right)}   C_{\mathbb{P},p} \left\|   w  \right\|_{L^p\left( \mathcal{S}\times\mathcal{A} \right)},
    \end{align}
    where  the last inequality comes from $\left\|w\right\|_{l^\infty\left( \mathcal{A}\right)}\le \left\|w\right\|_{l^p\left( \mathcal{A}\right)}$. 
    Then, we have that
    \begin{align}
        \left(\frac{1}{C_{d_{\mu_0}^\pi,p}} - \gamma \left\| d_{\mu_0}^\pi \right\|_{L^{p^2}\left( \mathcal{S}\times\mathcal{A} \right)}   C_{\mathbb{P},p} \right) \left\|   w  \right\|_{L^p\left( \mathcal{S}\times\mathcal{A} \right)} \le \left\| f \right\|_{L^{p^2,d_{\mu_0}^\pi}\left( \mathcal{S}\times\mathcal{A} \right)} .
    \end{align}
    When $C_{d_{\mu_0}^\pi}:=\inf_{(s,a)\in\mathcal{S}\times\mathcal{A}} d_{\mu_0}^\pi(s,a)  >0$, we have 
    \begin{align}
        C_{d_{\mu_0}^\pi,p} \le \frac{\left( \left| \mathcal{A} \right| \mu\left( \mathcal{S} \right) \right)^{\frac{p-1}{p^2}}}{C_{d_{\mu_0}^\pi}} 
    \end{align}
    Thus, when the following conditions hold
    $$C_{d_{\mu_0}^\pi, \mathbb{P},p}:= \frac{\left\| d_{\mu_0}^\pi \right\|_{L^{p^2}\left( \mathcal{S}\times\mathcal{A} \right)} C_{\mathbb{P},p}}{C_{d_{\mu_0}^\pi}}<\frac{1}{\gamma} ;\quad p\ge \frac{\log \left( \left| \mathcal{A}\right|\right) + \log \left( \mu\left( \mathcal{S} \right) \right)}{\log \frac{1}{\gamma C_{d_{\mu_0}^\pi, \mathbb{P},p}} } - 1; \quad q\ge p^2,$$ 
    we have 
    \begin{align} 
        \left\| w \right\|_{L^p\left( \mathcal{S}\times\mathcal{A} \right)} &\le \frac{\left( \left| \mathcal{A} \right| \mu\left( \mathcal{S} \right) \right)^{\frac{p-1}{p^2}}}{C_{d_{\mu_0}^\pi} - \gamma \left\| d_{\mu_0}^\pi \right\|_{L^{p^2}\left( \mathcal{S}\times\mathcal{A} \right)}   C_{\mathbb{P},p} \left( \left| \mathcal{A} \right| \mu\left( \mathcal{S} \right) \right)^{\frac{p-1}{p^2}}}   \left\| f \right\|_{L^{p^2,d_{\mu_0}^\pi}\left( \mathcal{S}\times\mathcal{A} \right)} \\
        &\le \frac{\left( \left| \mathcal{A} \right| \mu\left( \mathcal{S} \right) \right)^{\frac{1}{p} - \frac{1}{q}}}{C_{d_{\mu_0}^\pi} - \gamma \left\| d_{\mu_0}^\pi \right\|_{L^{p^2}\left( \mathcal{S}\times\mathcal{A} \right)}   C_{\mathbb{P},p} \left( \left| \mathcal{A} \right| \mu\left( \mathcal{S} \right) \right)^{\frac{p-1}{p^2}}}  \left\| f \right\|_{L^{q,d_{\mu_0}^\pi}\left( \mathcal{S}\times\mathcal{A} \right)},
    \end{align}
    where the last inequality comes from $\left\| f \right\|_{L^{p^2,d_{\mu_0}^\pi}\left( \mathcal{S}\times\mathcal{A} \right)} \le  \mu\left( \mathcal{S}\times\mathcal{A} \right)^{\frac{1}{p^2} - \frac{1}{q}}  \left\| f \right\|_{L^{q,d_{\mu_0}^\pi}\left( \mathcal{S}\times\mathcal{A} \right)}$.
\end{proof}
\begin{remark}
    The conditions are not satisfactory. The result implicitly adds the constrain for $\gamma$ because $\lim_{p\rightarrow\infty} C_{d_{\mu_0}^\pi, \mathbb{P},p} = \frac{1}{C_{d_{\mu_0}^\pi}}$, which indicates $\gamma$ may be very small, i.e. $\gamma < C_{d_{\mu_0}^\pi}$. 
\end{remark}
In the following theorem, we describe the instablility of DQN.

\begin{theorem}
    There exists a MDP $\mathcal{M}$ such that for all $\pi$ satisfying $M_{d_{\mu_0}^\pi}:=\sup_{(s,a)\in\mathcal{S}\times\mathcal{A}} d_{\mu_0}^\pi(s,a) <\infty$, Bellman optimality equation $\mathcal{T}_B Q = Q$ is not $\left(  L^{q,d_{\mu_0}^\pi}\left( \mathcal{S}\times\mathcal{A} \right),  L^p \left( \mathcal{S}\times\mathcal{A} \right) \right)$-stable, for all $1 \le q < p\le \infty$.
\end{theorem}
\begin{proof}
    According to the proof of Theorem \ref{thm: unstable of bellman opt eq}, for $1 \le q < p\le \infty$, there exists a MDP $\mathcal{M}$ satisfying the following statement. For all $\delta>0$ and $n\in\mathbb{N}$, there exists a $Q\in L^p \left( \mathcal{S}\times\mathcal{A} \right) \cap L^q \left( \mathcal{S}\times\mathcal{A} \right)$ satisfying $\left\| \mathcal{T}_B Q - Q \right\|_{L^q \left( \mathcal{S}\times\mathcal{A} \right)} \le \frac{\delta}{M_{d_{\mu_0}^\pi}}$ such that $\left\| Q - Q^* \right\|_{L^p \left( \mathcal{S}\times\mathcal{A} \right)} > n\left\| \mathcal{T}_B Q - Q \right\|_{L^q \left( \mathcal{S}\times\mathcal{A} \right)}$. 

    According to Lemma \ref{lem: norm property}, if $M_{d_{\mu_0}^\pi}:=\sup_{(s,a)\in\mathcal{S}\times\mathcal{A}} d_{\mu_0}^\pi(s,a) <\infty$, we have
    \begin{equation}
        \left\| Q \right\|_{L^{q,d_{\mu_0}^\pi}\left( \mathcal{S}\times\mathcal{A} \right)} \le M_{d_{\mu_0}^\pi} \left\| Q \right\|_{L^q \left( \mathcal{S}\times\mathcal{A} \right)} < \infty.
    \end{equation}
    Thus, we have $Q\in L^{q,d_{\mu_0}^\pi}\left( \mathcal{S}\times\mathcal{A} \right) \cap L^p \left( \mathcal{S}\times\mathcal{A} \right)$. For the same reason, we have
    \begin{equation}
        \left\| \mathcal{T}_B Q - Q \right\|_{L^{q,d_{\mu_0}^\pi}\left( \mathcal{S}\times\mathcal{A} \right)} \le M_{d_{\mu_0}^\pi} \left\| \mathcal{T}_B Q - Q \right\|_{L^q \left( \mathcal{S}\times\mathcal{A} \right)} \le \delta.
    \end{equation}
    Hence, we get that For all $\delta>0$ and $n\in\mathbb{N}$, there exists a $Q\in L^{q,d_{\mu_0}^\pi}\left( \mathcal{S}\times\mathcal{A} \right) \cap L^p \left( \mathcal{S}\times\mathcal{A} \right)$ satisfying $\left\| \mathcal{T}_B Q - Q \right\|_{L^{q,d_{\mu_0}^\pi}\left( \mathcal{S}\times\mathcal{A} \right)} \le \delta$ such that $\left\| Q - Q^* \right\|_{L^p \left( \mathcal{S}\times\mathcal{A} \right)} > n \left\| \mathcal{T}_B Q - Q \right\|_{L^q \left( \mathcal{S}\times\mathcal{A} \right)}$.
\end{proof}

\begin{remark}
    If $M_{d_{\mu_0}^\pi} = \infty$, $d_{\mu_0}^\pi$ degenerates to the discrete probability distribution. Then, $L^{p,d_{\mu_0}^\pi}$ can be considered as a norm defined on a finite dimension space. In this setting, we also have $C_{d_{\mu_0}^\pi}=0$ and Bellman optimality equation $\mathcal{T}_B Q = Q$ is not $\left(  L^{q,d_{\mu_0}^\pi}\left( \mathcal{S}\times\mathcal{A} \right),  L^p \left( \mathcal{S}\times\mathcal{A} \right) \right)$-stable, for any $p$ and $q$.
\end{remark}

According to the above theorems and remarks, we have the following corollary in the DQN procedure.
\begin{corollary}
     In practical DQN procedure, the Bellman optimality equations $\mathcal{T}_B Q = Q$ is $(  L^{\infty,d_{\mu_0}^\pi}( \mathcal{S}\times\mathcal{A}),  L^p( \mathcal{S}\times\mathcal{A}))$-stable for all $1\le p \le \infty$, while it is not $(  L^{q,d_{\mu_0}^\pi}( \mathcal{S}\times\mathcal{A}),  L^p( \mathcal{S}\times\mathcal{A}))$-stable for all $1 \le q < p\le \infty$. 
\end{corollary}

\subsection{Stability of DQN: the Bad} \label{app:Stability of DQN: the Bad}
\begin{theorem}
    There exists an MDP $\mathcal{M}$ such that for all $\pi$ satisfying $ d_{\mu_0}^\pi$ is a discrete probability distribution, Bellman optimality equation $\mathcal{T}_B Q = Q$ is not $\left(  L^{q,d_{\mu_0}^\pi}\left( \mathcal{S}\times\mathcal{A} \right),  L^p \left( \mathcal{S}\times\mathcal{A} \right) \right)$-stable, for any $p$ and $q$.
\end{theorem}
\begin{proof}
    We only need to show there exists an MDP such that $\forall n\in\mathbb{N}, \forall \delta>0, \exists Q(s,a),$ such that $\|\mathcal{T}_BQ-Q\|_{L^{q,d_{\mu_0}^\pi}\left( \mathcal{S}\times\mathcal{A} \right)}<\delta, \text{ but }  \|Q-Q^*\|_{L^p\left( \mathcal{S}\times\mathcal{A} \right)} \ge n.$
        
        Consider an MDP $\mathcal{M}$ where $\mathcal{S}=[-1,1], \mathcal{A}=\{a_1,a_2\},$
	\begin{equation*}
			\mathbb{P}(s^\prime |s,a_1)=\left\{ \begin{aligned}
				&\mathbbm{1}_{\left\{s^\prime=s-0.1\right\}}, \quad && s\in[-0.9,1] \\
				&\mathbbm{1}_{\left\{s^\prime=s \right\}}    ,      && else 
			\end{aligned}\right.,\quad		
			\mathbb{P}(s^\prime |s,a_2)=\left\{ \begin{aligned}
				&\mathbbm{1}_{\left\{s^\prime=s+0.1\right\}}, \quad && s\in[-1,0.9] \\
				&\mathbbm{1}_{\left\{s^\prime=s \right\}}    ,      && else 
			\end{aligned}\right.,
	\end{equation*}
	$r(s,a_i)=k_is,\ k_2\ge k_1> 0$. The transition function is essentially a deterministic transition dynamic and for convenience, we denote that
        \begin{equation*}
            p(s,a_1)=\left\{ \begin{aligned}
                &s-0.1, \quad && s\in[-0.9,1] \\
                &s     ,      && else 
            \end{aligned}\right.,\quad		
            p(s,a_2)=\left\{ \begin{aligned}
                &s+0.1, \quad && s\in[-1,0.9] \\
                &s     ,      && else 
            \end{aligned}\right..
       \end{equation*}
	
    Let $Q^*(s,a)=Q^{\pi^*}(s,a)$ be the optimal Q-function, where $\pi^*$ is the optimal policy. 
 
    Define $B_0=\{ s\in\mathcal{S}:\exists a\in\mathcal{A}, s.t.\, d^{\pi}_{\mu_0}(s,a)\neq 0 \}$, which contains all the states that can be explored. Let $B=\left\{B_0\cup\{B_0+0.1\}\cup\{B_0-0.1\}\right\}\cap\mathcal{S}$, then $\forall s\in B,\ p(s,a)\in B.$ Since $d_{\mu_0}^\pi$ is a discrete probability distribution, $\mu (B)=\mu (B_0)=0.$
    
    Let $D=[-1,1]\setminus B$, and $Q(s,a)=Q^*(s,a)+h\cdot \mathbbm{1}_D$, where $h=\frac{n}{2^{\frac{1}{p}}}$. We have that $Q(s,a)=Q^*(s,a),\forall s\notin D.$
    We then have for any $s\in B$ that
    \begin{align*}
        \mathcal{T}_B Q(s,a) & = r(s,a)+\gamma\cdot\max_{a^{\prime}\in\mathcal{A}}Q(p(s,a),a^{\prime})\\
        & = r(s,a)+\gamma\cdot\max_{a^{\prime}\in\mathcal{A}}Q^*(p(s,a),a^{\prime})\\
        & = Q^*(s,a)\\
        & = Q(s,a).
    \end{align*}
    For any $s\notin B$, $d_{\mu_0}^\pi(s,a)=0$ for all $a\in\mathcal{A}$. Hence, $\|\mathcal{T}_B Q(s,a)-Q(s,a)\|_{L^{q,d^{\pi}_{\mu_0}}(\mathcal{S}\times\mathcal{A})}=0 < \delta.$

    However, we find that $$\|Q(s,a)-Q^*(s,a)\|_{L^p(\mathcal{S}\times\mathcal{A})} = h\cdot \mu(D)^{\frac{1}{p}}=h\cdot 2^{\frac{1}{p}} \ge n,$$
    which completes the proof.
\end{proof}

\begin{remark}
    If $d_{\mu_0}^\pi$ is a discrete probability distribution, $L^{\infty,d_{\mu_0}^\pi}\left( \mathcal{S}\times\mathcal{A} \right)$ is not a good choice. However, the sample process should be considered in practical reinforcement learning algorithms and as a consequence, we have to apply the space $L^{\infty,d_{\mu_0}^{\pi}}\left( \mathcal{S}\times\mathcal{A} \right)$ rather than $L^{\infty}\left( \mathcal{S}\times\mathcal{A} \right)$.
\end{remark}

\section{Derivation of CAR-DQN}\label{app:derive car-dqn}

Our theory motivates us to use the following objective:
\begin{align}
    \mathcal{L}_{car}(\theta) &:= \left\|\mathcal{T}_B Q_\theta - Q_\theta \right\|_{L^{\infty,d_{\mu_0}^{\pi_\theta}}\left( \mathcal{S}\times\mathcal{A} \right)} \\
    &=\sup_{(s,a)\in\mathcal{S}\times\mathcal{A}} d_{\mu_0}^{\pi_\theta}(s,a) \left| \mathcal{T}_B Q_\theta(s,a) - Q_\theta(s,a) \right| \\
    &=\sup_{(s,a)\in\mathcal{S}\times\mathcal{A}} d_{\mu_0}^{\pi_\theta}(s,a) \left| r(s, a)+\gamma \mathbb{E}_{s^{\prime} \sim \mathbb{P}(\cdot \mid s, a)}\left[\max _{a^{\prime} \in \mathcal{A}} Q_\theta\left(s^{\prime}, a^{\prime}\right)\right] - Q_\theta(s,a) \right| .
\end{align}
However, the objective is intractable in a model-free setting, due to the unknown environment, i.e. unknown reward function and unknown transition function.
\begin{remark}
    We apply the space $L^{\infty,d_{\mu_0}^{\pi_\theta}}\left( \mathcal{S}\times\mathcal{A} \right)$ rather than $L^{\infty}\left( \mathcal{S}\times\mathcal{A} \right)$ because the sampling process should be considered in practical reinforcement learning algorithms.
\end{remark}

\subsection{Surrogate Objective}
We can derive that
\begin{align}
    \mathcal{L}_{car}(\theta) &= \sup_{s\in\mathcal{S}} \max_{a\in\mathcal{A}} d_{\mu_0}^{\pi_\theta}(s,a) \left| \mathcal{T}_{B} Q_\theta(s,a) - Q_\theta(s,a) \right|  \\
    &= \sup_{s\in\mathcal{S}} \max_{s_\nu \in B_\epsilon(s)} \max_{a\in\mathcal{A}} d_{\mu_0}^{\pi_\theta}(s,a) \left| \mathcal{T}_{B} Q_\theta(s_\nu,a) - Q_\theta(s_\nu,a) \right| \\
    &= \sup_{(s,a)\in\mathcal{S}\times\mathcal{A}} d_{\mu_0}^{\pi_\theta}(s,a) \max_{s_\nu \in B_\epsilon(s)}  \left| \mathcal{T}_{B} Q_\theta(s_\nu,a) - Q_\theta(s_\nu,a) \right|. 
\end{align}
However, in a practical reinforcement learning setting, we cannot directly get the estimation of $\mathcal{T}_{B} Q_\theta(s_\nu, a)$.

\begin{theorem}
    Let $\mathcal{L}_{car}^{train}(\theta) := \sup_{(s,a)\in\mathcal{S}\times\mathcal{A}} d_{\mu_0}^{\pi_\theta}(s,a) \max_{s_\nu \in B_\epsilon(s)} \left| \mathcal{T}_{B} Q_\theta(s,a) - Q_\theta(s_\nu,a) \right| $ and $\mathcal{L}_{car}^{diff}(\theta) := \sup_{(s,a)\in\mathcal{S}\times\mathcal{A}} d_{\mu_0}^{\pi_\theta}(s,a) \max_{s_\nu \in B_\epsilon(s)} \left| \mathcal{T}_{B} Q_\theta(s_\nu,a) - \mathcal{T}_{B}Q_\theta(s,a) \right|$. We have that 
    \begin{equation}
        \left| \mathcal{L}_{car}^{train}(\theta) -  \mathcal{L}_{car}^{diff}(\theta)\right| \le \mathcal{L}_{car}(\theta) \le \mathcal{L}_{car}^{train}(\theta) + \mathcal{L}_{car}^{diff}(\theta).
    \end{equation}
\end{theorem}
\begin{proof}
    On one hand, we have
    \begin{align}
        \mathcal{L}_{car}(\theta) 
        &= \sup_{(s,a)\in\mathcal{S}\times\mathcal{A}} d_{\mu_0}^{\pi_\theta}(s,a) \max_{s_\nu \in B_\epsilon(s)}  \left| \mathcal{T}_{B} Q_\theta(s_\nu,a) - Q_\theta(s_\nu,a) \right| \\
        &\le \sup_{(s,a)\in\mathcal{S}\times\mathcal{A}} d_{\mu_0}^{\pi_\theta}(s,a) \max_{s_\nu \in B_\epsilon(s)}  \left( \left| \mathcal{T}_{B} Q_\theta(s,a) - Q_\theta(s_\nu,a) \right| + \left| \mathcal{T}_{B} Q_\theta(s_\nu,a) - \mathcal{T}_{B}Q_\theta(s,a) \right| \right) \\
        &\le  \sup_{(s,a)\in\mathcal{S}\times\mathcal{A}} d_{\mu_0}^{\pi_\theta}(s,a) \max_{s_\nu \in B_\epsilon(s)} \left| \mathcal{T}_{B} Q_\theta(s_\nu,a) - \mathcal{T}_{B}Q_\theta(s,a) \right|  \\
        &\quad + \sup_{(s,a)\in\mathcal{S}\times\mathcal{A}} d_{\mu_0}^{\pi_\theta}(s,a) \max_{s_\nu \in B_\epsilon(s)} \left| \mathcal{T}_{B} Q_\theta(s,a) - Q_\theta(s_\nu,a) \right|. 
    \end{align}
    On the other hand, we have
    \begin{align}
        \mathcal{L}_{car}(\theta) &= \sup_{(s,a)\in\mathcal{S}\times\mathcal{A}} d_{\mu_0}^{\pi_\theta}(s,a) \max_{s_\nu \in B_\epsilon(s)}  \left| \mathcal{T}_{B} Q_\theta(s_\nu,a) - Q_\theta(s_\nu,a) \right| \\
        &\ge \sup_{(s,a)\in\mathcal{S}\times\mathcal{A}} d_{\mu_0}^{\pi_\theta}(s,a) \max_{s_\nu \in B_\epsilon(s)}  \left| \left| \mathcal{T}_{B} Q_\theta(s_\nu,a) - \mathcal{T}_{B}Q_\theta(s,a) \right| - \left| \mathcal{T}_{B} Q_\theta(s,a) - Q_\theta(s_\nu,a) \right| \right| \\
        &\ge \left| \sup_{(s,a)\in\mathcal{S}\times\mathcal{A}} d_{\mu_0}^{\pi_\theta}(s,a) \max_{s_\nu \in B_\epsilon(s)} \left| \mathcal{T}_{B} Q_\theta(s_\nu,a) - \mathcal{T}_{B}Q_\theta(s,a) \right|  \right.\\
        &\quad \left. - \sup_{(s,a)\in\mathcal{S}\times\mathcal{A}} d_{\mu_0}^{\pi_\theta}(s,a) \max_{s_\nu \in B_\epsilon(s)} \left| \mathcal{T}_{B} Q_\theta(s,a) - Q_\theta(s_\nu,a) \right| \right|,
    \end{align}
    where the second inequality comes from Lemma \ref{lem: basic ineq 2}.
\end{proof}

It is hard to calculate or estimate $\mathcal{L}_{car}^{diff}(\theta)$ in practice. Fortunately, we think $\mathcal{L}_{car}^{diff}(\theta)$ should be small in practice and we give a constant upper bound of $\mathcal{L}_{car}^{diff}(\theta)$ in the smooth environment.

\begin{lemma} \label{lem: uniform bound of bellman operator}
    Suppose $Q$ and $r$ are uniformly bounded, i.e. $\exists\ M_Q,M_r >0$ such that $\left|Q(s,a)\right| \le M_Q,\ \left|r(s,a)\right| \le M_r\ \forall s\in\mathcal{S}, a\in\mathcal{A}$. Then $\mathcal{T}_{B} Q(\cdot,a)$ is uniformly bounded, i.e.
    \begin{equation}
        \left|\mathcal{T}_{B} Q(s,a) \right| \le C_Q,\ \forall s\in\mathcal{S}, a\in\mathcal{A},
    \end{equation}
    where $C_Q = \max\left\{ M_Q, \frac{M_r}{1-\gamma} \right\}$. Further, for any $k\in\mathbb{N}$, $\mathcal{T}_{B}^{k} Q(\cdot,a)$ has the same uniform bound as $\mathcal{T}_{B} Q(\cdot,a)$, i.e.
    \begin{equation}\label{eq: uniform bound of bellman operator}
        \left|\mathcal{T}_{B}^{k} Q(s,a) \right| \le C_Q,\ \forall s\in\mathcal{S}, a\in\mathcal{A}.
    \end{equation}
\end{lemma}
\begin{proof}
    \begin{align}
        \left|\mathcal{T}_{B} Q(s,a) \right| &= \left|r(s,a)  + \gamma \mathbb{E}_{ s^\prime \sim \mathbb{P}(\cdot|s,a)} \left[ \max _{a^{\prime} \in \mathcal{A}} Q\left(s^{\prime}, a^{\prime}\right) \right]\right| \\
        &\le \left|r(s,a) \right| + \gamma\mathbb{E}_{ s^\prime \sim \mathbb{P}(\cdot|s,a)} \left|\max _{a^{\prime} \in \mathcal{A}} Q\left(s^{\prime}, a^{\prime}\right)\right| \\
        &\le M_r + \gamma M_Q \\
        &\le \max\left\{ M_Q, \frac{M_r}{1-\gamma} \right\}, \qquad \forall s\in\mathcal{S}, a\in\mathcal{A}.
    \end{align}

    Let $C_Q = \max\left\{ M_Q, \frac{M_r}{1-\gamma} \right\}$. Suppose the inequality (\ref{eq: uniform bound of bellman operator}) holds for $k=n$. Then, for $k=n+1$, we have
    \begin{align}
        \left|\mathcal{T}_{B}^{n+1} Q(s,a) \right| &= \left|r(s,a)  + \gamma \mathbb{E}_{ s^\prime \sim \mathbb{P}(\cdot|s,a)} \left[  \max _{a^{\prime} \in \mathcal{A}} \mathcal{T}_{B}^{n} Q\left(s^{\prime}, a^{\prime}\right)\right]\right| \\
        &\le \left|r(s,a) \right| + \gamma\mathbb{E}_{ s^\prime \sim \mathbb{P}(\cdot|s,a)}  \left| \max _{a^{\prime} \in \mathcal{A}} \mathcal{T}_{B}^{n} Q\left(s^{\prime}, a^{\prime}\right)\right| \\
        &\le M_r + \gamma C_Q \\
        &\le (1-\gamma) C_Q + \gamma C_Q \\
        &= C_Q.
    \end{align}
    By induction, we have $\left|\mathcal{T}_{B}^{k} Q(s,a) \right| \le C_Q,\ \forall s\in\mathcal{S}, a\in\mathcal{A}, k\in\mathbb{N}$.
\end{proof}

\begin{lemma} \label{lem: lip of bellman operator}
    Suppose the environment is $\left(L_r, L_{\mathbb{P}}\right)$-smooth and suppose $Q$ and $r$ are uniformly bounded, i.e. $\exists\ M_Q,M_r >0$ such that $\left|Q(s,a)\right| \le M_Q,\ \left|r(s,a)\right| \le M_r\ \forall s\in\mathcal{S}, a\in\mathcal{A}$. Then $\mathcal{T}_{B} Q(\cdot,a)$ is Lipschitz continuous, i.e.
        \begin{equation}
            \left| \mathcal{T}_{B} Q(s,a) - \mathcal{T}_{B} Q(s^\prime,a) \right| \le L_{\mathcal{T}_{B}} \|s - s^\prime\|,
        \end{equation}
        where $L_{\mathcal{T}_{B}} =  L_r + \gamma C_Q L_{\mathbb{P}}$ and $C_Q = \max\left\{ M_Q, \frac{M_r}{1-\gamma} \right\}$. Further, for any $k\in\mathbb{N}$, $\mathcal{T}_{B}^{k} Q(\cdot,a)$ is Lipschitz continuous and has the same Lipschitz constant as $\mathcal{T}_{B} Q(\cdot,a)$, i.e.
        \begin{equation}
            \left| \mathcal{T}_{B}^{k} Q(s,a) - \mathcal{T}_{B}^{k} Q(s^\prime,a) \right| \le L_{\mathcal{T}_{B}} \|s - s^\prime\|.
        \end{equation}
\end{lemma}
\begin{proof}
    For all $s_1,s_2 \in \mathcal{S}$, we have
    \begin{align}
        &\quad \mathcal{T}_{B} Q(s_1,a) - \mathcal{T}_{B} Q(s_2,a) \\
        &= r(s_1,a)  + \gamma \mathbb{E}_{ s^\prime \sim \mathbb{P}(\cdot|s_1,a)} \left[ \max _{a^{\prime} \in \mathcal{A}} Q\left(s^{\prime}, a^{\prime}\right) \right] - r(s_2,a) - \gamma \mathbb{E}_{ s^\prime \sim \mathbb{P}(\cdot|s_2,a)} \left[ \max _{a^{\prime} \in \mathcal{A}} Q\left(s^{\prime}, a^{\prime}\right) \right] \\
        &= \left(r(s_1,a) - r(s_2,a)\right) + \gamma \int_{s^\prime} \left( \mathbb{P}(s^\prime|s_1,a) - \mathbb{P}(s^\prime|s_2,a)\right) \max _{a^{\prime} \in \mathcal{A}} Q\left(s^{\prime}, a^{\prime}\right) ds^\prime.
    \end{align}

    Then, we have 
    \begin{align}
        &\quad \left| \mathcal{T}_{B} Q(s_1,a) - \mathcal{T}_{B} Q(s_2,a)\right| \\
        &\le \left| \left(r(s_1,a) - r(s_2,a)\right) \right| + \left| \gamma \int_{s^\prime} \left( \mathbb{P}(s^\prime|s_1,a) - \mathbb{P}(s^\prime|s_2,a)\right) \max _{a^{\prime} \in \mathcal{A}} Q\left(s^{\prime}, a^{\prime}\right) ds^\prime \right| \\    
        &\le L_r \|s_1 - s_2\| + \gamma \int_{s^\prime} \left| \mathbb{P}(s^\prime|s_1,a) - \mathbb{P}(s^\prime|s_2,a)\right| \left| \max _{a^{\prime} \in \mathcal{A}} Q\left(s^{\prime}, a^{\prime}\right) \right| ds^\prime\\
        &\le  L_r \|s_1 - s_2\| + \gamma C_Q \int_{s^\prime} \left| \mathbb{P}(s^\prime|s_1,a) - \mathbb{P}(s^\prime|s_2,a)\right| ds^\prime\\
        &\le  L_r \|s_1 - s_2\| + \gamma C_Q L_{\mathbb{P}} \|s_1 - s_2\| \\
        &= \left( L_r + \gamma C_Q L_{\mathbb{P}} \right) \|s_1 - s_2\|.
    \end{align}
    The second inequality comes from the Lipschitz property of $r$. The third inequality comes from the uniform boundedness of $Q$ and the last inequality utilizes the Lipschitz property of $\mathbb{P}$.

    Note that the uniform boundedness used in the above proof is $C_Q$ rather than $M_Q$. Then, due to lemma \ref{lem: uniform bound of bellman operator}, we can extend the above proof to $\mathcal{T}_{B}^{k}$.
\end{proof}

\begin{theorem}
    Suppose the environment is $\left(L_r, L_{\mathbb{P}}\right)$-smooth and suppose $Q_\theta$ and $r$ are uniformly bounded, i.e. $\exists\ M_Q,M_r >0$ such that $\left|Q_\theta(s,a)\right| \le M_Q,\ \left|r(s,a)\right| \le M_r\ \forall s\in\mathcal{S}, a\in\mathcal{A}$. If $M:=\sup_{\theta,(s,a)\in\mathcal{S}\times\mathcal{A}} d_{\mu_0}^{\pi_\theta}(s,a) <\infty$, then we have
    \begin{equation}
        \mathcal{L}_{car}^{diff}(\theta) \le C_{\mathcal{T}_{B}} \epsilon,
    \end{equation}
    where $C_{\mathcal{T}_{B}}=L_{\mathcal{T}_{B}} M$, $L_{\mathcal{T}_{B}} =  L_r + \gamma C_Q L_{\mathbb{P}}$ and $C_Q = \max\left\{ M_Q, \frac{M_r}{1-\gamma} \right\}$. 
\end{theorem}
\begin{proof}
    \begin{align}
        \mathcal{L}_{car}^{diff}(\theta) &= \sup_{(s,a)\in\mathcal{S}\times\mathcal{A}} d_{\mu_0}^{\pi_\theta}(s,a) \max_{s_\nu \in B_\epsilon(s)} \left| \mathcal{T}_{B} Q_\theta(s_\nu,a) - \mathcal{T}_{B}Q_\theta(s,a) \right| \\
        & \le \sup_{(s,a)\in\mathcal{S}\times\mathcal{A}} d_{\mu_0}^{\pi_\theta}(s,a) \max_{s_\nu \in B_\epsilon(s)} \left( L_r + \gamma C_Q L_{\mathbb{P}} \right) \|s_\nu - s\| \\
        & \le \left( L_r + \gamma C_Q L_{\mathbb{P}} \right) \epsilon \sup_{(s,a)\in\mathcal{S}\times\mathcal{A}} d_{\mu_0}^{\pi_\theta}(s,a) \\
        & \le M \left( L_r + \gamma C_Q L_{\mathbb{P}} \right) \epsilon,
    \end{align}
    where the first inequality comes from Lemma \ref{lem: lip of bellman operator} and the last inequality comes from the uniform boundedness of $d_{\mu_0}^{\pi_\theta}$.
\end{proof}

\subsection{Approximate Objective}

\begin{lemma} \label{lem: soft}
    For any function $f: \Omega \rightarrow \mathbb{R}$ and $\lambda>0$, we have
    \begin{equation}
        \max_{p\in \Delta(\Omega)} \left[ \mathbb{E}_{\omega\sim p} f(\omega) - \lambda \operatorname{KL}(p\| p_0) \right]=\lambda \ln{\mathbb{E}_{\omega\sim p_0} \left[ e^{\frac{f(\omega)}{\lambda}} \right]},
    \end{equation}
    where $ \Delta(\Omega)$ denotes the set of probability distributions on $\Omega$. And the solution is achieved by the following distribution $q$:
    \begin{equation}
         q(\omega) = \frac{ p_0(\omega)e^{\frac{f(\omega)}{\lambda}}}{\int_{\omega\in\Omega} p_0(\omega) e^{\frac{f(\omega)}{\lambda}} d\mu(\omega)} = \frac{1}{C} p_0(\omega)e^{\frac{f(\omega)}{\lambda}}.
    \end{equation}
\end{lemma}
\begin{proof}
    Let $$C:=  \mathbb{E}_{\omega\sim p_0} \left[ e^{\frac{f(\omega)}{\lambda}} \right] = \int_{\omega\in\Omega} p_0(\omega) e^{\frac{f(\omega)}{\lambda}} d\mu(\omega), $$
    $$ q(\omega) = \frac{ p_0(\omega)e^{\frac{f(\omega)}{\lambda}}}{\int_{\omega\in\Omega} p_0(\omega) e^{\frac{f(\omega)}{\lambda}} d\mu(\omega)} = \frac{1}{C} p_0(\omega)e^{\frac{f(\omega)}{\lambda}} .$$
    Then, we have
    \begin{align}
        &\mathbb{E}_{\omega\sim p} f(\omega) - \lambda \operatorname{KL}(p\| p_0) \\
        =&\mathbb{E}_{\omega\sim p} \left[ \lambda \ln{e^{\frac{f(\omega)}{\lambda}}} - \lambda \ln{\frac{p(\omega)}{p_0(\omega)}} \right]\\
        =&\lambda \mathbb{E}_{\omega\sim p} \left[  \ln{ \frac{e^{\frac{f(\omega)}{\lambda}}p_0(\omega)}{p(\omega)}} \right]\\
        =& \lambda \mathbb{E}_{\omega\sim p} \left[  \ln{ \frac{C q(\omega)}{p(\omega)}} \right]\\
        =& \lambda \left[ \ln{C} - \operatorname{KL}(p \| q) \right]\\
        \le & \lambda \ln{C}\\
        =& \lambda \ln{\mathbb{E}_{\omega\sim p_0} \left[ e^{\frac{f(\omega)}{\lambda}} \right]}.
    \end{align}
    Note that the equal sign holds if and only if $p=q$. Thus, we get
    $$ q \in \arg\max_{p\in \Delta(\Omega)} \left[ \mathbb{E}_{\omega\sim p} f(\omega) - \lambda \operatorname{KL}(p\| p_0) \right] .  $$ 
\end{proof}

We get the following approximate objective of $\mathcal{L}_{car}^{train}(\theta)$:
\begin{equation}
    \mathcal{L}_{car}^{app}(\theta)=\max_{(s,a,r,s^\prime)\in \mathcal{B}} \frac{1}{\left| \mathcal{B} \right|} \max_{s_\nu \in B_\epsilon(s)} \left|r + \gamma \max_{a^\prime} Q_{\bar{\theta}}(s^\prime,a^\prime) - Q_{\theta}(s_\nu,a) \right|.
\end{equation}

Denote 
\begin{equation}
    f_i = f(s_i,a_i,r_i,s_i^\prime):= \max_{s_\nu \in B_\epsilon(s_i)} \left|r_i + \gamma \max_{a^\prime} Q_{\bar{\theta}}(s_i^\prime,a^\prime) - Q_{\theta}(s_\nu,a_i)   \right|.    
\end{equation}
To fully utilize each sample in the batch, we derive the soft version of the above objective:
\begin{align}
    \mathcal{L}_{car}^{train}(\theta)&=\max_{(s,a,r,s^\prime)\in \mathcal{B}} \frac{1}{\left| \mathcal{B} \right|} f(s,a,r,s^\prime) \\
    &= \frac{1}{\left| \mathcal{B} \right|} \max_{p\in\Delta\left(\mathcal{B}\right)} \sum_{(s_i,a_i,r_i,s_i^\prime)\in \mathcal{B}} p_if(s_i,a_i,r_i,s_i^\prime) \label{obj: max} \\
    &\ge \frac{1}{\left| \mathcal{B} \right|} \max_{p\in\Delta\left(\mathcal{B}\right)} \left( \sum_{(s_i,a_i,r_i,s_i^\prime)\in \mathcal{B}} p_i f(s_i,a_i,r_i,s_i^\prime) - \lambda \operatorname{KL}\left( p \| U\left(\mathcal{B}\right) \right) \right), \label{obj: soft max}
\end{align}
where $U\left(\mathcal{B}\right)$ represents the uniform distribution over $\mathcal{B}$. According to Lemma \ref{lem: soft}, the optimal solution of the maximization problem (\ref{obj: soft max}) is $p^*$:
\begin{equation}
    p_i^* = \frac{e^{\frac{1}{\lambda} f_i}}{\sum_{i\in \mathcal{\left|B\right|}} e^{\frac{1}{\lambda} f_i}}.
\end{equation}
The maximization problem (\ref{obj: soft max}) is the lower bound of the maximization problem (\ref{obj: max}) so $p^*$ is a proper approximation of the optimal solution of the maximization problem (\ref{obj: max}). Thus, we get the soft version of the CAR objective:
\begin{equation}
    \mathcal{L}_{car}^{soft}(\theta) = \frac{1}{\left| \mathcal{B} \right|} \sum_{(s_i,a_i,r_i,s_i^\prime)\in \mathcal{B}} \frac{e^{\frac{1}{\lambda} f_i}}{\sum_{i\in \mathcal{\left|B\right|}} e^{\frac{1}{\lambda} f_i}} \max_{s_\nu \in B_\epsilon(s_i)} \left|r_i + \gamma \max_{a^\prime} Q_{\bar{\theta}}(s_i^\prime,a^\prime) - Q_{\theta}(s_\nu,a_i)  \right|.
\end{equation}

\section{Examples of Intrinsic States}\label{app: instrinsic state}
In Figure \ref{app fig:intrinsic states pong}, \ref{app fig:intrinsic states free}, \ref{app fig:intrinsic states road}, we show some examples in 3 Atari games (Pong, Freeway, and RoadRunner) that the state observation $s$ and adversarial observation $s_\nu$ have the same intrinsic state.

\begin{figure}[htb]
    \centering
\includegraphics[width=0.7\textwidth]{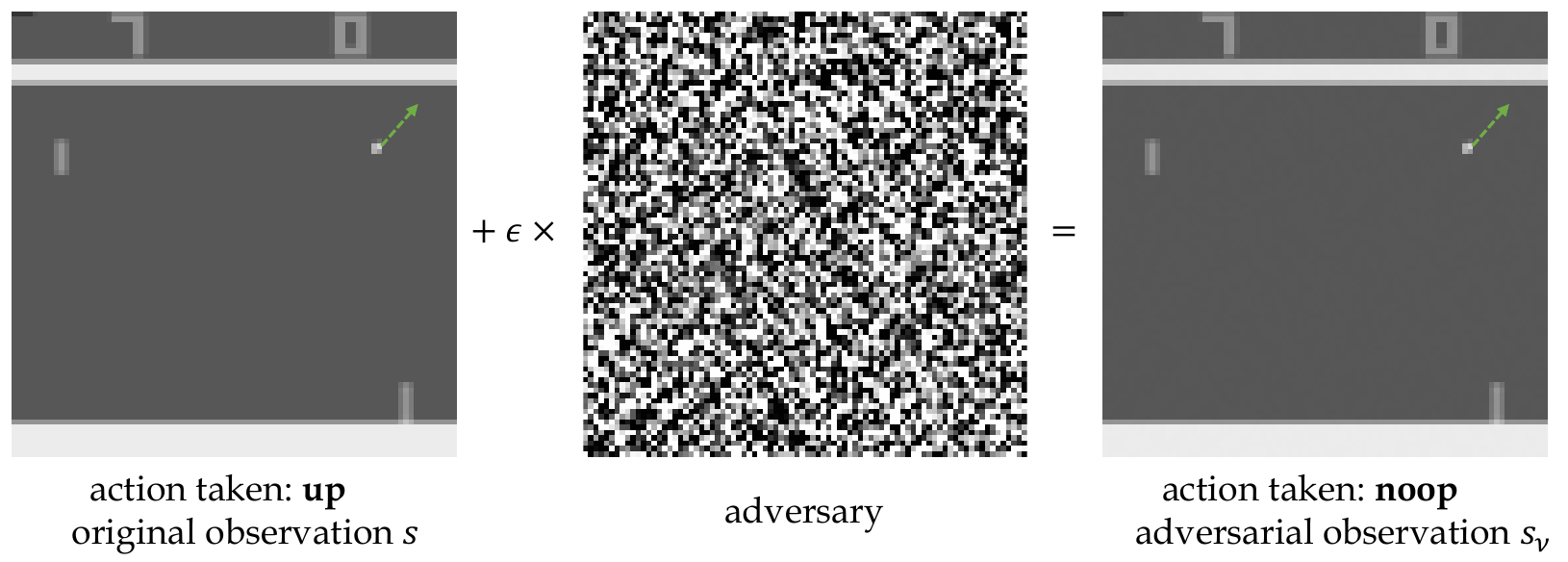} \includegraphics[width=0.7\textwidth]{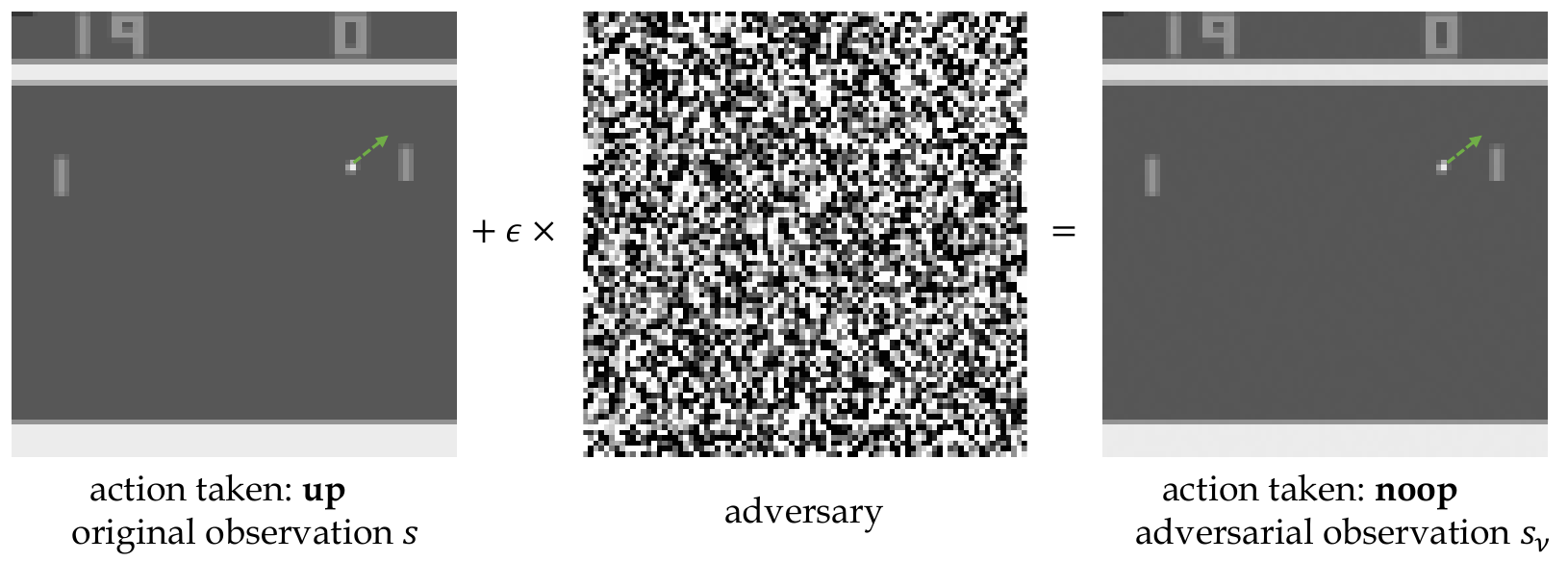}
\includegraphics[width=0.7\textwidth]{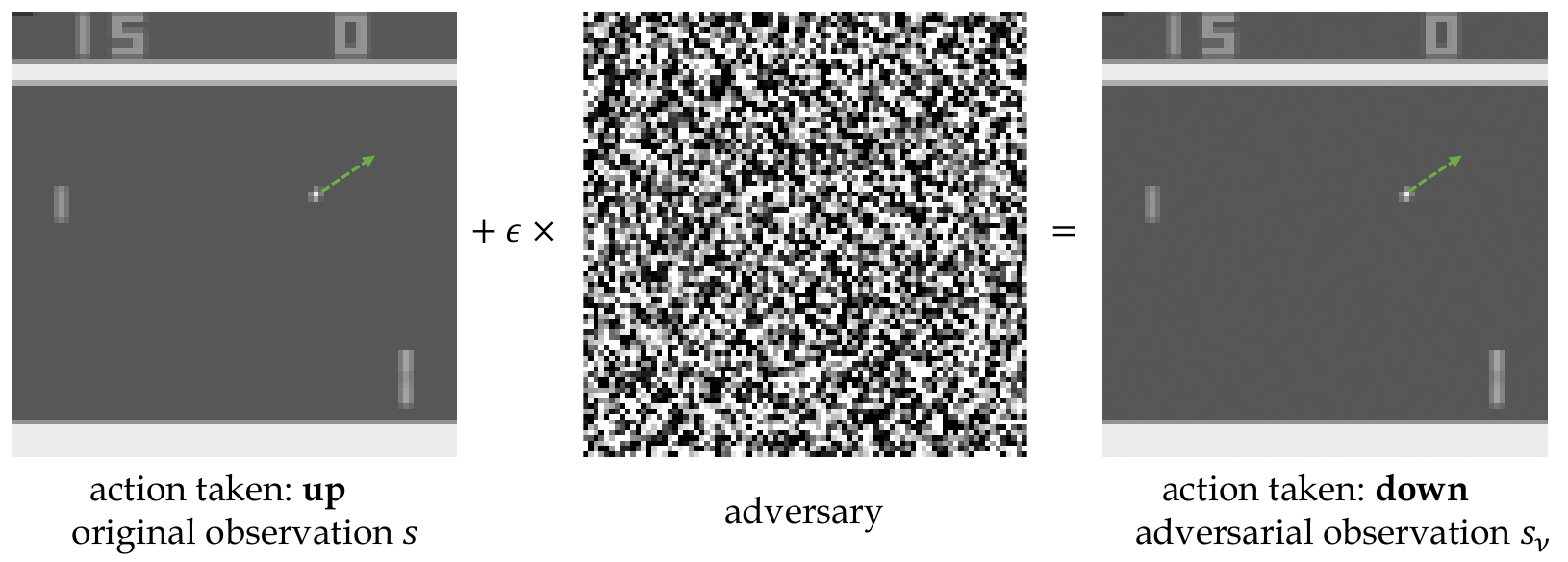}
\includegraphics[width=0.7\textwidth]{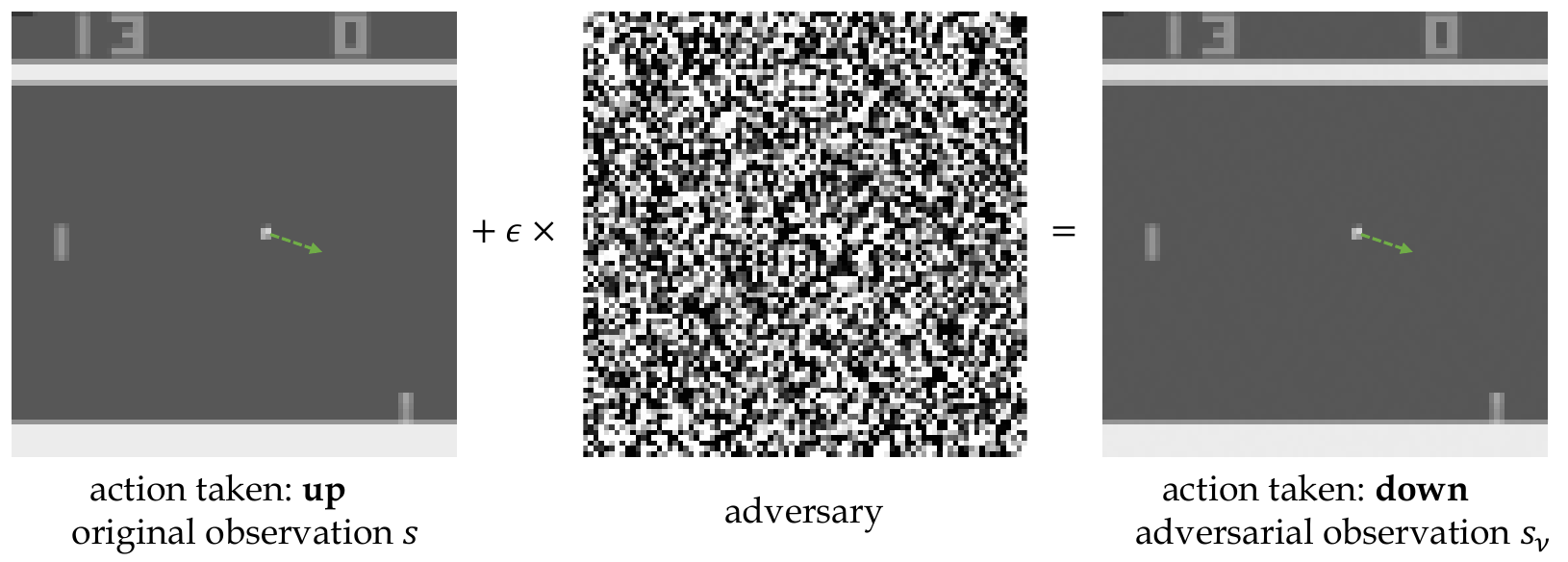}
\includegraphics[width=0.7\textwidth]{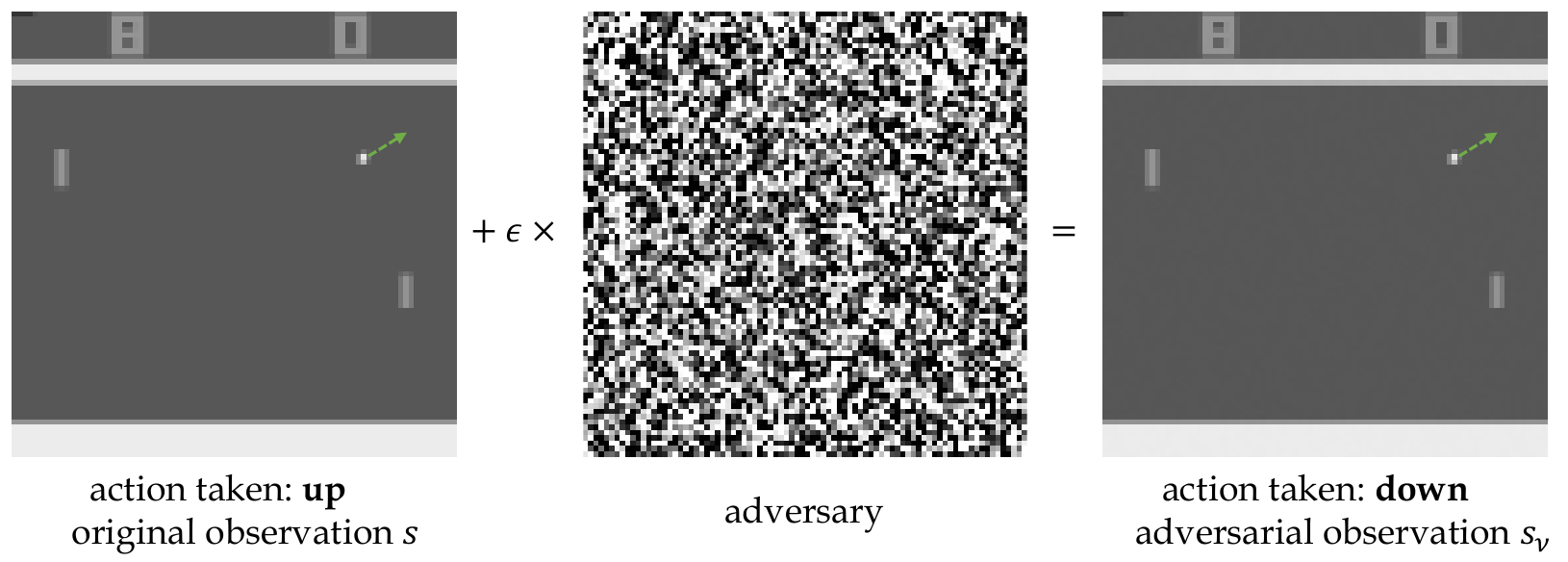}
    \caption{Examples of intrinsic states in Pong games. The direction of movement of the ball is marked.}
    \label{app fig:intrinsic states pong}
\end{figure}

\begin{figure}[htb]
    \centering
\includegraphics[width=0.7\textwidth]{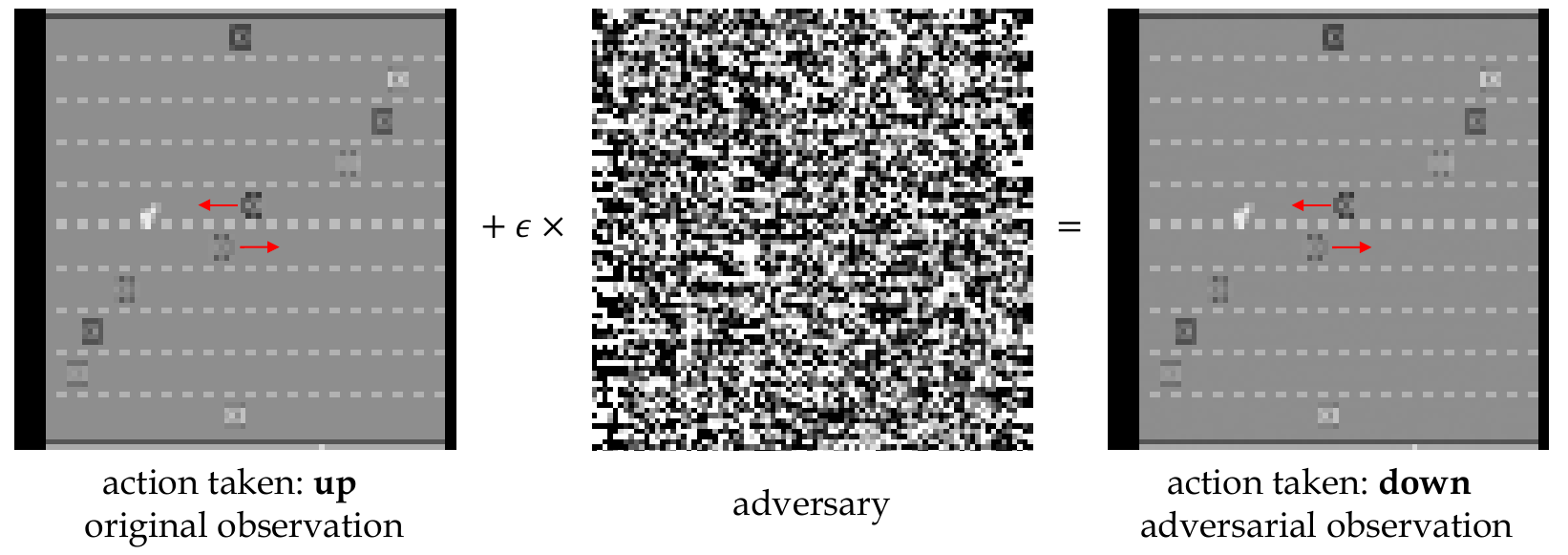}
\includegraphics[width=0.7\textwidth]{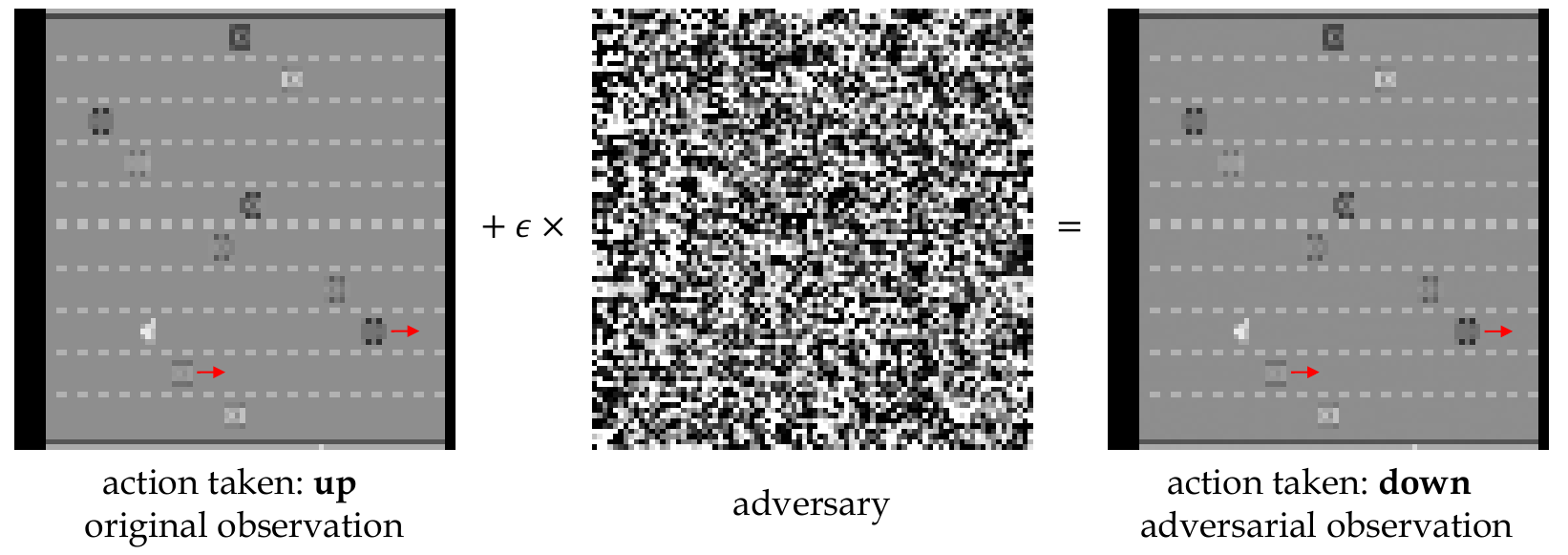}
\includegraphics[width=0.7\textwidth]{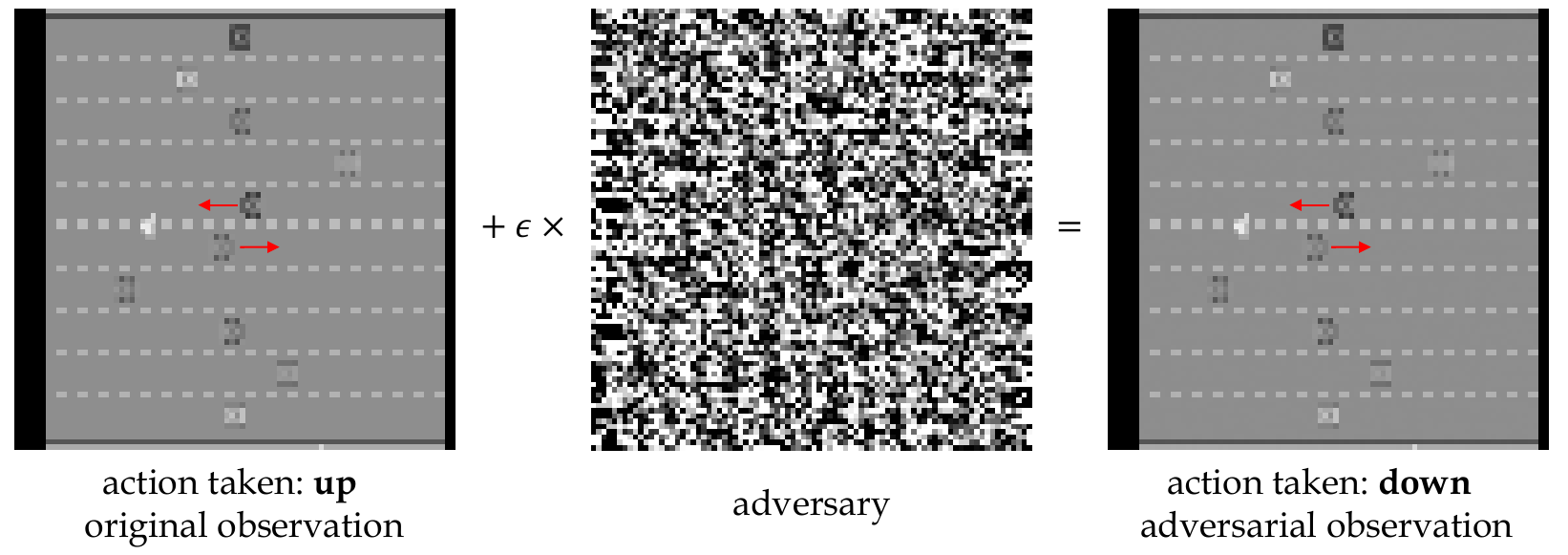}
\includegraphics[width=0.7\textwidth]{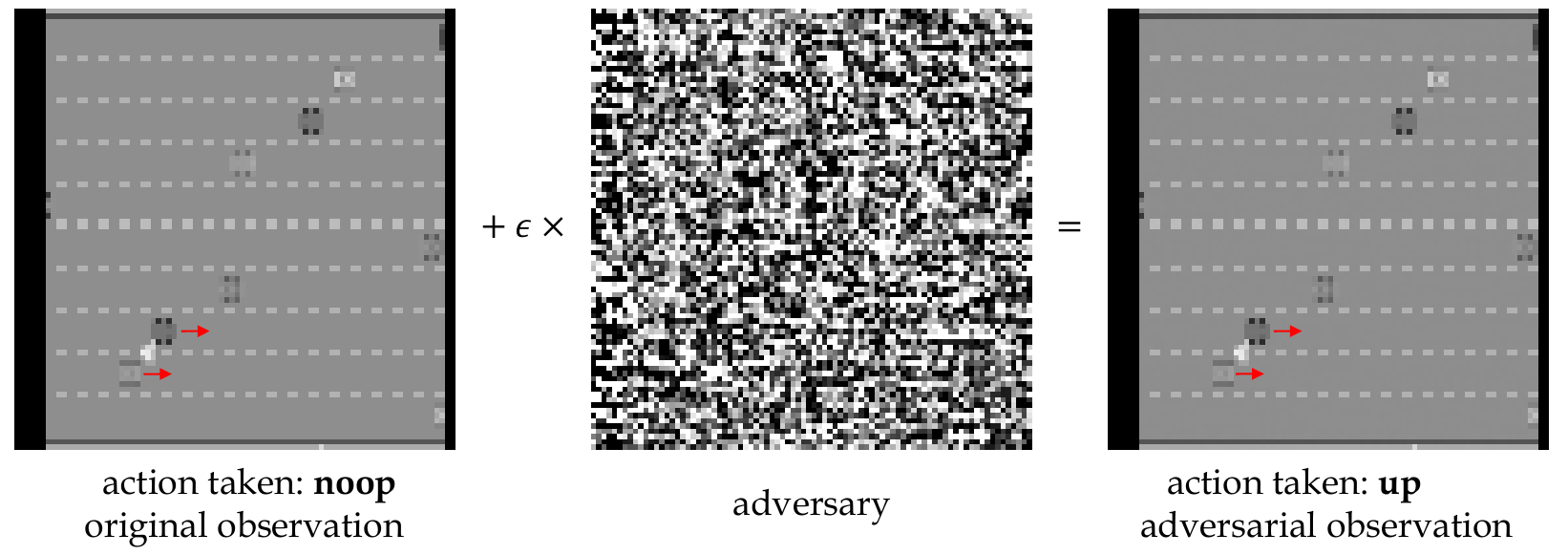}
\includegraphics[width=0.7\textwidth]{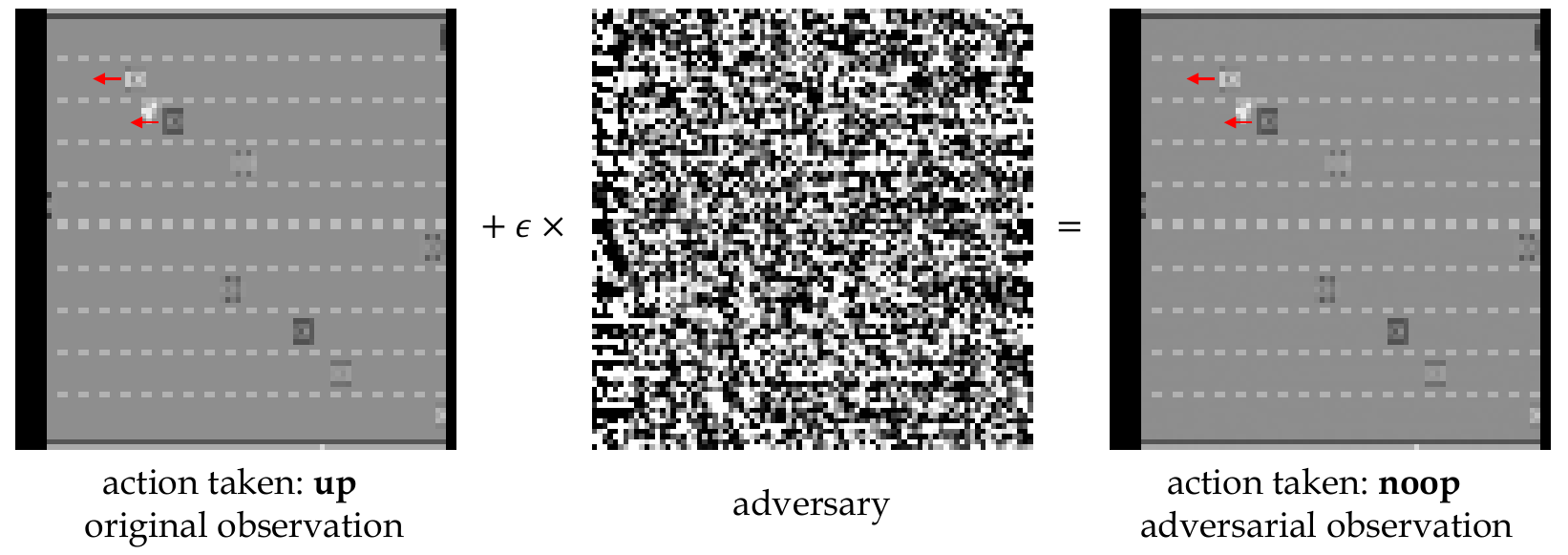}
    \caption{Examples of intrinsic states in Freeway games. The directions of movement of cars around the chicken are marked.}
    \label{app fig:intrinsic states free}
\end{figure}

\begin{figure}[htb]
    \centering
\includegraphics[width=0.7\textwidth]{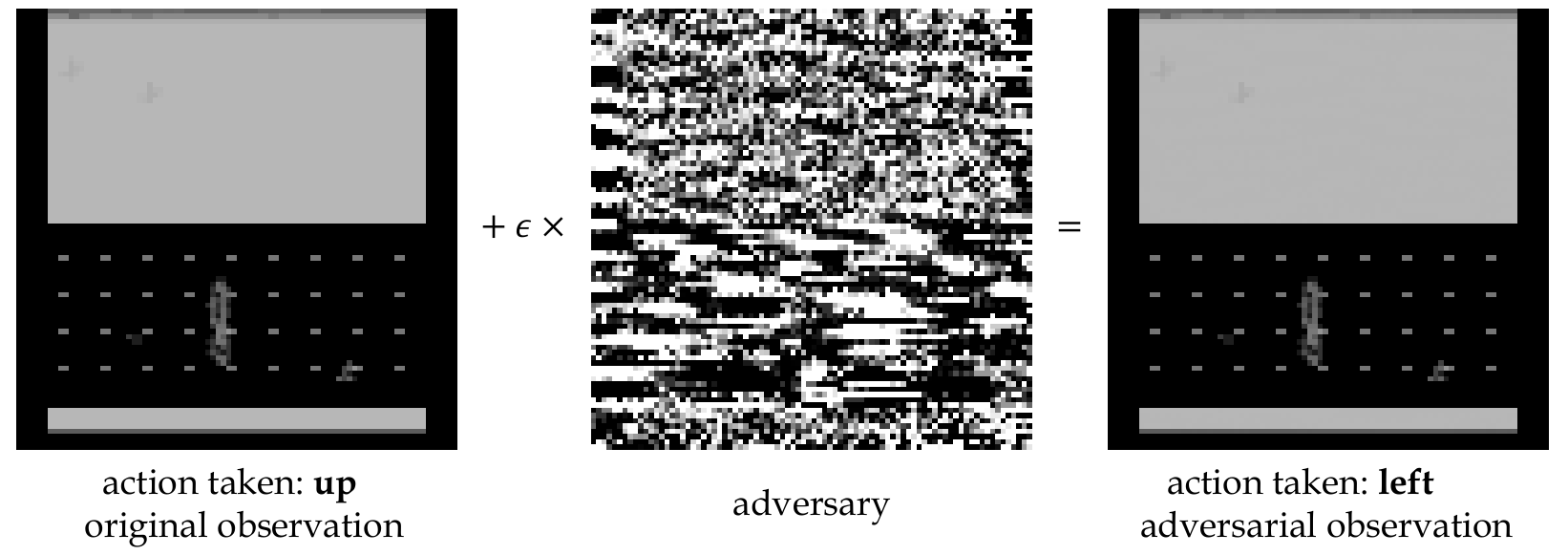}
\includegraphics[width=0.7\textwidth]{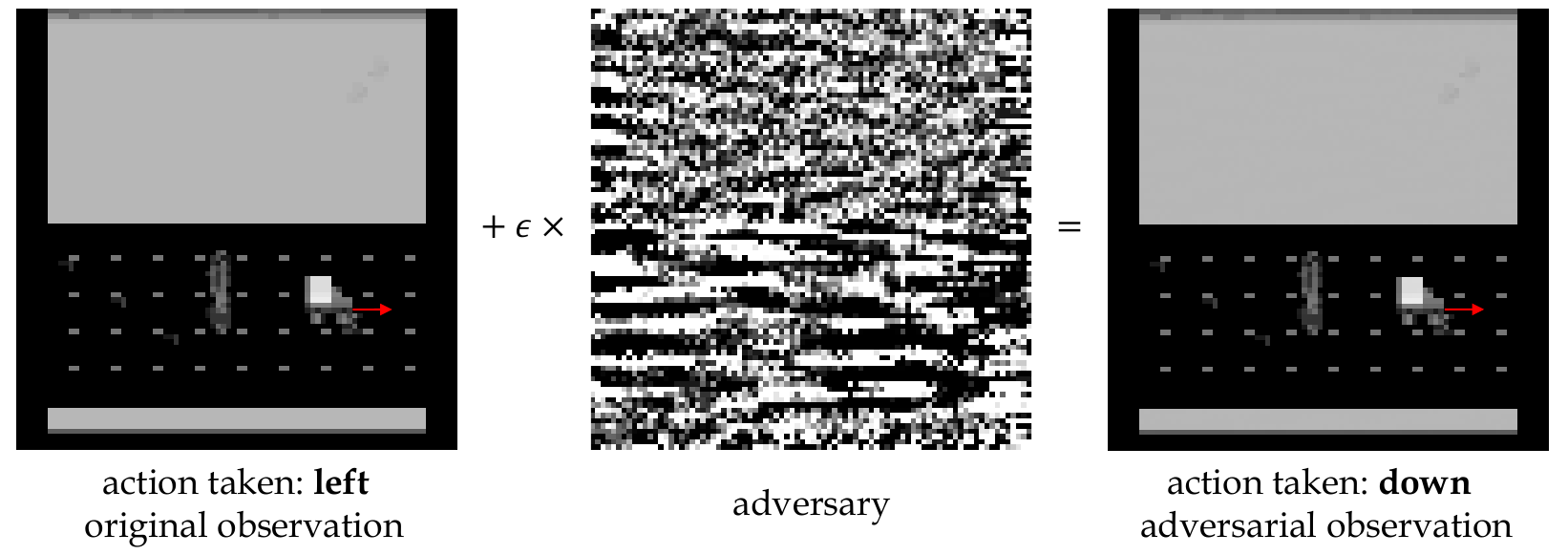}
\includegraphics[width=0.7\textwidth]{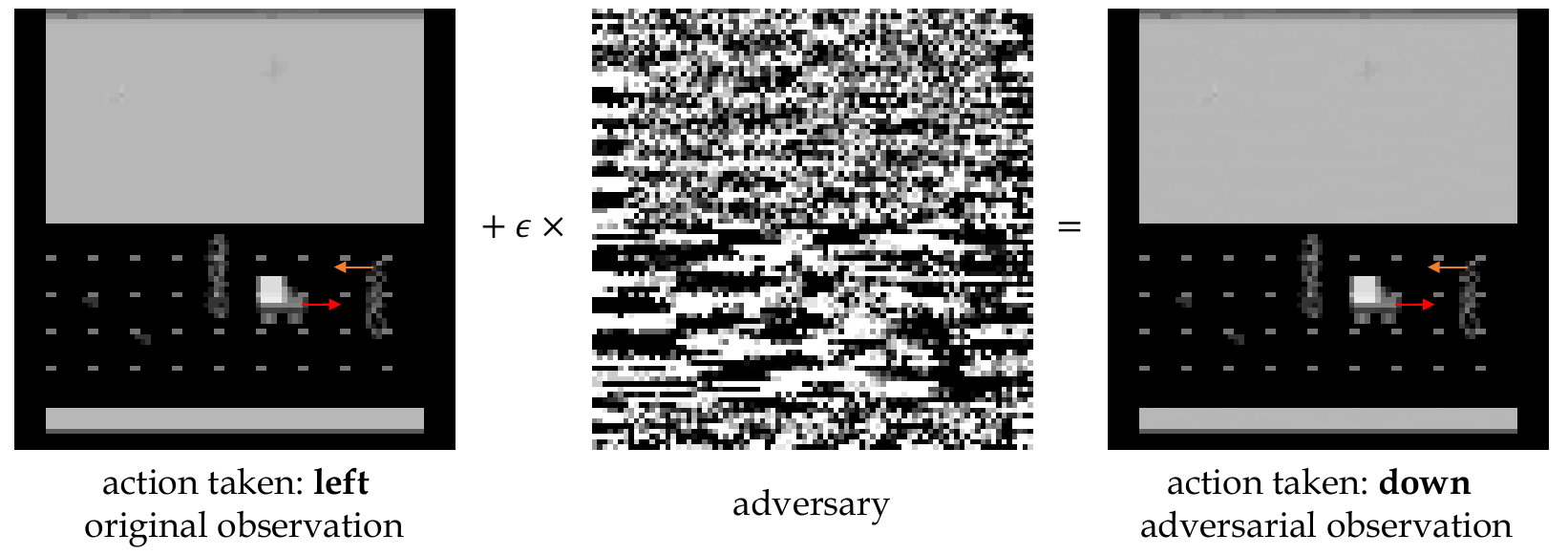}
\includegraphics[width=0.7\textwidth]{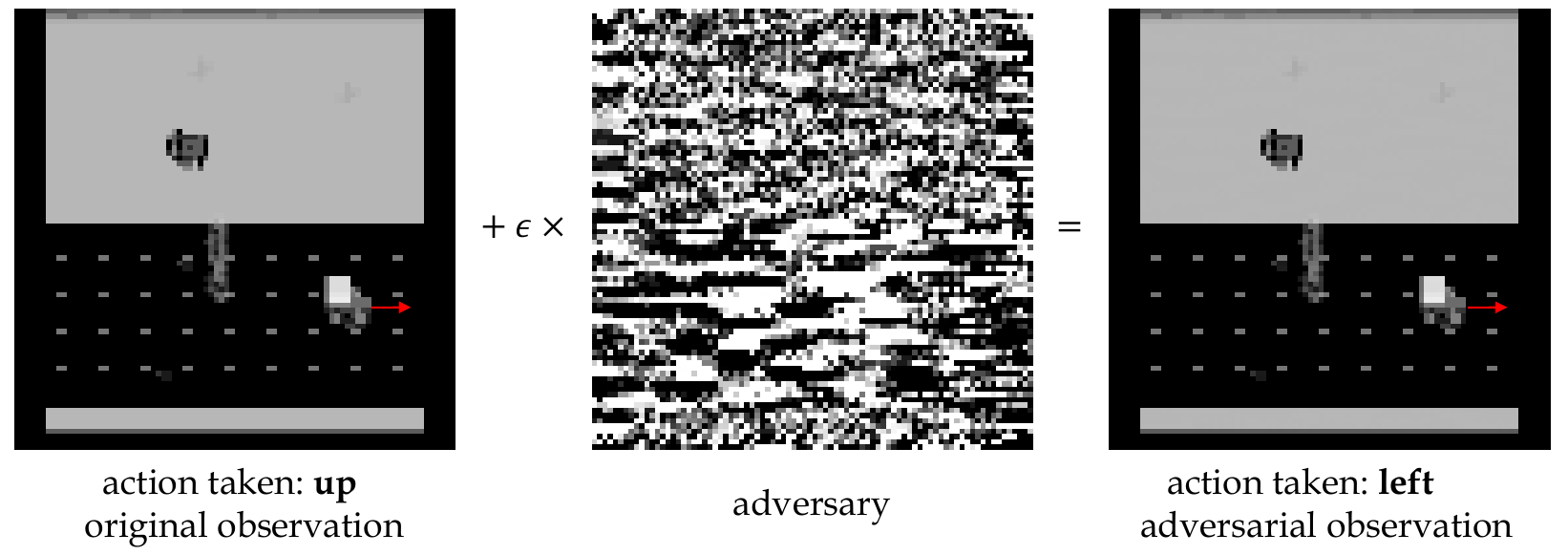}
\includegraphics[width=0.7\textwidth]{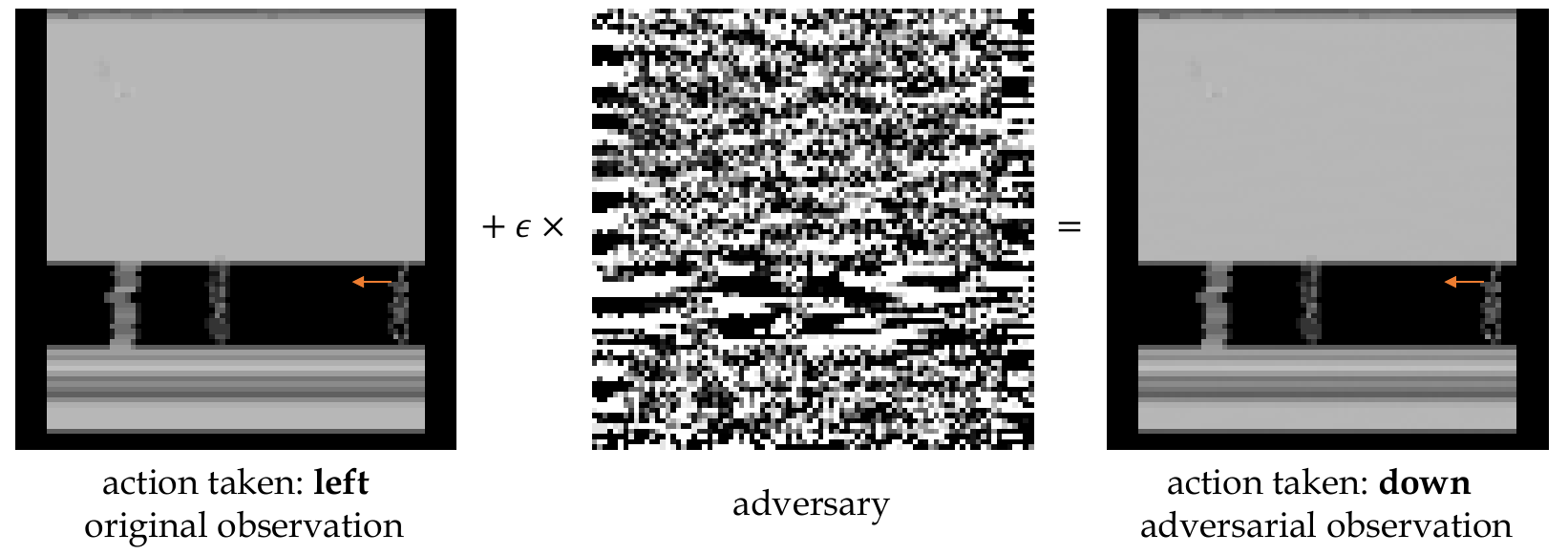}
    \caption{Examples of intrinsic states in Road Runner games. The directions of movement of trucks and the competitors are marked.}
    \label{app fig:intrinsic states road}
\end{figure}

\section{Additional Algorithm Details}

\textbf{Algorithm.} We present the CAR-DQN training algorithm in Algorithm \ref{alg:car-dqn}.

\begin{algorithm}[!tb]
    \caption{Consistent Adversarial Robust Deep Q-Learning (CAR-DQN).}
    \label{alg:car-dqn}
\begin{algorithmic}
    \STATE {\bfseries Input:} Number of iterations $T$, target network update frequency $M$, a schedule $\beta_t$ for the exploration probability $\beta$, a schedule $\epsilon_t$ for the perturbation radius $\epsilon$.
    \STATE Initialize current Q network $Q(s,a)$ with parameters $\theta$ and target Q network $Q^\prime(s,a)$ with parameters $\theta^\prime \leftarrow \theta$. 
    \STATE Initialize replay buffer $\mathcal{B}$.
    \FOR{$t=1$ {\bfseries to} $T$}
    \STATE With probability $\beta_t$ select random action $a_t$, otherwise select $a_t = \mathop{\arg\max}_a Q(s_t, a;\theta)$.
    \STATE Execute action $a_t$ in environment and observe reward $r_t$ and the next state $s_{t+1}$.
    \STATE Store transition pair $\left\{ s_t, a_t, r_t, s_{t+1} \right\}$ in $\mathcal{B}$.
    \STATE Randomly sample a minibatch of $N$ transition pairs $\left\{ s_i, a_i, r_i, s_{i+1} \right\}$ from $\mathcal{B}$.
    \STATE Set $y_i = r_i + \gamma Q^\prime(s_{i+1}, \mathop{\arg\max}_{a^\prime} Q(s_i, a^\prime;\theta);\theta^\prime)$ for non-terminal $s_i$, and $y_i = r_i$ for terminal $s_i$.
    \STATE Define $\mathcal{L}_{car}^{soft}(\theta)$:
    \begin{equation}\notag
        \mathcal{L}_{car}^{soft}(\theta):=\sum_{i\in \mathcal{\left|B\right|}} \alpha_i \max_{s_\nu \in B_{\epsilon_t}(s_i)} \left|y_i - Q(s_\nu,a_i;\theta)  \right|,
    \end{equation}
    where $ \alpha_i = \frac{e^{\frac{1}{\lambda} \max_{s_\nu } \left|y_i - Q(s_\nu,a_i;\theta)  \right|}}{\sum_{i\in \mathcal{\left|B\right|}} e^{\frac{1}{\lambda} \max_{s_\nu } \left|y_i - Q(s_\nu,a_i;\theta)  \right|}}$.
    \STATE Option 1: Use projected gradient descent (PGD) to solve $\mathcal{L}_{car}^{soft}(\theta)$.
    \STATE \qquad For every $i\in \mathcal{\left|B\right|}$, run PGD to solve: $$s_{i,\nu}=\mathop{\arg\max}_{s_\nu \in B_{\epsilon_t}(s_i)}\left|y_i - Q(s_\nu,a_i;\theta)  \right|. $$
    \STATE \qquad Compute the approximation of $\mathcal{L}_{car}^{soft}(\theta)$:
        $$\mathcal{L}_{car}(\theta) = \sum_{i\in \mathcal{\left|B\right|}} \alpha_i \left|y_i - Q(s_{i,\nu},a_i;\theta)  \right|,$$
        \qquad where $ \alpha_i = \frac{e^{\frac{1}{\lambda}  \left|y_i - Q(s_{i,\nu},a_i;\theta)  \right|}}{\sum_{i\in \mathcal{\left|B\right|}} e^{\frac{1}{\lambda}  \left|y_i - Q(s_{i,\nu},a_i;\theta)  \right|}}$.
    \STATE Option 2: Use convex relaxations of neural networks to solve a surrogate loss for $\mathcal{L}_{car}^{soft}(\theta)$.
    \STATE \qquad For every $i\in \mathcal{\left|B\right|}$, obtain upper and lower bounds on $Q(s, a^\prime;\theta)$ for all $s \in B_{\epsilon_t}(s_i)$:
        $$u_i(\theta)=\operatorname{ConvexRelaxUB}\left( Q(s, a^\prime;\theta), \theta,  s \in B_{\epsilon_t}(s_i)\right),$$
        $$l_i(\theta)=\operatorname{ConvexRelaxLB}\left( Q(s, a^\prime;\theta), \theta,  s \in B_{\epsilon_t}(s_i)\right).$$
    \STATE \qquad Compute the surrogate loss for $\mathcal{L}_{car}^{soft}(\theta)$:
        $$\mathcal{L}_{car}(\theta) = \sum_{i\in \mathcal{\left|B\right|}} \alpha_i \max\left\{ \left| y_i - u_i(\theta)\right|, \left| y_i - l_i(\theta)\right| \right\}, $$
        \qquad where $ \alpha_i = \frac{e^{\frac{1}{\lambda}  \max\left\{ \left| y_i - u_i(\theta)\right|, \left| y_i - l_i(\theta)\right| \right\}}}{\sum_{i\in \mathcal{\left|B\right|}} e^{\frac{1}{\lambda}  \max\left\{ \left| y_i - u_i(\theta)\right|, \left| y_i - l_i(\theta)\right| \right\}}}$.
    \STATE Update the Q network by performing a gradient descent step to minimize $\mathcal{L}_{car}(\theta)$.
    \STATE Update target Q network every $M$ steps: $\theta^\prime \leftarrow \theta$.
    \ENDFOR
\end{algorithmic}
    
\end{algorithm}

\textbf{DQN architecture.}
We implement Dueling network architectures~\cite{wang2016dueling} and the same architecture following \cite{zhang2020robust, oikarinen2021robust} which has 3 convolutional layers and a two-head fully connected layers. The first convolutional layer has $8\times 8$ kernel, stride 4, and 32 channels. The second convolutional layer has $4\times 4$ kernel, stride 2, and 64 channels. The third convolutional layer has $3\times 3$ kernel, stride 1, and 64 channels and is then flattened. The fully connected layers have 512 units for both heads wherein one head outputs a state value and the other outputs advantages of each action. Every middle layer is applied by the ReLU activation function.

\section{Comparation between RADIAL-DQN and CAR-DQN with Increasing Perturbation Radius} \label{app: exp with increasing perturbation radius}

The core at RADIAL-DQN is a heuristic robust regularization that minimizes the overlap between bounds of perturbed Q values of the current action and others:
$$
\mathcal{L}_{radial}\left(\theta\right)=\mathbb{E}_{\left(s, a, s^{\prime}, r\right)}\left[\sum_y Q_{\text {diff}}(s, y) \cdot {Ovl}(s, y, \epsilon)\right],
$$
where
$
Q_{\text {diff}}(s, y)=\max (0, Q(s, a)-Q(s, y)), \ O v l(s, y, \epsilon)=\max (0, \bar{Q}(s, y, \epsilon)-\underline{Q}(s, a, \epsilon)+\eta)
$ and $\eta=c\cdot Q_{\text {diff}}(s, y), c=0.5$. $Q_{\text {diff}}$ is treated as a constant for the optimization. 
We consider RADIAL-DQN could perform better than CAR-DQN with increasing perturbation radius since $\mathcal{L}_{radial}\left(\theta\right)$ is a stronger regularization to enhance robustness while compromising natural rewards. The stronger robust constraint is mainly reflected in two aspects:
\begin{itemize}
    \item The loose bounds. RADIAL-DQN uses the cheap but loose convex relaxation method (IBP) to estimate $\bar{Q}(s, y, \epsilon)$ and $\underline{Q}(s, a, \epsilon)$.
    \item The positive margin $\eta$.  
\end{itemize}
They both result in ${Ovl}(s, y, \epsilon)$ a stronger constraint for representing the overlap of perturbed Q values.
However, RADIAL-DQN has the following weaknesses:
\begin{itemize}
    \item $\mathcal{L}_{radial}\left(\theta\right)$ will harm natural rewards. As shown in Figure~\ref{fig: natural and robustness rewards during training}, the natural rewards curve of RADIAL-DQN on RoadRunner distinctly tends to decrease, especially around 0.5 million steps. In contrast, the natural curves of our CAR-DQN showcase more stable upward trends in all environments. Besides, as shown in Table~\ref{table: car dqn trained with larger radius}, RADIAL-DQN training with a larger radius attains lower natural rewards which also restricts robustness according to our theory, while CAR-DQN keeps a better and consistent natural and robust performance.
    \item Heuristic implementation lacks theoretical guarantees and introduces sensitive hyperparameter $c$. We conduct additional experiments on RoadRunner with different $c$ and observe the sensitivity of RADIAL-DQN to the choice of $c$. Larger $c$ could cause poor performance because the robustness constraint is too strict and thus the policy degrades to some simple policy with lower rewards. Smaller $c$ may result in much weaker robustness. By contrast, our CAR-DQN is developed based on the theory of optimal robust policy and stability. Although we also introduce a hyperparameter $\lambda$, our ablation studies in Figure~\ref{fig:soft roadrunner} show that our algorithm is insensitive to it. 

    \begin{table}[h]
    \caption{Performance of RADIAL-DQN sensitive to different positive margins $c\cdot Q_{\text {diff}}(s, y)$.}
    \vskip 0.15in
    \resizebox{\textwidth}{!}{%
    \begin{tabular}{c|c|ccc|ccc}
    \hline
    \multirow{2}{*}{Model} & \multirow{2}{*}{Natural Return} & \multicolumn{3}{c|}{PGD}                                & \multicolumn{3}{c}{MinBest}                            \\ \cline{3-8} 
                           &                                 & $\epsilon=1/255$ & $\epsilon=3/255$  & $\epsilon=5/255$ & $\epsilon=1/255$ & $\epsilon=3/255$  & $\epsilon=5/255$ \\ \hline
    RADIAL-DQN ($c=0.25$)  & $14678 \pm 329$                 & $ 14836\pm 314$  & $ 13670 \pm 466$  & $ 13512 \pm 617$ & $14712 \pm 309$  & $14804 \pm 457$   & $13226 \pm 351$  \\ \hline
    RADIAL-DQN ($c=0.5$)   & $\bf 46224 \pm 1133$                & $\bf 45990 \pm 1112$ & $\bf  42162 \pm 1147$ & $\bf  23248 \pm 499$ & $\bf 46082 \pm 1128$ & $\bf  42036 \pm 1048$ & $ \bf 25434 \pm 756$ \\ \hline
    RADIAL-DQN ($c=0.75$)  & $3992 \pm 482$                  & $3992 \pm 482$   & $3992 \pm 482$    & $3992 \pm 482$   & $3992 \pm 482$   & $3992 \pm 482$    & $3992 \pm 482$   \\ \hline
    \end{tabular}%
    }
    \end{table}
    \item Depending on the currently learned optimal action. $\mathcal{L}_{radial}\left(\theta\right)$ essentially takes the currently learned action as a robust label which may produce a wrong direction for robustness training if the learned action is not optimal. In contrast, our CAR-DQN seeks optimal robust policies with theoretical guarantees and does not utilize the learned action for robustness training, simultaneously improving natural and robust performance.
\end{itemize}

The main motivation of CAR-DQN based on our theory is to improve natural and robust performance concurrently which makes sense in real-world scenarios where strong adversarial attacks are relatively rare. 
Our training loss can guarantee robustness under the attack with the same perturbation radius as the training. We also think it is a very significant problem whether and how we can design an algorithm training with little epsilon and theoretically guarantee robustness for larger epsilon. However, this is beyond the scope of our paper and we will consider this problem in subsequent work. 

Moreover, as shown in Table~\ref{app table: compare}, CAR-DQN also achieves the top performance in larger perturbation radiuses on Pong and BankHeist and matches the RADIAL-DQN on Freeway. To show the superiority of CAR-DQN further, we also train CAR-DQN with a perturbation radius of 3/255 and 5/255 on RoadRunner for 4.5 million steps (see Table~\ref{table: car dqn trained with larger radius}).
\begin{table}[h]
\caption{Performance of CAR-DQN and RADIAL-DQN trained with different perturbation radiuses on the RoadRunner environment. The best results of the algorithm with the same training radius are highlighted in bold.}
\label{table: car dqn trained with larger radius}
\vskip 0.15in
\resizebox{\textwidth}{!}{%
\begin{tabular}{c|c|ccc|ccc}
\hline
\multirow{2}{*}{Model}        & \multirow{2}{*}{Natural Return} & \multicolumn{3}{c|}{PGD}                                  & \multicolumn{3}{c}{MinBest}                               \\ \cline{3-8} 
                              &                                 & $\epsilon=1/255$  & $\epsilon=3/255$  & $\epsilon=5/255$  & $\epsilon=1/255$   & $\epsilon=3/255$  & $\epsilon=5/255$  \\ \hline
RADIAL-DQN ($\epsilon=1/255$) & $46224 \pm 1133$                & $45990 \pm 1112$  & $\bf  42162 \pm 1147$ & $\bf  23248 \pm 499$  & $46082 \pm 1128$   & $ \bf 42036 \pm 1048$ & $\bf  25434 \pm 756$  \\ \hline
CAR-DQN ($\epsilon=1/255$)    & $\bf 49398 \pm 1106$               & $\bf  49456 \pm 992$  & $28588 \pm 1575$  & $15592 \pm 885$   & $\bf  47526 \pm 1132$  & $32878 \pm 1898$  & $21102 \pm 1427$  \\ \hline \hline
RADIAL-DQN ($\epsilon=3/255$) & $34656 \pm 1104$                & $35094 \pm 1277$  & $ 35082 \pm 948 $ & $\bf  32770 \pm 1062$ & $ 35096 \pm 1277 $ & $ 34374 \pm 996$  & $\bf  27926 \pm 881$  \\ \hline
CAR-DQN ($\epsilon=3/255$)    & $\bf 47348 \pm 1305$                & $\bf 46284 \pm 1114 $ & $\bf  43578 \pm 1315$ & $ 27060 \pm 1117$ & $\bf  46286 \pm 1122 $ & $\bf  42602 \pm 1336$ & $ 24862 \pm 1195$ \\ \hline \hline
RADIAL-DQN ($\epsilon=5/255$) & $35160 \pm 1157$                & $36158 \pm 1104$  & $36732 \pm 1076$  & $34826 \pm 913$   & $36158 \pm 1104$   & $36732 \pm 1076$  & $34592 \pm 913$   \\ \hline
CAR-DQN ($\epsilon=5/255$)    & $\bf 42545 \pm 2028$                & $\bf 43230 \pm 1468 $ & $\bf  37845 \pm 2344$ & $\bf  39235 \pm 1519$ & $\bf  43645 \pm 1531 $ & $\bf  37535 \pm 2112$ & $\bf  38150 \pm 1316$ \\ \hline
\end{tabular}%
}
\end{table}
We can see that CAR-DQN still attains superior natural and robust performance training with larger attack radiuses while RADIAL-DQN markedly degrades its natural performance due to the too-strong robustness constraint. CAR-DQN always has a higher robust return on the training radius than RADIAL-DQN.

\section{Additional Experiment Results}\label{app: add exp}

\textbf{Training stability.} We also observe that there are some instability phenomena in the training of CAR, RADIAL, and WocaR in Figure \ref{fig: natural and robustness rewards during training}. We conjecture that the occasional instability in CAR training comes from the unified loss combining natural and robustness objectives which may cause undesirable optimization direction under a batch of special samples. The instability of RADIAL is particularly evident in the robustness curve on the BankHeist environment and natural curve on the Freeway environment and it may be from the larger batch size (=128) setting during the RADIAL training while CAR, SA-DQN and WocaR set the batch size as 32 or 16. The worst-case estimation of WocaR may be inaccurate in some states and WocaR also uses a small batch of size 16. The combination of these two can lead to instability, especially in complex environments such as RoadRunner and BankHeist. Another possible reason is that CAR, RADIAL, and WocaR all use the cheap relaxation method leading to a loose bound while SA-DQN utilizes a tighter relaxation.

\begin{table*}[ht]
\caption{Average episode rewards $\pm$ standard error of the mean over 50 episodes on baselines and CAR-DQN. The best results of the algorithm with the same type of solver are highlighted in bold.}
\label{app table: compare}
\vskip 0.15in
\resizebox{\textwidth}{!}{%
\begin{tabular}{c|cc|c|ccc|ccc|ccc}
\hline \hline
\multirow{2}{*}{Environment} & \multicolumn{2}{c|}{\multirow{2}{*}{Model}}                                                                         & \multirow{2}{*}{Natural Return} & \multicolumn{3}{c|}{PGD}                                               & \multicolumn{3}{c|}{MinBest}                                                      & \multicolumn{3}{c}{ACR}                               \\
                             & \multicolumn{2}{c|}{}                                                                                               &                                 & $\epsilon=1/255$          & $\epsilon=3/255$          & $\epsilon=5/255$          & $\epsilon=1/255$          & $\epsilon=3/255$          & $\epsilon=5/255$          & $\epsilon=1/255$ & $\epsilon=3/255$ & $\epsilon=5/255$ \\ \hline
\multirow{7}{*}{\textbf{Pong}}        & \multicolumn{1}{c|}{Standard}                                                                      & DQN            & $21.0 \pm 0.0$                  & $-21.0 \pm 0.0$           & $-21.0 \pm 0.0$           & $-20.8 \pm 0.1$           & $-21.0 \pm 0.0$           & $-21.0 \pm 0.0$           & $-21.0 \pm 0.0$           & $0$              & $0$              & $0$              \\ \cline{2-13} 
                             & \multicolumn{1}{c|}{\multirow{2}{*}{PGD}}                                                          & SA-DQN         & \textbf{$\bf 21.0 \pm 0.0$}         & \textbf{$\bf 21.0 \pm 0.0$}   & $-19.4 \pm 0.3$           & $-21.0 \pm 0.0$           & \textbf{$\bf 21.0 \pm 0.0$}   & $-19.4 \pm 0.2$           & $-21.0 \pm 0.0$           & $0$              & $0$              & $0$              \\
                             & \multicolumn{1}{c|}{}                                                                              & CAR-DQN (Ours) & \textbf{$\bf 21.0 \pm 0.0$}         & \textbf{$\bf 21.0 \pm 0.0$}   & $\bf 16.8 \pm 0.7$            & $-21.0 \pm 0.0$           & \textbf{$\bf 21.0 \pm 0.0$}   & $\bf 20.7 \pm 0.1$            & $\bf -0.8 \pm 2.8$            & $0$              & $0$              & $0$              \\ \cline{2-13} 
                             & \multicolumn{1}{c|}{\multirow{4}{*}{\begin{tabular}[c]{@{}c@{}}Convex \\ Relaxation\end{tabular}}} & SA-DQN         & \textbf{$\bf 21.0 \pm 0.0$}         & \textbf{$\bf 21.0 \pm 0.0$}   & \textbf{$\bf 21.0 \pm 0.0$}   & $-19.6 \pm 0.1$           & \textbf{$\bf 21.0 \pm 0.0$}   & \textbf{$\bf 21.0 \pm 0.0$}   & $-9.5 \pm 1.3$            & $1.000$          & $0$              & $0$              \\
                             & \multicolumn{1}{c|}{}                                                                              & RADIAL-DQN     & \textbf{$\bf 21.0 \pm 0.0$}         & \textbf{$\bf 21.0 \pm 0.0$}   & \textbf{$\bf 21.0 \pm 0.0$}   & \textbf{$\bf 21.0 \pm 0.0$}   & \textbf{$\bf 21.0 \pm 0.0$}   & \textbf{$\bf 21.0 \pm 0.0$}   & $4.9 \pm 0.6$             & $0.898$          & $0$              & $0$              \\
                             & \multicolumn{1}{c|}{}                                                                              & WocaR-DQN      & \textbf{$\bf 21.0 \pm 0.0$}         & \textbf{$\bf 21.0 \pm 0.0$}   & $20.5 \pm 0.1$            & $20.6 \pm 0.1$            & \textbf{$\bf 21.0 \pm 0.0$}   & $20.7 \pm 0.1$            & $20.9 \pm 0.1$            & $0.979$          & $0$              & $0$              \\
                             & \multicolumn{1}{c|}{}                                                                              & CAR-DQN (Ours) & \textbf{$\bf 21.0 \pm 0.0$}         & \textbf{$\bf 21.0 \pm 0.0$}   & \textbf{$\bf 21.0 \pm 0.0$}   & \textbf{$\bf 21.0 \pm 0.0$}   & \textbf{$\bf 21.0 \pm 0.0$}   & \textbf{$\bf 21.0 \pm 0.0$}   & \textbf{$\bf 21.0 \pm 0.0$}   & $0.986$          & $0$              & $0$              \\ \hline \hline
\multirow{7}{*}{\textbf{Freeway}}     & \multicolumn{1}{c|}{Standard}                                                                      & DQN            & $33.9 \pm 0.0$                  & $0.0 \pm 0.0$             & $0.0 \pm 0.0$             & $0.0 \pm 0.0$             & $0.0 \pm 0.0$             & $0.0 \pm 0.0$             & $0.0 \pm 0.0$             & $0$              & $0$              & $0$              \\ \cline{2-13} 
                             & \multicolumn{1}{c|}{\multirow{2}{*}{PGD}}                                                          & SA-DQN         & $33.6 \pm 0.1$                  & $23.4 \pm 0.2$            & $20.6 \pm 0.3$            & \textbf{$\bf 7.6 \pm 0.3$}    & $21.1 \pm 0.2$            & $21.3 \pm 0.2$            & $21.8 \pm 0.3$            & $0.250$          & $0.275$          & $0.275$          \\
                             & \multicolumn{1}{c|}{}                                                                              & CAR-DQN (Ours) & \textbf{$\bf 34.0 \pm 0.0$}         & \textbf{$\bf 33.7 \pm 0.1$}   & \textbf{$\bf 25.8 \pm 0.2$}   & $3.8 \pm 0.2$             & \textbf{$\bf 33.7 \pm 0.1$}   & \textbf{$\bf 30.0 \pm 0.3$}   & \textbf{$\bf 26.2 \pm 0.2$}   & $0$              & $0$              & $0$              \\ \cline{2-13} 
                             & \multicolumn{1}{c|}{\multirow{4}{*}{\begin{tabular}[c]{@{}c@{}}Convex \\ Relaxation\end{tabular}}} & SA-DQN         & $30.0 \pm 0.0$                  & $30.0 \pm 0.0$            & $30.2 \pm 0.1$            & $27.7 \pm 0.1$            & $30.0 \pm 0.0$            & $30.0 \pm 0.0$            & $29.2 \pm 0.1$            & $1.000$          & $0.912$          & $0$              \\
                             & \multicolumn{1}{c|}{}                                                                              & RADIAL-DQN     & $33.1 \pm 0.1$                  & \textbf{$\bf 33.3 \pm 0.1$}   & \textbf{$\bf 33.3 \pm 0.1$}   & \textbf{$\bf 29.0 \pm 0.1$}   & \textbf{$\bf 33.3 \pm 0.1$}   & \textbf{$\bf 33.3 \pm 0.1$}   & \textbf{$\bf 31.2 \pm 0.2$}   & $0.998$          & $0$              & $0$              \\
                             & \multicolumn{1}{c|}{}                                                                              & WocaR-DQN      & $30.8 \pm 0.1$                  & $31.0 \pm 0.0$            & $30.6 \pm 0.1$            & $29.0 \pm 0.2$            & $31.0 \pm 0.0$            & $31.1 \pm 0.1$            & $29.0 \pm 0.2$            & $0.992$          & $0.150$          & $0$              \\
                             & \multicolumn{1}{c|}{}                                                                              & CAR-DQN (Ours) & \textbf{$\bf 33.2 \pm 0.1$}         & $33.2 \pm 0.1$            & $32.3 \pm 0.2$            & $27.6 \pm 0.3$            & $33.2 \pm 0.1$            & $32.8 \pm 0.2$            & $31.0 \pm 0.1$            & $0.981$          & $0$              & $0$              \\ \hline \hline
\multirow{7}{*}{\textbf{BankHeist}}   & \multicolumn{1}{c|}{Standard}                                                                      & DQN            & $1317.2 \pm 4.2$                & $22.2 \pm 1.9$            & $0.0 \pm 0.0$             & $0.0 \pm 0.0$             & $0.0 \pm 0.0$             & $0.0 \pm 0.0$             & $0.0 \pm 0.0$             & $0$              & $0$              & $0$              \\ \cline{2-13} 
                             & \multicolumn{1}{c|}{\multirow{2}{*}{PGD}}                                                          & SA-DQN         & $1248.8 \pm 1.4$                & $965.8 \pm 35.9$          & $35.6 \pm 3.4$            & $0.6 \pm 0.3$             & $1118.0 \pm 6.3$          & $50.8 \pm 2.5$            & $4.8 \pm 0.7$             & $0$              & $0$              & $0$              \\
                             & \multicolumn{1}{c|}{}                                                                              & CAR-DQN (Ours) & \textbf{$\bf 1307.0 \pm 6.1$}       & \textbf{$\bf 1243.2 \pm 7.4$} & \textbf{$\bf 908.2 \pm 17.0$} & \textbf{$\bf 83.0 \pm 2.2$}   & \textbf{$\bf 1242.6 \pm 8.4$} & \textbf{$\bf 970.8 \pm 9.6$}  & \textbf{$\bf 819.4 \pm 9.0$}  & $0$              & $0$              & $0$              \\ \cline{2-13} 
                             & \multicolumn{1}{c|}{\multirow{4}{*}{\begin{tabular}[c]{@{}c@{}}Convex \\ Relaxation\end{tabular}}} & SA-DQN         & $1236.0 \pm 1.4$                & $1232.2 \pm 2.5$          & $1208.8 \pm 1.7$          & $1029.8 \pm 34.6$         & $1232.2 \pm 2.5$          & $1214.8 \pm 2.6$          & $1051.0 \pm 35.5$         & $0.991$          & $0.409$          & $0$              \\
                             & \multicolumn{1}{c|}{}                                                                              & RADIAL-DQN     & $1341.8 \pm 3.8$                & $1341.8 \pm 3.8$          & \textbf{$\bf 1346.4 \pm 3.2$} & $1092.6 \pm 37.8$         & $1341.8 \pm 3.8$          & $1328.6 \pm 5.4$          & $732.6 \pm 11.5$          & $0.982$          & $0$              & $0$              \\
                             & \multicolumn{1}{c|}{}                                                                              & WocaR-DQN      & $1315.0 \pm 6.1$                & $1312.0 \pm 6.1$          & $1323.4 \pm 2.2$          & $1094.0 \pm 10.2$         & $1312.0 \pm 6.1$          & $1301.6 \pm 3.9$          & $1041.4 \pm 17.4$         & $0.987$          & $0.093$          & $0$              \\
                             & \multicolumn{1}{c|}{}                                                                              & CAR-DQN (Ours) & \textbf{$\bf 1349.6 \pm 3.0$}       & \textbf{$\bf 1347.6 \pm 3.6$} & $1332.0 \pm 7.3$          & \textbf{$\bf 1191.0 \pm 9.0$} & \textbf{$\bf 1347.4 \pm 3.6$} & \textbf{$\bf 1338.0 \pm 2.9$} & \textbf{$\bf 1233.6 \pm 5.0$} & $0.974$          & $0$              & $0$              \\ \hline \hline
\multirow{7}{*}{\textbf{RoadRunner}}  & \multicolumn{1}{c|}{Standard}                                                                      & DQN            & $41492 \pm 903$                 & $0 \pm 0$                 & $0 \pm 0$                 & $0 \pm 0$                 & $0 \pm 0$                 & $0 \pm 0$                 & $0 \pm 0$                 & $0$              & $0$              & $0$              \\ \cline{2-13} 
                             & \multicolumn{1}{c|}{\multirow{2}{*}{PGD}}                                                          & SA-DQN         & $33380 \pm 611$                 & $20482 \pm 1087$          & $0 \pm 0$                 & $0 \pm 0$                 & $24632 \pm 812$           & $614 \pm 72$              & $214 \pm 26$              & $0$              & $0$              & $0$              \\
                             & \multicolumn{1}{c|}{}                                                                              & CAR-DQN (Ours) & \textbf{$\bf 49700 \pm 1015$}       & \textbf{$\bf 43286 \pm 801$}  & \textbf{$\bf 25740 \pm 1468$} & \textbf{$\bf 2574 \pm 261$}   & \textbf{$\bf 48908 \pm 1107$} & \textbf{$\bf 35882 \pm 904$}  & \textbf{$\bf 23218 \pm 698$}  & $0$              & $0$              & $0$              \\ \cline{2-13} 
                             & \multicolumn{1}{c|}{\multirow{4}{*}{\begin{tabular}[c]{@{}c@{}}Convex \\ Relaxation\end{tabular}}} & SA-DQN         & $46372 \pm 882$                 & $44960\pm 1152$           & $20910 \pm 827$           & $3074 \pm 179$            & $45226 \pm 1102$          & $25548 \pm 737$           & $12324 \pm 529$           & $0.819$          & $0$              & $0$              \\
                             & \multicolumn{1}{c|}{}                                                                              & RADIAL-DQN     & $46224\pm 1133$                 & $45990 \pm 1112$          & \textbf{$\bf 42162 \pm 1147$} & \textbf{$\bf 23248 \pm 499$}  & $46082 \pm 1128$          & \textbf{$\bf 42036 \pm 1048$} & \textbf{$\bf 25434 \pm 756$}  & $0.994$          & $0$              & $0$              \\
                             & \multicolumn{1}{c|}{}                                                                              & WocaR-DQN      & $43686 \pm 1608$                & $45636 \pm 706$           & $19386 \pm 721$           & $6538 \pm 464$            & $45636 \pm 706$           & $21068 \pm 1026$          & $15050 \pm 683$           & $0.956$          & $0$              & $0$              \\
                             & \multicolumn{1}{c|}{}                                                                              & CAR-DQN (Ours) & \textbf{$\bf 49398 \pm 1106$}       & \textbf{$\bf 49456 \pm 992$}  & $28588 \pm 1575$          & $15592 \pm 885$           & \textbf{$\bf 47526 \pm 1132$} & $32878 \pm 1898$          & $21102 \pm 1427$          & $0.760$          & $0$              & $0$              \\ \hline
\end{tabular}%
}
\end{table*}

\begin{figure}[t]
    \centering
    \includegraphics[width=0.23\textwidth]{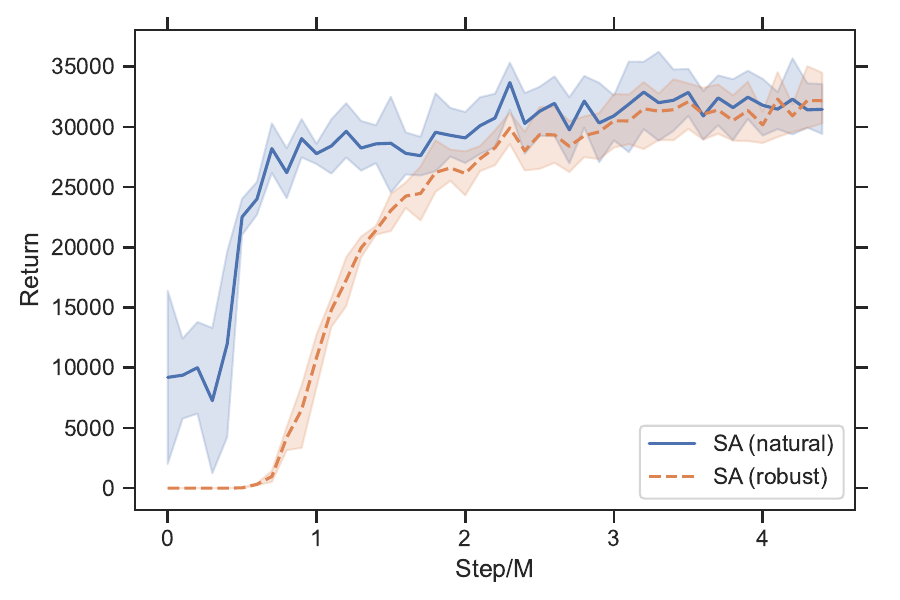}
    \includegraphics[width=0.23\textwidth]{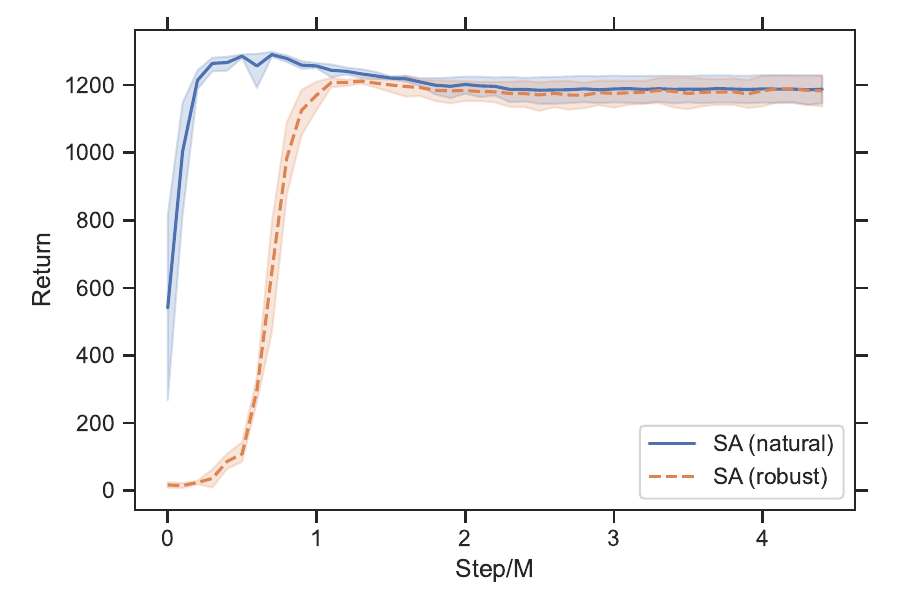}
    \includegraphics[width=0.23\textwidth]{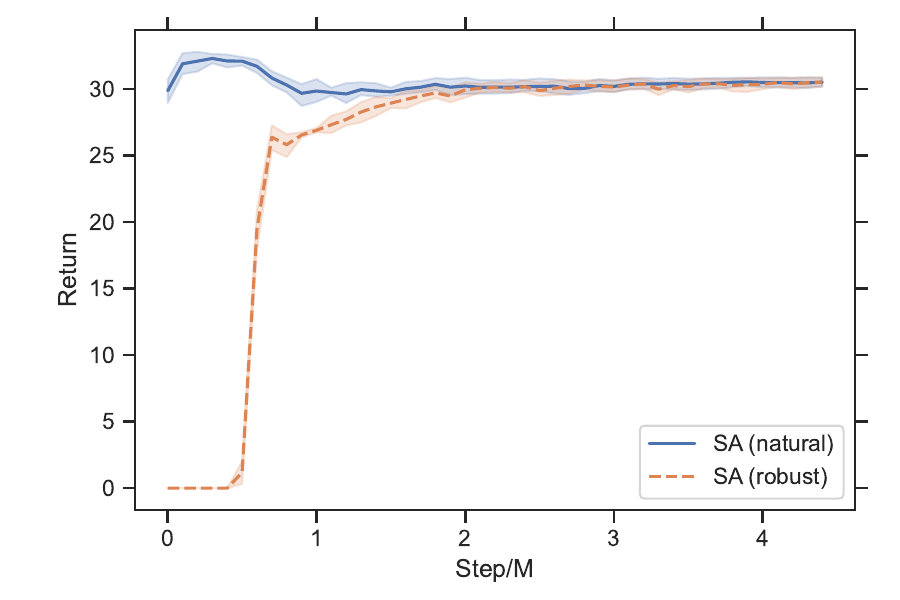}
    \includegraphics[width=0.23\textwidth]{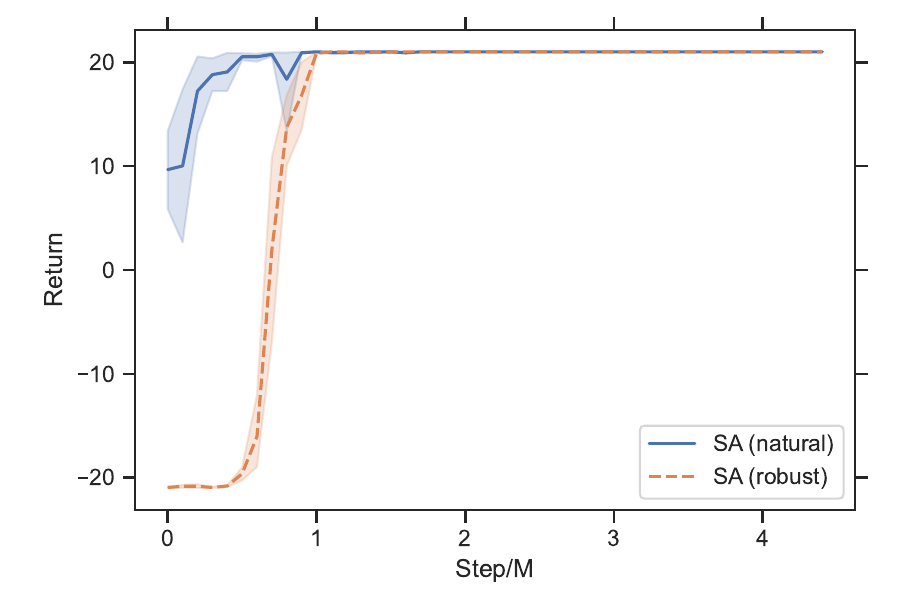}
    \caption{Natural and robustness performance exhibited
by SA-DQN agents during the training process on 4 Atari games.}
    \label{app fig:sa NRreturn}
\end{figure}

\begin{figure}[t]
    \centering
\includegraphics[width=0.23\textwidth]{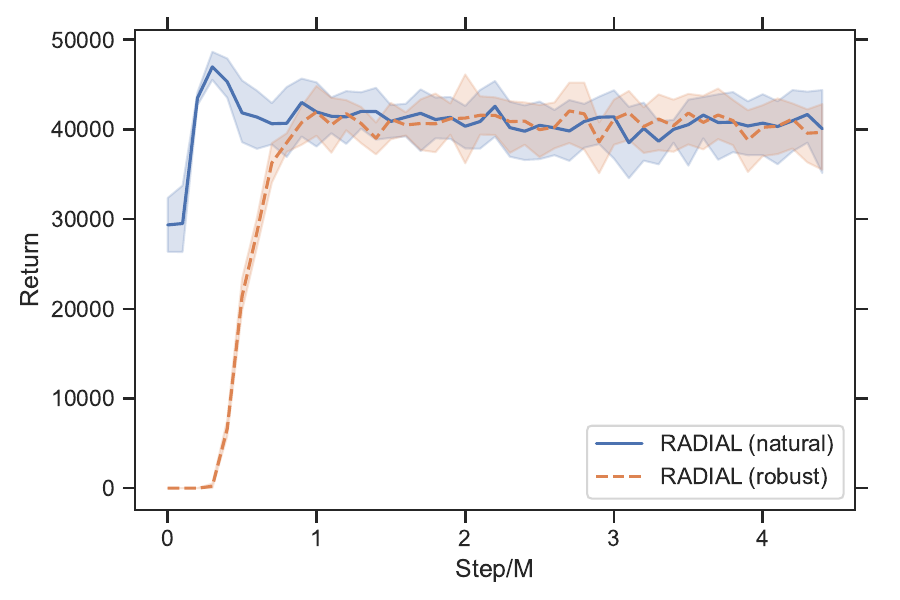}
\includegraphics[width=0.23\textwidth]{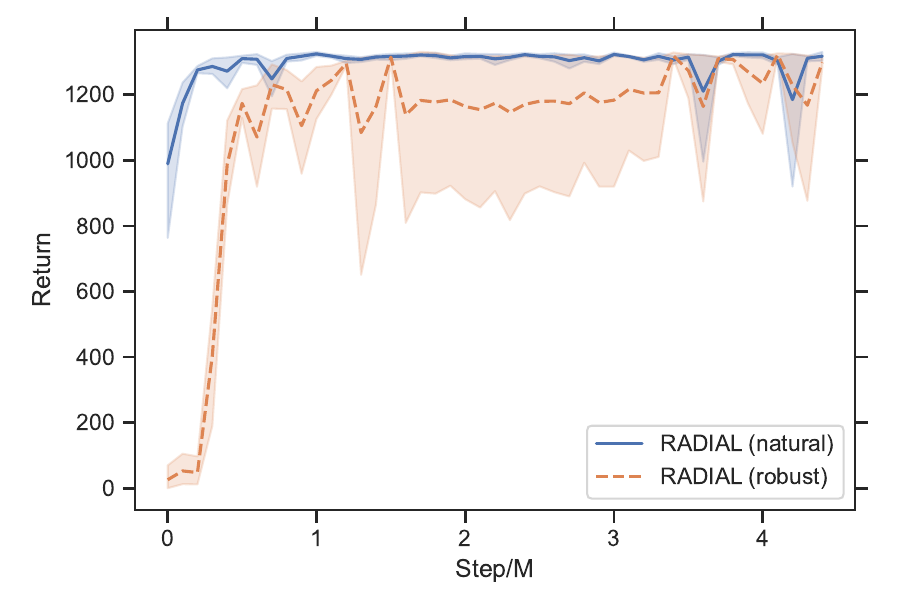}
\includegraphics[width=0.23\textwidth]{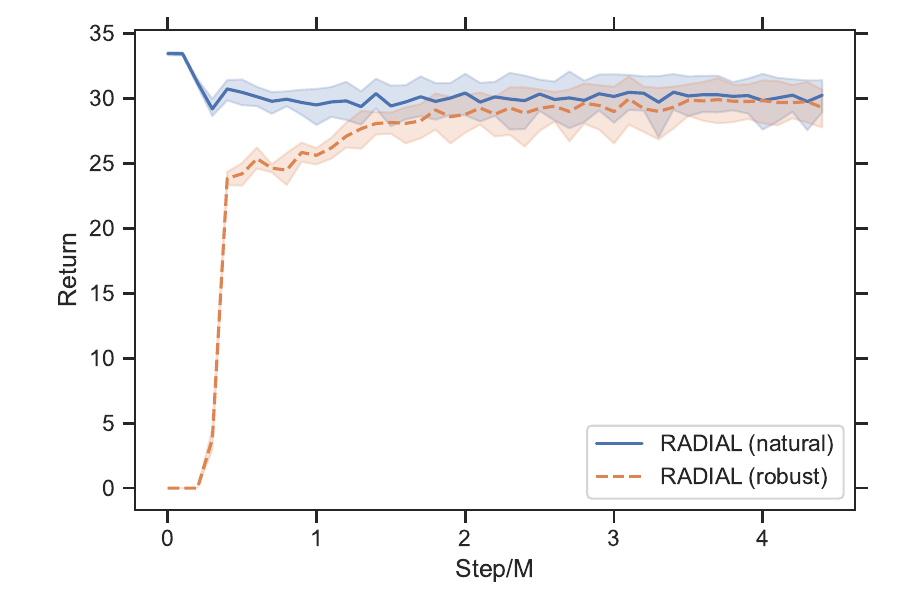}
\includegraphics[width=0.23\textwidth]{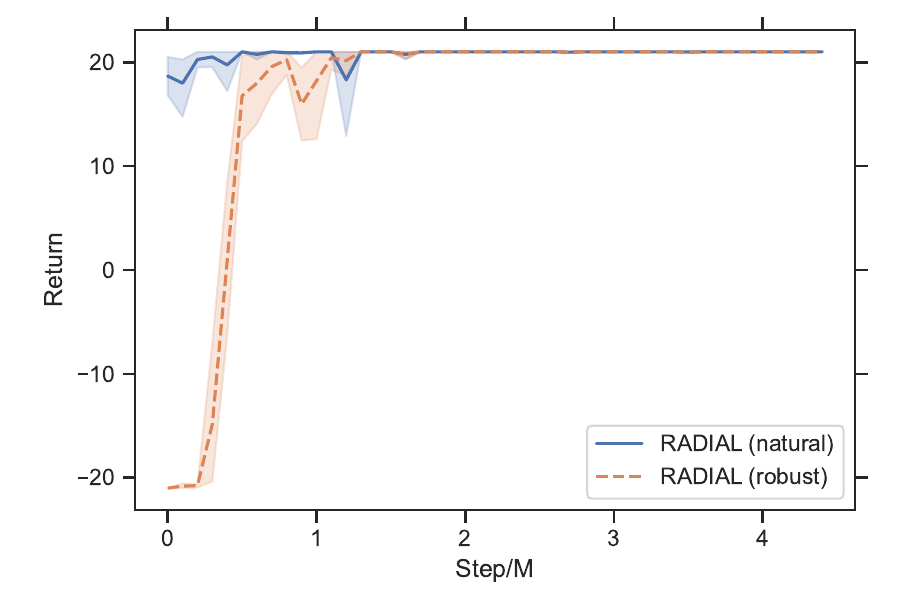}
    \caption{Natural and robustness performance exhibited
by RADIAL-DQN agents during the training process on 4 Atari games.}
    \label{app fig:ra NRreturn}
\end{figure}

\begin{figure}[t]
    \centering
\includegraphics[width=0.23\textwidth]{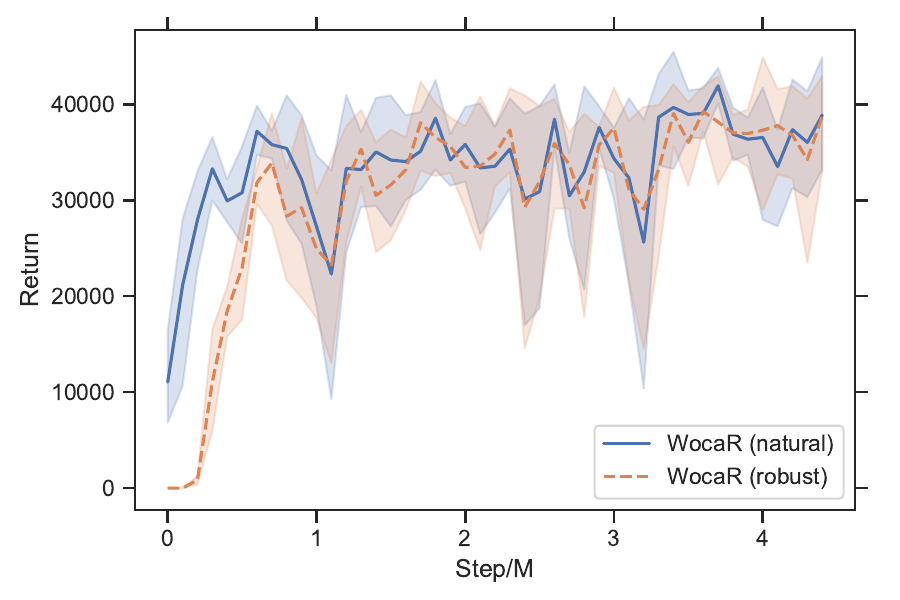}
\includegraphics[width=0.23\textwidth]{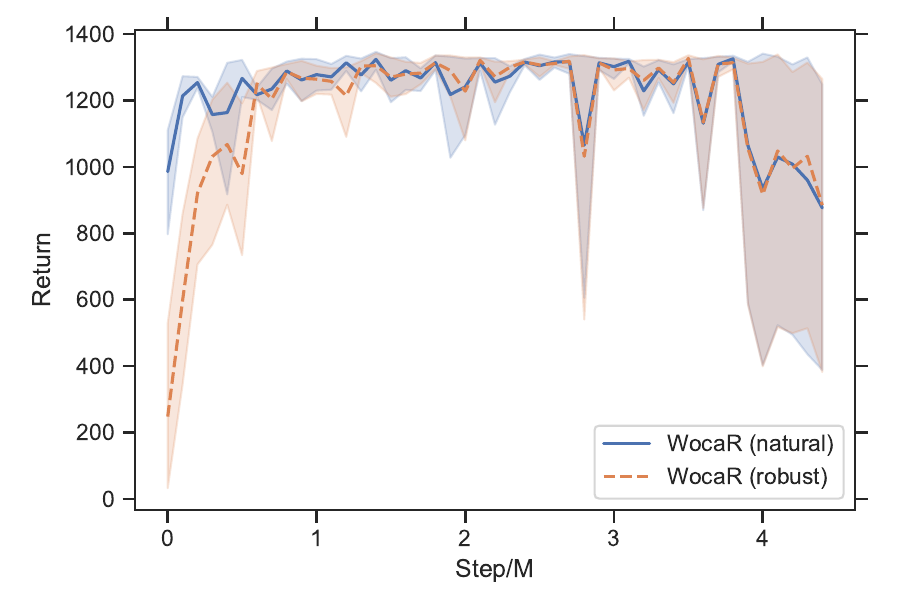}
\includegraphics[width=0.23\textwidth]{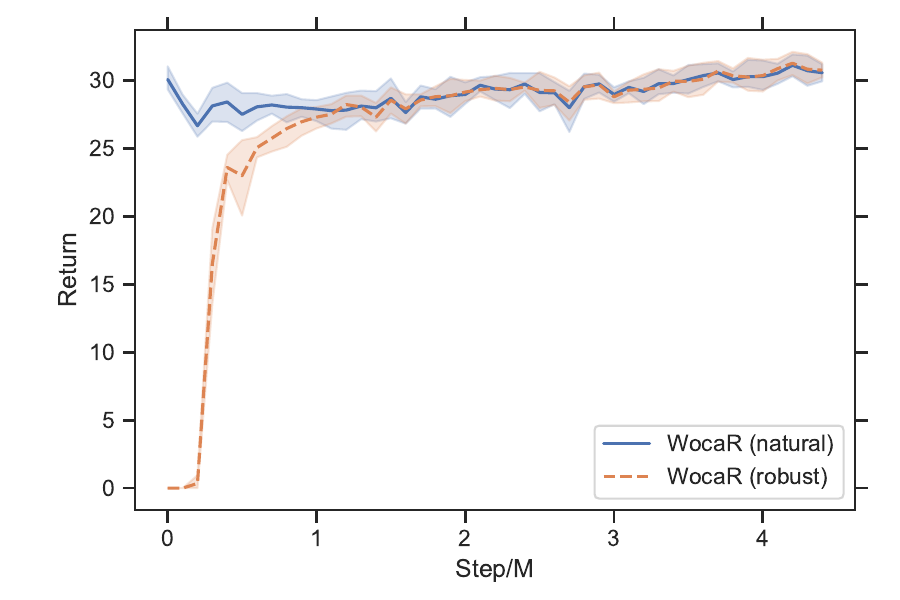}
\includegraphics[width=0.23\textwidth]{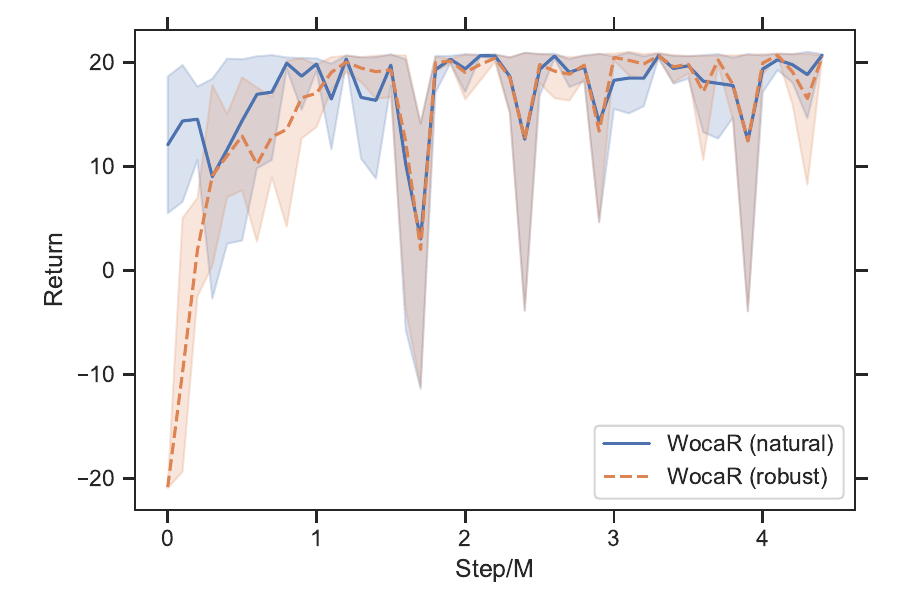}
    \caption{Natural and robustness performance exhibited
by WocaR-DQN agents during the training process on 4 Atari games.}
    \label{app fig:wocar NRreturn}
\end{figure}

\begin{figure}[t]
    \centering
\includegraphics[width=0.5\columnwidth]{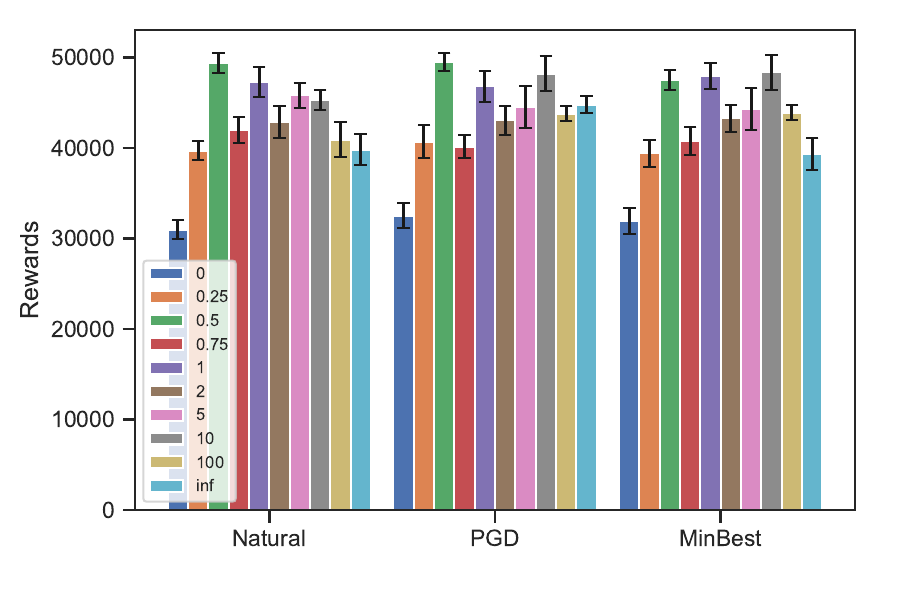}
    \vskip -0.1in
    \caption{Natural, PGD attack and MinBest attack rewards of CAR-DQN with different soft coefficients on RoadRunner.}
    \label{app fig:soft roadrunner}
\end{figure}

\begin{figure}[t]
    \centering
    \begin{minipage}{0.228\textwidth}
        \centering
        \setlength{\parindent}{0.5em}
        \quad \scriptsize{RoadRunner}
    \end{minipage}
    \begin{minipage}{0.228\textwidth}
        \centering
        \setlength{\parindent}{1em}
        \quad \scriptsize{BankHeist}
    \end{minipage}
    \\
    \rotatebox{90}{\tiny{\qquad \qquad \quad Robust}}
    \begin{subfigure}
        \centering
        \includegraphics[width=0.228\textwidth]{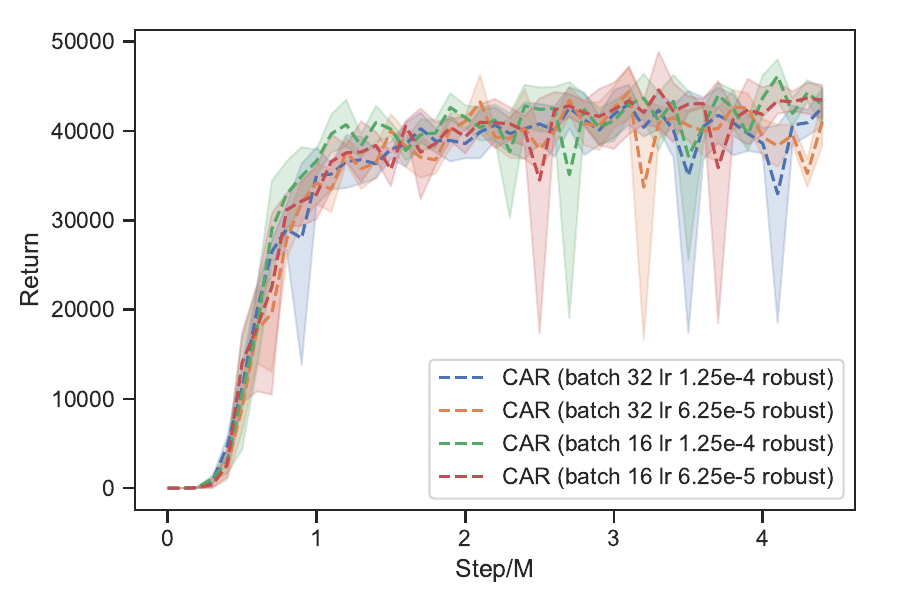}        
    \end{subfigure}
    \begin{subfigure}
        \centering
        \includegraphics[width=0.228\textwidth]{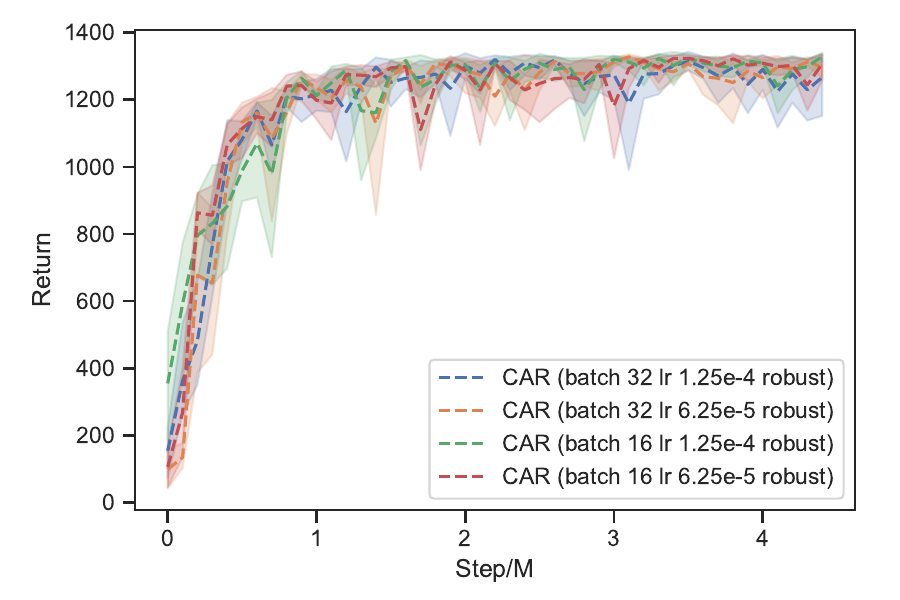}
    \end{subfigure}
    \\
    \rotatebox{90}{\tiny{\qquad \qquad \quad Natural}}
    \begin{subfigure}
        \centering
        \includegraphics[width=0.228\textwidth]{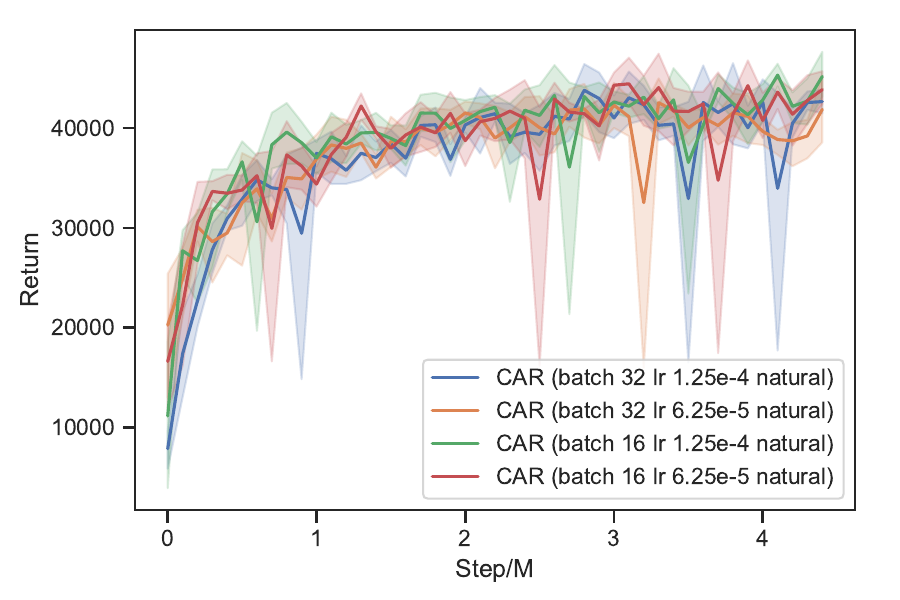}
    \end{subfigure}
    \begin{subfigure}
        \centering
        \includegraphics[width=0.228\textwidth]{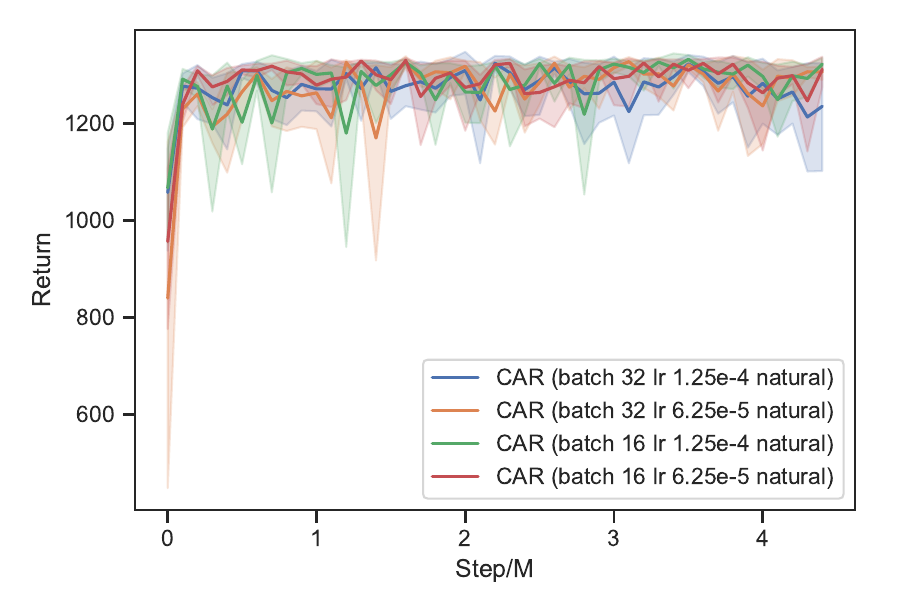}
    \end{subfigure}
    \caption{Episode rewards of CAR-DQN with different batch sizes and learning rates during training on RoadRunner and BankHeist with and without PGD attack.}
    \label{fig:batch and lr}
\end{figure}

\textbf{Insensitivity of learning rate and batch size.} 
We compare the performance of CAR-DQN with different small batch size $(16, 32)$ and learning rate $(1.25\times 10^{-4}, 6.25\times 10^{-5})$ which are respectively used by~\citet{zhang2020robust, liang2022efficient}. As shown in Figure \ref{fig:batch and lr}, we can see CAR-DQN is insensitive to these parameters.

\end{document}